%% file: main.tex
\documentclass{article}

\PassOptionsToPackage{numbers, compress}{natbib}
\usepackage[main,final]{neurips_2026}

\usepackage[utf8]{inputenc} 
\usepackage[T1]{fontenc}    
\usepackage[colorlinks,citecolor=blue]{hyperref}       
\usepackage{url}            
\usepackage{booktabs}       
\usepackage{amsfonts}       
\usepackage{nicefrac}       
\usepackage{microtype}      
\usepackage{xcolor}         
\usepackage{graphicx}
\usepackage{hyperref}
\usepackage{subcaption}
\usepackage{amsmath}
\usepackage{amssymb}
\usepackage{mathtools}
\usepackage{amsthm}

\usepackage{algorithm}
\usepackage{algorithmic}
\usepackage[most]{tcolorbox}

\input{arxiv_preamble}

\newcommand{\sft}{\textrm{SFT}}
\newcommand{\tsk}{\textrm{TASK}}

\newcommand{\dpo}{\textrm{DPO}}
\newcommand{\newtask}{\tau_{\text{new}}}
\newcommand{\inner}[2]{\left\langle #1,\, #2\right\rangle}
\newcommand{\Hbase}{H_{\mathrm{base}}}
\newcommand{\thbase}{\theta_{\mathrm{base}}^{*}}
\newcommand{\thM}{\theta_{M}^{*}}
\newcommand{\tht}{\theta_{\tau}^{*}}
\newcommand{\thbar}{\bar{\theta}^{*}}

\newcommand{\Fbase}{F_{\mathrm{base}}}
\newcommand{\Fmaml}{\calL_{\mathrm{MAML}}}
\newcommand{\Dbase}{D_{\mathrm{base}}^{2}}

\newcommand{\maml}{\mathrm{MAML}}
\newcommand{\base}{\mathrm{base}}

\title{Meta-Learning Preferences for Multilingual LLM Alignment}

\author{%
  Jiaying Lin$^1$\thanks{Equal contribution. Correspondence to
  \texttt{seong.son.22@ucl.ac.uk}.},
  Seongho Son$^{2,*}$, Nam Phuong Tran$^3$, \\
  \textbf{Long Tran-thanh$^1$, Ilija Bogunovic$^4$, Debmalya Mandal$^1$} \\
  $^1$Department of Computer Science, University of Warwick, Coventry, United Kingdom \\
  $^2$Department of Computer Science, University College London, London, United Kingdom \\
  $^3$Inria centre at the University of Lille, Lille, France \\
  $^4$Department of Mathematics and Computer Science, University of Basel, Basel, Switzerland \\
  \texttt{\{Jiaying.Lin, Long.Tran-Thanh, Debmalya.Mandal\}@warwick.ac.uk} \\
  \texttt{seong.son.22@ucl.ac.uk, phuong-nam.tran@inria.fr, ilija.bogunovic@unibas.ch}
}

\begin{document}

\maketitle

\doparttoc
\faketableofcontents

\begin{abstract}
Unequal availability of human preference data across languages poses a significant challenge for aligning large language models in multilingual settings. 
To address the lack of sufficient data in low-resource language alignment, we propose a meta-learning framework for Reinforcement Learning from Human Feedback and Direct Preference Optimization.
By leveraging preference data from other languages, our framework learns a transferable initialization that enables effective adaptation to a target language with minimal data.
We provide theoretical guarantees for both the meta-reward modeling and meta-policy optimization settings, and empirically demonstrate the effectiveness of our approach on 
multilingual benchmarks.
In an extremely low-resource setting with only 100 target-language preference samples, our approach achieves  up to $33.4\%$ win-rate improvements over baseline methods, and consistently outperforms baselines across multiple target languages and model scales.
Our approaches retain these advantages across different combinations of meta-training languages and varying linguistic distances from the target languages.
\looseness=-1

\end{abstract}

\section{Introduction}\label{sec:introduction}
Alignment of Large Language Models (LLMs) have shown great progress, using methods such as Reinforcement Learning with Human Feedback (RLHF) \citep{ziegler2020finetuninglanguagemodelshuman, ouyang2022training} and Direct Preference Optimization (DPO) \citep{rafailov2024direct}. However, these approaches heavily rely on a large amount of high-quality human preference data, whose availability is severely unequal \textit{across languages}. As a result, LLMs often exhibit significantly weaker performance in low-resource languages, failing to provide global accessibility.\looseness=-1

Existing works have investigated improving multilingual LLMs through approaches such as fine-tuning to a specific target language \citep{dac2023okapi} or aggregating preference data across multiple languages and training models on them \citep{lai2024llms}. While these approaches are effective, they rely on strong assumptions, either requiring sufficient target-language data or a close linguistic relationship between languages.
In practice, direct transfer of general capabilities across languages is suboptimal, particularly, when only a small amount of data is available for a given language.\looseness=-1

In this work, we propose a meta-learning approach for preference-based multilingual alignment of LLMs. We formulate multilingual alignment as an \emph{adaptation problem}, by treating each language in the dataset as a task. Our approach leverages  data from multiple languages to meta-train the LLM policy, that can be rapidly adapted to a new target language using only a small number of preference comparisons. 
We visually summarize our approach in \Cref{fig: summary}. For the first time, we instantiate this framework for both RLHF (via meta-reward learning) and DPO (via meta-policy learning) using gradient-based meta-learning.\looseness=-1

\begin{figure*}[t!]
    \includegraphics[width=\linewidth, trim={4cm 0 3.5cm 0}, clip]{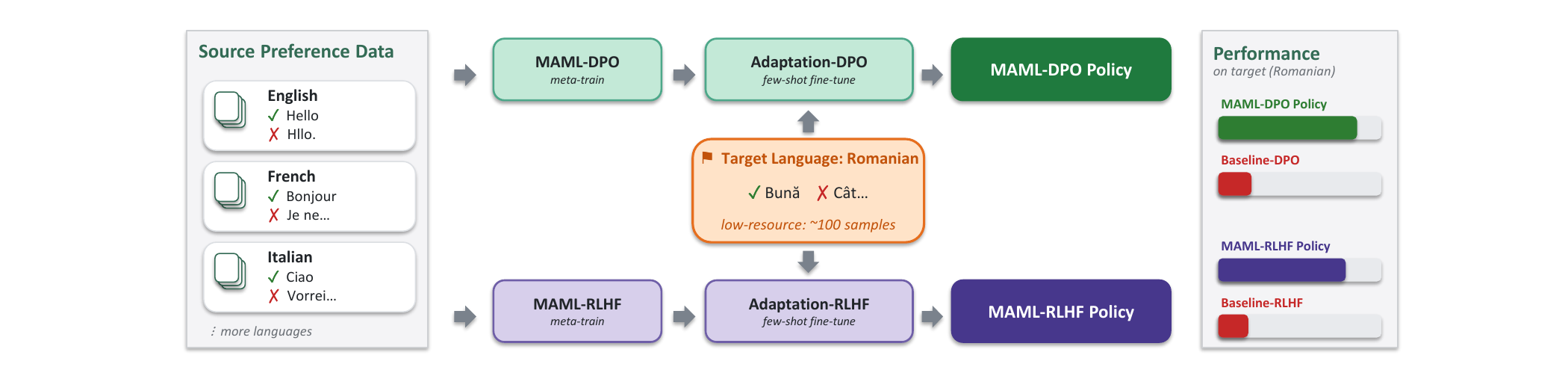}
    \vspace{-10pt}
    \caption{
        When the target language has low resource availability (e.g. Romanian) and has a very small number of data samples, such as 100, conventional approaches fail to properly fine-tune LLMs. On the other hand, our approach first meta-trains the model using other languages with sufficient data, allowing the model to boost sample efficiency in the target language. Our MAML-DPO/MAML-RLHF methods 
        achieve significant performance improvement over the baseline methods, given the same amount of datapoints from the target language.\looseness=-1
    }
    \label{fig: summary}
\end{figure*}

Multilingual preference alignment naturally exhibits a task-level adaptation structure, where each language can be viewed as a distinct task with unevenly distributed preference data. Existing methods commonly extend RLHF or DPO through joint training over multilingual data, leaving task-level adaptation largely underexplored. Gradient-based meta-learning is well suited to this setting, as it learns an initialization that can efficiently adapt across related tasks~\citep{finn2017model, zou2022unraveling}. We close this gap by formulating multilingual preference alignment as a meta-learning problem and instantiating it as MAML-RLHF and MAML-DPO. 
The resulting algorithms meta-train across languages and adapt to a target language with as few as 100 preference samples. On the theory side, the Bradley-Terry losses that govern preference learning sit outside the regression/classification setting of existing MAML convergence theory, and we establish the first convergence guarantees for MAML in this setting, characterising the regime in which the meta-trained initialisation provably accelerates target-language adaptation over both target-only and multitask baselines.\looseness=-1

We empirically evaluate our methods on various popular benchmarks across multiple target languages \citep{clark2018think, zellers2019hellaswag, hendrycks2020measuring, dac2023okapi}.
Using open-source LLMs with 270M and 7B parameters, we demonstrate that meta-trained models substantially outperform standard RLHF and DPO baselines in low-resource settings. Notably, with only 100 target-language preference samples, our approach improves win rates by up to 33.4\%.
In addition, the benefits of meta-training persist as the amount of adaptation data increases to more than 1k samples, still more stably improving performances across benchmarks than the baseline method.\looseness=-1

In this paper, we provide four main contributions:
\begin{itemize}
    \item We propose a unified meta-learning framework for aligning LLMs using preference data. This framework treats the multilingual preference alignment as a task distribution problem, where each language corresponds to a task. The framework effectively addresses the issue of uneven preference data availability and significantly improves sample efficiency for low-resource languages.
    \item We provide convergence guarantees of our reward modeling and preference optimization algorithms, theoretically proving the efficiency of our approach.
    \item Using multilingual instruction-following datasets, we empirically demonstrate the greatly improved sample efficiency of our approach, outperforming baselines by up to 33.4\% in win rates.
    \item We show that our method retains advantage across different combinations of meta-training and target languages, suggesting that our method learns a broadly transferable initialization for adaptation.\looseness=-1
\end{itemize}

\section{Preliminaries}\label{sec: preliminaries}

\paragraph{Offline learning via human preference.}
{ In the offline setting, we assume access to a dataset $\mathcal D = {(x,y_w,y_\ell)}$, where $x$ is a prompt and $(y_w,y_\ell)$ are two possible completions of the prompt $x$. A human labeler then indicates a preference between the two responses according to a reward function $r^\star(x,y)$ and the Bradley--Terry (BT) model:}
\begin{equation}
    \begin{aligned}
        P(y_1 \succ y_2 | x) &= \sigma(r^\star(x,y_1) - r^\star(x,y_2)) 
    = \frac{\exp(r^\star(x,y_1))}{ \exp(r^\star(x,y_1)) + \exp(r^\star(x,y_2))},
    \end{aligned}
\end{equation}
where $\sigma$ is the sigmoid function, $\sigma(z)=1/(1+\exp(-z))$.
We aim to learn an aligned policy from the offline preference dataset $\mathcal D$. To this end, we study RLHF and DPO.\looseness=-1

\paragraph{Reinforcement learning from human feedback.}
As proposed by \cite{ziegler2020finetuninglanguagemodelshuman}, RLHF consists of two phases: reward learning and policy optimization.
Let $r_\phi$ be a reward function parameterized by $\phi$. The RLHF loss is
\begin{equation}\label{eq:log-loss RLHF}
\begin{aligned}
    \calL^{\mathrm{RLHF}}&(\phi) \triangleq  -\E_{(x,y_w, y_\ell)\sim \mathcal D} 
    \left[ \log \left(\sigma\left(r_\phi(x,y_w) - r_\phi(x,y_\ell) \right) \right)\right].
\end{aligned}
\end{equation}
After learning the reward function (denoted $r_{\widehat \phi}$), the last step is to optimize a policy using the following regularized objective with respect to $r_{\widehat \phi}$:
\begin{equation} \label{eq: optimal polcy wrt learnt reward}
    \widehat \pi = \argmax_{\pi} \mathbb E_{\substack{x\sim \mathcal D  \\ y\sim \pi(\cdot|x)}}\left[r_{\widehat \phi}(x,y) -  \beta \log \frac{\pi(y|x)}{\pi^{\mathrm{SFT}}(y|x)}\right],
\end{equation}
where $\beta>0$ is a regularization coefficient that penalizes deviation from the reference policy $\pi^{\mathrm{SFT}}$, a supervised fine-tuned policy.
\looseness=-1

\paragraph{Direct preference optimization.} 
Direct Preference Optimization (DPO), proposed by \cite{rafailov2024direct}, is an alternative to RLHF. DPO updates the policy directly from preference data, skipping the separate reward-modeling step. The main idea is that the reward can be written implicitly in terms of the optimal policy and $\pi^{\mathrm{SFT}}$, which leads to a loss defined purely by preference comparisons.
In particular, let $\pi_\theta$ be a policy parameterized by $\theta$. Then, the DPO loss is\looseness=-1
\begin{equation}\label{eq:log-loss DPO}
\begin{aligned}
    \calL_{\mathrm{DPO}}&(\theta) \triangleq - \E_{(x,y_w, y_\ell) \sim \mathcal D} \,\,
    \! \! \!\left[ \log \left(\sigma\left(\beta\log \frac{\pi_\theta(y_w|x)}{\pi^{\mathrm{SFT}}(y_w|x)} - \beta\log \frac{\pi_\theta(y_\ell|x)}{\pi^\mathrm{SFT}(y_\ell|x)}\right)\right)\right].
\end{aligned}
\end{equation}

\section{Multilingual preference alignment as a meta-learning problem}
\label{sec: meta-learning framework}

Multilingual preference alignment is fundamentally a \emph{multi-task adaptation} problem. Each language $\tau$ has its own preference distribution, shaped by elements such as  grammar and culture. On the other hand, large-scale comparison data is collected for only a handful of high-resource languages. Thus, the relevant question is not ``how do we align a \emph{single} language?'' but ``how do we align a \emph{target} language efficiently when preference data is available only from \emph{other} languages?''. In this point of view, the problem maps to the meta-learning regime: each language is a task drawn from a distribution $P_\tsk$, with taskwise preference loss $\calL(\omega;\tau)$. The goal here is a meta-initialisation from which a few gradient steps on a small target-language dataset $\calD^{\newtask}$ recover a near-optimal aligned policy. The Model-Agnostic Meta-Learning (MAML) framework~\citep{finn2017model} formalizes this objective through the one-step meta-learning loss\looseness=-1
\begin{equation}\label{eq:shared-reward}
    \calL_M(\omega) \;\triangleq\; \E_{\tau \sim P_\tsk}\!\bigl[\calL\bigl(\omega - \alpha\,\nabla_\omega \calL(\omega;\tau);\,\tau\bigr)\bigr],
\end{equation}
with $\alpha>0$ the inner-loop learning rate. Adaptation to a new task $\tau_{\text{new}}$ then runs gradient descent on the empirical loss $\widehat\calL(\,\cdot\,;\tau_{\rm new})$ initialized at the meta-solution $\omega_M^\star = \arg\min_\omega \calL_M(\omega)$:
\begin{equation}\label{eq:adaptation-update}
    \widehat\omega^{\,\newtask}_{t+1} \;=\; \widehat\omega^{\,\newtask}_t - \eta\,\nabla_\omega \widehat\calL\bigl(\omega^{\,\newtask}_t;\newtask\bigr), \qquad \widehat\omega^{\,\newtask}_0 = \omega_M^\star,
\end{equation}
with adaptation step size $\eta>0$. Instantiating $\calL(\omega;\tau)$ as the per-language RLHF reward-modelling loss or the per-language DPO policy loss yields the two algorithms developed in this paper: MAML-RLHF and MAML-DPO.\looseness=-1

Both RLHF and DPO are based on a Bradley-Terry sigmoid log-likelihood structure, which is different from the regression and classification settings on which existing MAML convergence theory is built~\citep{fallah2020convergence,zou2022unraveling}. The per-task strong-convexity constant of the DPO loss depends on both the feature-coverage constant and the reference policy log-ratio bound, and the meta-objective $\calL_M$ becomes strongly convex only when the inner learning rate $\alpha$ is restricted in terms of these constants. For DPO, the reference policy term further couples the loss to an external SFT model. We give the first convergence analysis of MAML for both losses in \Cref{sec: theoretical_analysis} (\Cref{theorem: meta-RLHF convergence,theorem: meta-DPO convergence}), and characterize the regime in which the meta-trained initialization provably accelerates target language adaptation over both target-only and multitask baselines (\Cref{thm:accelerated_MAML_DPO_wrt_Multitask,thm: maml-vs-baseline}).\looseness=-1

\begin{algorithm}[t]
    \caption{Meta-reward learning}\label{alg:FO_MAML RLHF}
    {\raggedright\quad \textbf{Input}: Dataset $\calD = \set{(x^i,y^i_w,y^i_\ell, \tau^i)}_{i=1}^N$, \#iterations $T$, initial reward model $\phi_1$.\par}
    \begin{algorithmic}[1]
    \STATE Initialize $\phi \leftarrow \phi_1$
    \FOR{$t = 1,\ldots, T$}
        \STATE Choose a batch $\calB_t \subseteq \calI$ of i.i.d. tasks from distribution $P_\tsk$ with size $\abs{\calB_t}=B$
        \FOR{each task $\tau \in \calB_t$}
            \STATE Sample $\calD^{\tau}_{\text{in}} \subseteq \calD^\tau$
            \STATE Compute empirical loss  $\widehat{\calL}^{\mathrm{RLHF}}(\cdot;\tau)$ with $\calD^{\tau}_{\text{in}}$

            \STATE $\phi_{t+1}^\tau = \phi_t - \alpha \nabla_\phi \widehat{\calL}^{\mathrm{RLHF}}(\phi_t;\tau)$
            
            \STATE Sample $\calD^\tau_{\text{o}} \subseteq \calD^\tau$
            \STATE Compute empirical loss $\widehat \calL^{\mathrm{RLHF}}(\phi_{t+1}^\tau;\tau)$ with  $\calD^\tau_{\text{o}}$
        \ENDFOR
        \STATE $\phi_{t+1} = \phi_t - \frac{\eta}{B}\sum_{\tau \in \calB_t} \nabla_\phi \widehat{\calL}^{\mathrm{RLHF}}(\phi^{\tau}_{t+1} ;\tau)$
    \ENDFOR
    \end{algorithmic}
    \quad \textbf{Output}: $\widehat \phi_M = \phi_{t^\star}$ where $t^\star \sim \text{Unif}\set{1,\ldots,T}$. 
\end{algorithm}

In both algorithms, meta-learning is anchored at the preference learning stage. Formally, a task (language) $\tau$ is drawn from a distribution $P_{\tsk}$, and we assume access to a multi-task dataset $\mathcal D = \{(x^i,y_w^i,y_\ell^i,\tau^i)\}_{i=1}^N$. We also denote $\mathcal D^\tau \subset \mathcal D$ as the subset of the dataset corresponding to task $\tau$. 
Slightly abusing notation, we write $(x,y_w,y_\ell) \sim \cal D^\tau$ to denote uniform sampling from the dataset $\cal D^\tau$ for task $\tau$.\looseness=-1

\subsection{MAML-RLHF} 
\label{subsec: Meta Learning for RLHF}

\begin{wrapfigure}{r}{0.6\linewidth}
\begin{minipage}{\linewidth}
\vspace{-25pt}
\begin{algorithm}[H]
    \caption{Reward adaptation} \label{alg:meta-rlhf adaption}
    {\raggedright\quad \textbf{Input}: Dataset $\calD^{\newtask}$, meta reward model $\widehat \phi_M$, \#gradient steps $\widetilde{m}$, step size $\eta$.\par}
    \begin{algorithmic}[1]
    \STATE Initialize $\phi \leftarrow \widehat\phi_M$
    \FOR{$i = 1,\ldots, \widetilde{m}$}
    \STATE $\phi \leftarrow \phi - \eta\nabla_\phi \widehat{\calL}^{\mathrm{RLHF}}(\phi;\newtask)$.
    \ENDFOR
    \end{algorithmic}
     \quad \textbf{Output}: $\phi$.
\end{algorithm}
\end{minipage}
\vspace{-20pt}
\end{wrapfigure}

For RLHF, the per-task object that carries the cross-lingual preference signal is the reward function $r_\phi$. We therefore meta-learn $r_\phi$ across languages and adapt it to the target language at deployment time, leaving SFT and policy optimisation as standard single-task stages. Concretely, the task-wise reward-modelling loss in \eqref{eq:shared-reward} takes the form
\begin{equation}\label{eq:log-loss RLHF task-wise}
\begin{aligned}
    \calL^{\mathrm{RLHF}}&(\phi;\tau) \triangleq - \E_{(x,y_w, y_\ell)\sim \mathcal D^\tau}
    \left[ \log \left(\sigma\left(r_\phi(x,y_w) - r_\phi(x,y_\ell) \right) \right)\right].
\end{aligned}
\end{equation}
Given $\calL^{\mathrm{RLHF}}(\phi,\tau)$, we solve the meta-learning problem $\phi_M^\star \in \argmin_\phi \mathcal L_M^{\mathrm{RLHF}}(\phi)$, where\looseness=-1
\begin{equation}\label{eq:meta loss RLHF}
\begin{aligned}
    \calL_M^{\mathrm{RLHF}}(\phi) &\triangleq \E_{\tau \sim P_\tsk(\cdot)} 
    \left[ \calL^{\mathrm{RLHF}}(\phi^\tau ; \tau)\right], \quad
    \phi^\tau & \triangleq \phi - \alpha \nabla_\phi\calL^{\mathrm{RLHF}}(\phi; \tau).
\end{aligned}
\end{equation}

Algorithm \ref{alg:FO_MAML RLHF} utilizes the MAML framework~\cite{finn2017model, fallah2020convergence} to approximate the solution of $\eqref{eq:meta loss RLHF}$. 
 It learns a meta-initialization $\widehat\phi_M$  that can be quickly adapted to a new task using few samples. 
 
At each meta-iteration $t$, we sample a batch of i.i.d.~tasks $\calB_t \sim P_\tsk$. 
For each task $\tau \in \calB_t$, we first sample an inner-loop dataset $\calD^\tau_{\text{in}} \subset \mathcal D^\tau$ and perform one gradient step on the empirical RLHF loss $\widehat{\calL}^{\mathrm{RLHF}}(\cdot;\tau)$ starting from $\phi_t$, yielding an adapted parameter $\phi^{\tau}_{t+1}$. 
 We then evaluate the adapted parameter on a fresh outer dataset $\calD^\tau_{\text{o}} \subset \mathcal D^\tau$, which is sampled independently of $\calD^\tau_{\text{in}}$, to compute the outer empirical loss, and use its gradient to update the meta-parameter.
Finally, we aggregate gradients across tasks and take a meta-update step $\phi_{t+1} = \phi_t - \frac{\eta}{B}\sum_{\tau \in \calB_t}\nabla_\phi \widehat{\calL}^{\mathrm{RLHF}}(\phi^{\tau}_{t+1};\tau)$. 

In the adaption phase, we are given a new task $\tau_{\text{new}}$, and a preference dataset $\calD^{\newtask} = \{(x^i, y^i_w, y^i_\ell)\}_{i=1}^{n}$ where $y^i_w$ (resp. $y^i_\ell$) is the winning (resp. losing) completion among the pair of completions $y^i_w, y^i_\ell$ on the task $\newtask$ given prompt $x$. The reward model $r_{\phi_M}$ is updated as in Algorithm \ref{alg:meta-rlhf adaption}, by performing gradient descent on the empirical loss function $\widehat{\calL}^{\mathrm{RLHF}}(\phi; \tau_{\text{new}})$. This results in an adapted reward model $r_\phi^{\tau_{\text{new}}}$. 
Given the adapted reward model $\phi$, we can train the KL-regularized policy specific to task $\newtask$ using \eqref{eq: optimal polcy wrt learnt reward}.

\subsection{MAML-DPO} 
\label{subsec: Meta Learning for DPO}

For DPO, the policy is the only learnable object in the preference-learning stage, so meta-learning naturally lands on the policy itself. The policy is meta-trained on preference datasets from non-target languages, then adapted to the target language by running DPO on $\calD^{\newtask}$.\looseness=-1

For a task $\tau$, the task-wise loss function of policy learning is  
\begin{equation}\label{eq:log-loss DPO task-wise}
\begin{aligned}
    \calL_{\mathrm{DPO}}&(\theta;\tau) \triangleq - \E_{(x, y_w, y_\ell)\sim \calD^\tau} \,\,
    \! \! \!\left[ \log \left(\sigma\left(\beta\log\frac{\pi_\theta(y_w|x)}{\pi^{\mathrm{SFT}}(y_w|x)} - \beta\log \frac{\pi_\theta(y_\ell|x)}{\pi^\mathrm{SFT}(y_\ell|x)}\right)\right)\right].
\end{aligned}
\end{equation}
Given $\calL^{\mathrm{DPO}}(\theta,\tau)$ 
we solve the meta-learning problem $\phi_M^\star \in \argmin_\theta \mathcal L_M^\dpo(\theta)$, where
\begin{equation}\label{eq:meta loss DPO}
\begin{aligned}
    \calL_M^{\dpo}(\theta) &\triangleq \E_{\tau \sim P_\tsk(\cdot)} 
    \left[ \calL^{\dpo}(\theta^\tau ; \tau)\right], \quad
    \theta^\tau & \triangleq \theta - \alpha \nabla_\theta\calL^{\dpo}(\theta; \tau).
\end{aligned}
\end{equation}

Similar to MAML-RLHF, we employ a MAML framework (Algorithm \ref{alg:FO_MAML RLHF}) to approximate the solution for \eqref{eq:meta loss DPO} by using DPO loss instead.
In the adaptation phase, however, in contrast to meta-RLHF, we can directly update the meta-policy $\pi_M$ for the new task $\newtask$ using the empirical DPO loss~\cite{rafailov2024direct} 
i.e. $\widehat{\calL}_{\dpo}(\pi; \calD_{\tau_{\text{new}}})$.
In particular, starting from the meta-policy $\pi_M$, we perform gradient descent steps for the loss function $\widehat{\calL}_\dpo$ and dataset $\calD^{\newtask}$, which is described as in Algorithm \ref{alg:meta-dpo adaption}.
\begin{wrapfigure}{r}{0.45\linewidth}
    \vspace{-10pt}
    \begin{minipage}{\linewidth}
        \begin{algorithm}[H]
            \caption{Policy adaptation} \label{alg:meta-dpo adaption}
            {\raggedright \quad \textbf{Input}: Dataset $\calD^{\newtask}$, meta policy  $\widehat \theta_M$, number of gradient steps $\widetilde m$, and step size $\eta$.\par}
            \begin{algorithmic}[1]
            \STATE Initialize $\theta \leftarrow \widehat\theta_M$.
            \FOR{$i = 1,\ldots, \widetilde{m}$}
            \STATE $\theta \leftarrow \theta - \eta\nabla_\theta \widehat{\calL}^\dpo(\theta;\tau_{\text{new}})$
            \ENDFOR
            \end{algorithmic}
             \quad \textbf{Output}: $\pi^\dpo_{\tau_{\text{new}}} = \pi_\theta$.
        \end{algorithm} 
    \end{minipage}
    \vspace{-10pt}
\end{wrapfigure}

 Both approaches described in Section \ref{subsec: Meta Learning for RLHF}, \ref{subsec: Meta Learning for DPO} aim to transfer preference structure across languages for few-shot adaptation, but they apply meta-learning at different points in the pipeline. In practice, RLHF and DPO exhibit different stability and compute tradeoffs even in the standard setting, and this also carries over to the meta setting.
We therefore treat MAML-RLHF and MAML-DPO as complementary methods and evaluate how effectively each method adapts to the target language.\looseness=-1

\section{Theoretical analysis}\label{sec: theoretical_analysis}
In this section, we initiate the study of how accurately Algorithm \ref{alg:FO_MAML RLHF} (i.e., the MAML algorithm) can estimate the solution of the meta-learning problem in both the RLHF and DPO settings.
When clear from the context, we drop the subscripts RLHF and DPO and write $\mathcal L$ for the corresponding loss. In particular, we use $\mathcal L(\phi)=\mathcal L^{\mathrm{RLHF}}(\phi)$ for reward learning and $\mathcal L(\theta)=\mathcal L^{\mathrm{DPO}}(\theta)$ for policy learning. 

For theoretical tractability, we focus on a linear reward class and a log-linear policy class, a standard setup in theoretical studies (e.g., \cite{nika2024reward,son2025right,xiong2024iterative,ji2025reinforcement}). 
In particular, we study linear reward and log-linear policy functions based on a $d$-dimensional feature map $\psi$ satisfying $\max_{(x,y)}\|\psi(x,y)\|_2^2\leq 1$.
We also define $\Delta\psi(x,y_w,y_\ell)\triangleq (\psi(x,y_w) - \psi(x,y_\ell))$.

\subsection{Convergence of MAML-RLHF}
\label{subsec: analysis of MAML-RLHF}
In this section, we analyze the convergence rate of Algorithm~\ref{alg:FO_MAML RLHF} in the RLHF setting.
For some constant $B_\phi>0$, we consider the linear reward function class
\begin{equation}
    \mathcal F = \{r_\phi(x,y) = \phi^\top\psi(x,y) \mid \|\phi \|_2\leq B_\phi \},
\end{equation}
with the corresponding task-wise loss
\begin{align*}
    \calL(\phi;\tau) =-\E_{(x,y_w,y_\ell) \sim \calD_\tau} 
      \left[ \log \sigma \left( \phi^\top \Delta\psi(x,y_w, y_l)  \right)\right]. \label{eq: restate_loss_reward}
\end{align*}

We make the following assumptions on the data distribution.
\begin{assumption}[Uniform coverage of data distribution]\label{asn:full-coverage}
    For any task $\tau$, there exists a constant $\nu > 0$ s.t.
    \[\E_{(x,y_w, y_\ell) \sim D^\tau}[(\Delta\psi(x,y_w, y_\ell) \Delta\psi(x,y_w, y_\ell)^\top]\succcurlyeq \nu I.\]
\end{assumption}

Assumption \ref{asn:full-coverage} is the standard uniform coverage condition in offline reinforcement learning. Here it is used to ensure convexity of the meta-loss function.

\begin{assumption}[Variance of tasks]\label{asn:bounded-task-variance}
Let $\nabla_\phi \calL(\phi) = \E_{\tau \sim P_\tsk}\left[ \nabla_\phi \calL(\phi,\tau)\right]$. Then there exists $\sigma > 0$ such that 
\[\E_{\tau \sim P_\tsk}\left[ \norm{\nabla_\phi \calL(\phi) - \nabla_\phi \calL(\phi,\tau)}_2^2 \right] \le \sigma^2.\]
\end{assumption}

Assumption \ref{asn:bounded-task-variance} bounds task diversity and is standard in the MAML literature \cite{fallah2020convergence}.

\paragraph{Main results.} We state a convergence guarantee for MAML applied to the meta objective for RLHF.\looseness=-1
\begin{theorem} \label{theorem: meta-RLHF convergence}
Suppose Assumptions \ref{asn:full-coverage} and \ref{asn:bounded-task-variance} hold, and 
Algorithm \ref{alg:FO_MAML RLHF} is run for $T =  O\left( \frac{L}{\alpha^2 \sigma^2}\right)$ iterations, and the batch size $B =1/\alpha^2$, dataset sizes $D_o = \abs{\calD_o^i} = O(1), D_{\text{in}} = \abs{\calD^i_{\text{in}}} = 1/\sigma$.  Then for the linear model:
\[\mathbb E \left[ \norm{\phi^\star_M - \widehat{\phi}_M}_2 \right] \le \frac{2 \alpha \sigma}{m_M},
\]
where $m_M = \frac{c\cdot \nu}{\nu + 2(4+c)(1+\exp(2B_\phi))}$ for a constant $c > 0$.
\end{theorem}
\begin{proof}
    Provided in Appendix \ref{appendix: Proof of Convergerence for Meta-RLHF}.
\end{proof}
Theorem \ref{theorem: meta-RLHF convergence} indicates that the proposed first-order method can approximate the solution of the MAML-RLHF objective up to a prescribed error. In particular, it provides a guarantee on how close the output of MAML-RLHF is to the meta-solution of \eqref{eq:meta loss RLHF} in Euclidean distance. Moreover, \citet{zou2022unraveling} shows that the meta-solution is, in expectation, close (in Euclidean distance) to the task-wise solution (see their Theorem 2). This implies that the output of MAML-RLHF also yields an initialization close to the task-wise solution.\looseness=-1

\subsection{Convergence of MAML-DPO}
\label{subsec: analysis of MAML-DPO}
We now analyze the convergence rate of Algorithm~\ref{alg:FO_MAML RLHF} in the DPO setting.
Given constant $B_\theta >0$, we define the log-linear (softmax) policy with feature map $\psi$ as
\begin{equation}
\label{eq:loglinear_policy}
\textstyle 
\Pi = \left\{\pi_\theta(y\mid x) = \frac{\exp(\theta^\top \psi(x,y))}{\sum_{y'}\exp(\theta^\top \psi(x,y'))}\, \bigg|\, \|\theta \|_2\leq B_\theta  \right\}.
\end{equation}
Using this formulation, the population DPO loss for a task $\tau$ \cite{nika2024reward} can be written as
\begin{equation}
\label{eq:logistic_form}
\begin{aligned}
    \calL(\theta;\tau) &=
\mathbb{E}_{(x,y_w,y_\ell)\sim D^\tau} \left[
\log\big(1+\exp(-z_\theta(x,y_w,y_\ell))\big)
\right],
\end{aligned}
\end{equation}
where $z_\theta(x,y_w,y_\ell) \triangleq \beta\theta^\top \Delta\psi(x,y_w,y_\ell) - J(x,y_w,y_\ell)$, and $J(x,y_w,y_\ell) \triangleq \beta \log\frac{\pi^{\sft}(y_w\mid x)}{\pi^{\sft}(y_\ell\mid x)}$.
We further make the following assumptions.
\begin{assumption}[Bounded reference log-ratio]
\label{ass:bounded_ref_ratio}
For all $\tau$ and $(x,y_w,y_\ell) \in \mathcal D^\tau$, there exists $J_0<\infty$ such that
$
\left|\log\frac{\pi^{\sft}(y_w\mid x)}{\pi^{\sft}(y_\ell\mid x)}\right|\le J_0.
$
Equivalently, $|J(x,y_w,y_\ell)|\le \beta J_0$.
\end{assumption}
For all triplet $(x,y_w,y_\ell)$, define the uniform bound on $\Pi$:
\begin{equation}
\label{eq:Zmax_theta}
|z_\theta(x,y_w,y_\ell)|
\le 2\beta B_\theta + \beta J_0
\triangleq Z_{\max}.
\end{equation}
Assumption \ref{ass:bounded_ref_ratio} is mild because, in practice, both $y_w$ and $y_\ell$ typically occur with non-negligible probability under the reference policy, so imposing a uniform bound on the log-ratio is a standard coverage condition.

Next, we assume the following bound on the variability of task-wise gradients.
\begin{assumption}
\label{ass:variance_of_task_theta}
Let $\nabla_\theta \calL(\theta) = \E_{\tau \sim P_\tsk}\left[ \nabla_\theta \calL(\theta,\tau)\right]$. Then there exists $\sigma > 0$ such that 
\[\E_{\tau \sim P_\tsk}\left[ \norm{\nabla_\theta \calL(\theta) - \nabla_\theta \calL(\theta,\tau)}_2^2 \right] \le \sigma^2.\]
\end{assumption}

\paragraph{Main result.} The following theorem gives a finite-sample convergence guarantee for first-order MAML applied to the MAML-DPO objective.
\begin{theorem} \label{theorem: meta-DPO convergence}
Assume Assumptions \ref{asn:full-coverage}, \ref{ass:bounded_ref_ratio}, \ref{ass:variance_of_task_theta} hold.
Let $\alpha \leq \frac{1}{\beta^2\left(2+\frac{8(1+\exp(Z_{\max}))}{\nu}\right)}$.
Suppose Algorithm \ref{alg:FO_MAML RLHF} is run for $T =  O\left( \frac{L}{\alpha^2 \sigma^2}\right)$ iterations, and the batch size $B =1/\alpha^2$, dataset sizes $D_o = \abs{\calD_o^i} = O(1), D_{\text{in}} = \abs{\calD^i_{\text{in}}} = \beta/\sigma$.  Then for the log-linear model, we obtain 
\[
\norm{\theta^\star_M - \widehat{\theta}_M}_2 \le \frac{2 \alpha \sigma}{m_M},
\]
where $m_M = \beta^2 \frac{c\nu^3}{(1+\exp(Z_{\max}))^3}$ for some constant $c > 0$.
\end{theorem}
\begin{proof}
    Provided in Appendix \ref{appendix: Proof of Convergerence for Meta-DPO}.
\end{proof}

\begin{remark}
Theorem~\ref{theorem: meta-DPO convergence} is the DPO counterpart of Theorem~\ref{theorem: meta-RLHF convergence}. It shows that Algorithm~\ref{alg:FO_MAML RLHF} can also approximate the minimizer of the DPO meta objective, i.e., $\widehat{\theta}_M$ converges to $\theta_M^\star$, with accuracy controlled by the inner-loop step size $\alpha$, the task-gradient variability $\sigma$, and the conditioning of the meta objective through $m_M$ (the strong-convexity constant of $\mathcal L_M$ prescribed in the theorems). Together, these results indicate that the same meta-learning pipeline can be applied to both RLHF (meta-reward learning) and DPO (meta-policy learning), yielding comparable convergence guarantees under analogous coverage and task-variance assumptions.\end{remark}

\begin{remark}
    We analyze an unconditional improvement of MAML-DPO's adaptation speed over baseline DPO in \Cref{appendix: MAML-DPO vs baseline-DPO}. As long as the inner learning rate is small enough, $\alpha \le \min\set{1/(8\beta^2),\, m_M/(2\sqrt{2}\,\mu_\tau)}$ (where $\mu_\tau$ is the strong-convexity constant of the per-task DPO loss; see \Cref{lemma:task_smooth_sc_theta}), \Cref{thm: maml-vs-baseline} shows the adaptation of MAML-DPO is faster than baseline DPO. This result further shows that MAML-DPO obtains greater acceleration when the gap between the SFT-initialized policy and preference-optimal policy is bigger, which corresponds to cases where the SFT policy is not yet good enough on the target task. The comparison of MAML-DPO to Multitask-DPO is analyzed in \Cref{thm:accelerated_MAML_DPO_wrt_Multitask}: we show that the relative improvement of MAML-DPO over Multitask-DPO depends on a task-geometry parameter $\Phi$ that measures the alignment between the MAML correction direction and the direction from the multitask solution $\theta_{\rm base}^\star$ to the unweighted task-optima centroid $\bar{\theta}^\star = \E_\tau[\theta^\star_\tau]$. When $\Phi$ is positive and $\alpha$ is small enough, MAML-DPO shows faster adaptation than Multitask-DPO.
\end{remark}
\section{Experiments}
We evaluate the proposed meta-learning frameworks in a multilingual instruction-following setting, where
each language is treated as a separate task. 
Across multiple model sizes and languages, we assess three key aspects of the proposed approach: (i) sample efficiency under limited target-language data, (ii) performance when sufficient
adaptation data is available, and (iii) robustness to linguistic distance and different meta-training language groups.\looseness=-1

\subsection{Experiment setup} \label{sec:experiment_settings}

We provide further details of our experiment settings in \Cref{appendix: Further Details of Experimental Setting}, and provide code of our experiments. \footnote{Source code is available at \href{https://github.com/geronest/maml-preferences}{https://github.com/geronest/maml-preferences}}

\textbf{Datasets.} We use two multilingual preference datasets for the experiments. We use the Okapi dataset \cite{dac2023okapi}, derived from Alpaca-64K \cite{alpaca}, which provides preference datasets of over 26 languages with over 40k preference pairs per language. We also use \texttt{multilingual/orca\_dpo\_pairs} publicly available on Huggingface, which provides at least 3.4k preference pairs across 7 languages.
\looseness=-1

{\textbf{Language models.}} 
We evaluate our methods under two different model configurations. In the first configuration, we use BLOOM \cite{bigscience_workshop_2022} with 7.1B parameters for both the RLHF and DPO settings, which follows a GPT-3-like architecture and was originally pretrained on a multilingual corpus. 
To examine performance under smaller model scales, we additionally consider a lightweight configuration based on Gemma3 \cite{team2025gemma}, which supports efficient fine-tuning for rapid adaptation. 
In this configuration, the policy model is replaced by a Gemma3 model with 270 million parameters, and a lightweight BLOOM model with 1.7B parameters is used as the reward model in RLHF. To improve training efficiency, we apply LoRA adapters for both BLOOM models during fine-tuning to reduce memory consumption. \looseness=-1

\textbf{Baselines.} In order to examine the efficiency of our approach, we compare against \textit{standard} RLHF and DPO pipelines trained with the same amount of target-language preference data as the meta-trained models. 
We additionally include a \textit{multitask} baseline that uses the same multilingual training data as the meta-learning methods, but removes the inner-loop adaptation step (i.e., setting the inner loop learning rate $\alpha=0$ in Algorithm \ref{alg:FO_MAML RLHF}). This baseline allows us to isolate the effect of meta-learning from standard multitask pretraining. To provide a stronger cross-lingual transfer baseline, we also include translation-test approaches, which uses \texttt{gpt-4o} to translate target language prompts and responses to a high-resource language (e.g. French) and translate it back to the target language. Further details of the experiment configurations are provided in \Cref{appendix: Further Details of Experimental Setting}.
\looseness=-1

\textbf{Evaluation metrics.}
\begin{figure}[t]

    \centering
    \includegraphics[width=0.9\linewidth,trim={0 0 0 0},clip]{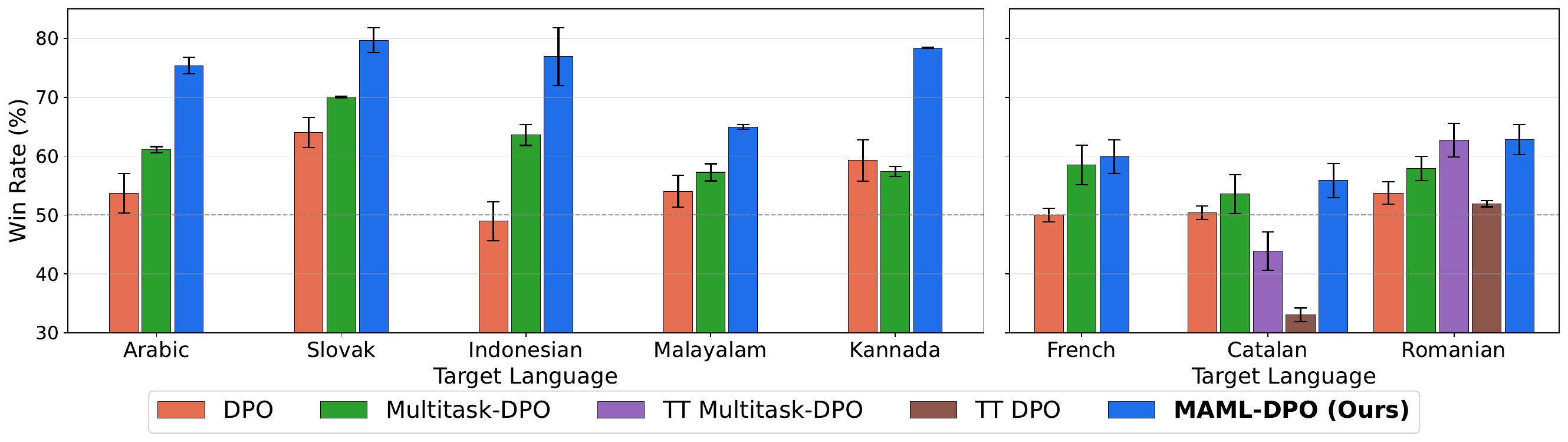}

    \caption{[Left] Win rate of Gemma 270M models adapted to 100 samples. MAML-DPO models significantly outperform baselines with higher win rates.  [Right]
    Win rate of BLOOM 7.1B models adapted to 100 samples. MAML-DPO models show comparable or higher win rates than the baselines, including Translate-Test (TT) with \texttt{gpt-4o} as translator.\looseness=-1 }
    \label{fig:mamldpo-exp1-exp2}
\end{figure}
We evaluate how quickly and robustly the models adapt to the limited target language data using separately trained and fixed reward models, referred to as the \textit{Judge reward models (RM)}. Each configuration uses its corresponding reward model architecture as the Judge RM, which is fully fine-tuned on the entire preference dataset and used exclusively for evaluation. The Judge RMs compute \textit{win rates} by comparing model-generated responses against a fixed base policy, defined as the initial policy model without any training or adaptation. During evaluation, all policy models generate responses to the same set of 1,000 prompts sampled from the test split of policy optimization dataset. 
\Cref{appendix: additional experimental results} also includes the \textit{reward accuracies} of models in the validation split of the preference dataset to examine how the performances of reward function evolve during the training.\looseness=-1

\subsection{Main results}
\label{sec:experiment results}

\textbf{How do meta-trained models perform under limited adaptation data?} 
We first evaluate adaptation performance in an extremely low-resource situation, using only 100 preference samples from
the target language. In this setting, the model repeatedly adapts to a fixed and small set of samples, resulting in both
limited data quantity and limited data diversity during adaptation.  
We evaluate model adaptation on eight target languages, reporting results for three Romance languages from the same family as the meta-training tasks and for five languages from different families in \Cref{fig:mamldpo-exp1-exp2}. We compare the win rates of different models against the base model using the judge RM as described in \Cref{sec:experiment_settings}. 
We report standard errors over three random seeds and show that MAML-DPO consistently outperforms baseline DPO and translation-test baseline. Figure \ref{fig:mamldpo-exp1-exp2}-left shows that when adapting to target languages from different language families, MAML-DPO significantly outperforms the multi-task baseline.
RLHF experiment results in \Cref{appendix:RLHF_results} 
show similar pattern, where MAML-RLHF models achieve consistently higher win rates across all target languages. Figure \ref{fig:mamldpo-exp1-exp2}-right demonstrates the performance of MAML-DPO comparable or better than the baselines when adapted to European languages, despite the translate-test baseline's use of a strong model \texttt{gpt-4o}. 

\looseness=-1

\textbf{How do meta-trained languages affect the adaptation performance?} We investigate the effect of languages used for meta-training on the adaptation performance in \Cref{fig:bloom7b1-languagegroups}. We test three different groups for meta-training: Romanian, Catalan, Italian, Spanish (Group~1, Romance); Vietnamese, Indonesian, Arabic, Hindi (Group~2, mixed non-Romance with, spanned high, medium and low resource); Chinese, Bengali, Malayalam, Kannada (Group~3, mixed non-Romance, paired a single high-resource language with three low-resource ones). The adaptation target is fixed to French and the goal is to test whether adaptation to French depends on linguistic similarity between meta-training and target languages, so both non-Romance groups were chosen to be distant from French, with the number of source languages fixed at four. We additionally vary the number of meta-training languages by sub-sampling Group~1 to its first two members (Romanian, Catalan) for the ``Group~1 (2 languages)'' bar in \Cref{fig:bloom7b1-languagegroups}~. Across three random seeds, the win rates of all three groups overlap within standard error, and Group~2 shows the highest point estimate despite being linguistically distant from French. This result is consistent with our theoretical bound \eqref{eq: clean asymptotic gap}, which depends on task-gradient variance $\sigma^2$ and the SFT preference-loss gap $R_{\rm SFT}$ rather than on linguistic distance to the target. Increasing the number of meta-training languages produces a small additional gain.\looseness=-1

\begin{figure}[t!]
    \centering
    \includegraphics[width=0.5\linewidth]{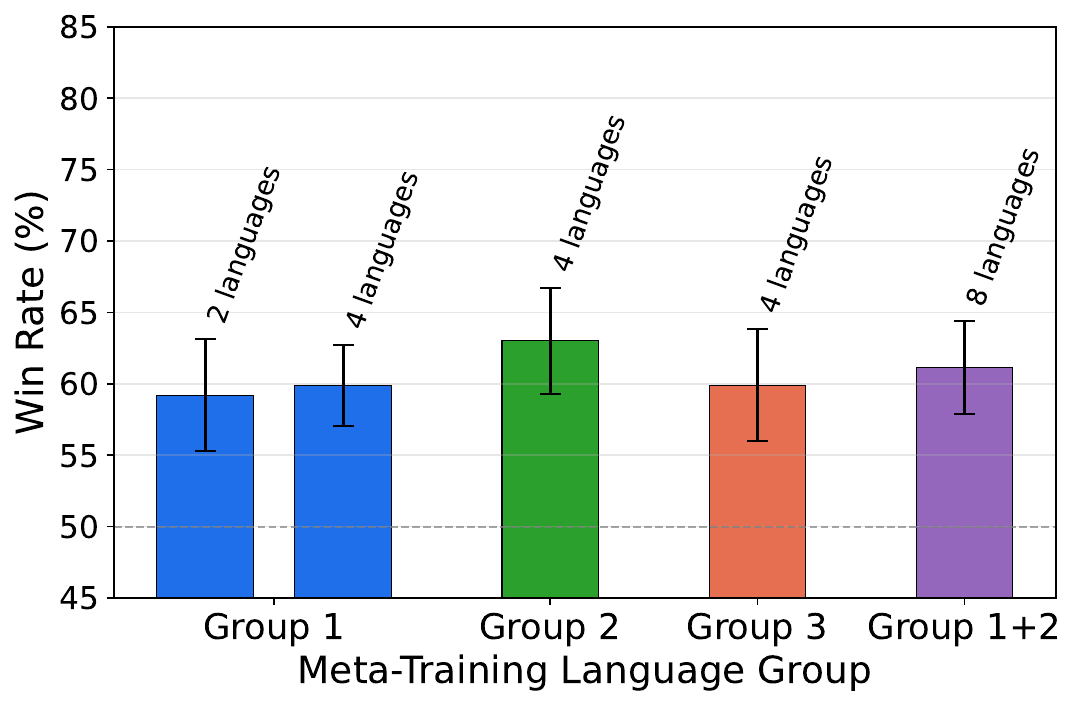}
        
    \caption{Win rates of BLOOM 7.1B models with different language groups for the MAML-DPO stage, adapted to 100 French samples. Neither the number of languages nor linguistic distance significantly affects adaptation performance. \looseness=-1}
    \label{fig:bloom7b1-languagegroups}
\end{figure}

\begin{table}[H]
\centering
\begin{minipage}[t]{0.47\linewidth}
    \centering
    \vspace{0pt}
    \begin{small}
    \begin{minipage}[t][7em][t]{\linewidth}
    \caption{Win rates of Gemma-270M models adapted to 4k and 40k samples from each target language. MAML-DPO models consistently outperform the baselines, followed by multitask method.\looseness=-1}
    \label{tab:WR_gemma_4k_40k_samples}
    \end{minipage}
    \begin{tabular}{cccc}
        \toprule
        Target & Algorithm & Samples &  WR(\%) \\
        \midrule
        \multirow{6}{*}{French} 
        
        & MAML-DPO & 4000 & $\mathbf{66.60}$ \\
        & Multitask-DPO & 4000 & $65.70$\\
          & DPO & 4000 &  $56.20$\\ 
         
        \cmidrule(lr){2-4}
        
        & MAML-DPO & 40000 & $\mathbf{75.40}$ \\ 
        & Multitask-DPO & 40000 &  $65.58$\\
          & DPO & 40000 & $54.00$ \\
         \midrule

         \multirow{6}{*}{Catalan}  
        
     & MAML-DPO & 4000 & $\mathbf{71.95}$\\
     & Multitask-DPO & 4000 & $70.20$ \\
          & DPO & 4000 & $61.35$ \\
         
        \cmidrule(lr){2-4}
        
         & MAML-DPO & 40000 & $\mathbf{74.50}$ \\
         & Multitask-DPO & 40000 & $70.30$ \\
          & DPO & 40000 & $62.90$ \\
    \midrule

         \multirow{6}{*}{Romanian}  
        
     & MAML-DPO & 4000 & $\mathbf{86.95}$ \\
     & Multitask-DPO & 4000 & $84.50$ \\
          & DPO & 4000 & $71.35$ \\
         
        \cmidrule(lr){2-4}
        
         & MAML-DPO & 40000 & $\mathbf{88.20}$ \\
         & Multitask-DPO & 40000 & $80.15$\\
          & DPO & 40000 & $75.30$ \\
        \bottomrule
    \end{tabular}
    \end{small}
\end{minipage}
\hfill
\begin{minipage}[t]{0.47\linewidth}
    \centering
    \vspace{0pt}
    \begin{small}
    \begin{minipage}[t][7em][t]{\linewidth}
    \caption{Win rate of BLOOM 7.1B models adapted to 8k/40k samples from each target language. MAML training stage also improves performance with a larger amount of data.}
    \label{tab:WR_bloom_DPO_8k_40k_samples}
    \end{minipage}
    \begin{tabular}{cccc}
        \toprule
        Target & Algorithm & Samples &  WR(\%) \\
        \midrule
        \multirow{6}{*}{French} 
        & MAML-DPO & 8000 & $\textbf{65.17}$ \\
        & Multitask-DPO & 8000 & $56.90$\\
          & DPO & 8000 &  $64.92$\\ 
        \cmidrule(lr){2-4}
        & MAML-DPO & 40000 & $71.4$ \\ 
        & Multitask-DPO & 40000 & $\textbf{73.74}$\\
        & DPO & 40000 & $67.4$ \\
        \midrule
         \multirow{6}{*}{Catalan}  
        & MAML-DPO & 8000 & $\textbf{66.75}$\\
        & Multitask-DPO & 8000 & $63.93$\\
        & DPO & 8000 & $59.46$ \\
        \cmidrule(lr){2-4}
        & MAML-DPO & 40000 & $68.16$ \\ 
        & Multitask-DPO & 40000 & $\textbf{71.08}$\\
        & DPO & 40000 & $68.16$ \\
    \midrule
        \multirow{6}{*}{Romanian}  
     & MAML-DPO & 8000 & $\textbf{64.47}$\\
     & Multitask-DPO & 8000 & $63.00$\\
        & DPO & 8000 & $59.12$ \\
        \cmidrule(lr){2-4}
        & MAML-DPO & 40000 & $67.29$ \\ 
        & Multitask-DPO & 40000 & $\textbf{69.56}$\\
        & DPO & 40000 & $65.27$ \\
        \bottomrule
    \end{tabular}
    \end{small}
\end{minipage}
\end{table}

\textbf{Do meta-trained models still have advantage under sufficient adaptation data?} 
In this setting, we adapt models using 4K/8K and 40K samples for MAML-DPO, while for MAML-RLHF we use 8K/12K and 40K samples. 
\Cref{tab:WR_gemma_4k_40k_samples} and \Cref{tab:WR_bloom_DPO_8k_40k_samples} show the performances of Gemma3 270M and BLOOM 7.1B DPO models when sufficient amount of adaptation samples are provided, from 4K/8K to 40K. While MAML-DPO outperforms both baselines across all three target languages with 4K/8K adaptation samples, we note that Multitask-DPO achieves higher win rates than MAML-DPO at 40K samples in BLOOM 7.1B and the performance gap between the MAML-DPO models and standard DPO narrows in many cases. This is expected
because the number of meta-training samples is less than 20K, whose advantageous effect on the
adaptation stage is supposed to diminish as the number of adaptation datapoints exceeds it.
MAML-RLHF in \Cref{tab:WR_gemma_12k_40k_RLHF} and \Cref{tab:WR_bloom_RLHF_8k_40k_samples} outperforms standard RLHF in every setting and is competitive with multitask-RLHF, which is consistent with multitask-RLHF being a strong baseline when the target language is closely aligned with the Romance meta-training tasks. These results show that meta-training accelerates early adaptation while preserving competitive performance with sufficient target data.\looseness=-1

\section{Related work}\label{sec:related_work}
While most LLMs possess multilingual capabilities, their performance across languages is unequal \citep{bawden2023investigating, thellmann2024towards, dong2025evaluating, WANG2026104616}. Previous works have investigated improving multilingual capability of LLMs by curating data \citep{chen2023tigerbotopenmultilingualmultitask, yue2024pangea, kallappa2025krutrimllmmultilingualfoundational}, adopting a new objective during pretraining \citep{apertus2025apertusdemocratizingopencompliant}, in-context learning \citep{li2024language,cahyawijaya2024llms, li2024improving}, and fine-tuning on multilingual datasets \citep{dac2023okapi, lai2024llms}.
Improving few-shot capabilities of deep neural networks have been investigated from meta-learning perspectives \citep{vinyals2016matching, finn2017model, hospedales2021meta}. Application of meta-learning to LLMs have been studied, but most of them discuss in-context learning scenarios \citep{chen-etal-2022-meta, codaforno2023metaincontextlearninglargelanguage, sinha2024maml, li2024metaincontextlearningmakes,  hili2025llmsincontextmetalearnersmodel, li2024mendmetademonstrationdistillation}. \citet{hu2023meta} meta-train a separate autoregressive model to train importance weighting of tokens in a document, which is then used to fine-tuning an LLM. A complementary line of work addresses multilingual preference alignment via gradient-level interventions, filtering conflicting per-language gradients during joint preference training to mitigate negative interference across languages \citep{li2026congrad}.  Our approach does not modify the gradient of joint training but replace it with a meta-objective that explicitly optimises for fast \emph{adaptation} to a held-out target language. To the best of our knowledge, our work is the first approach to addressing under-representation of low-resource languages by applying meta-learning to the preference learning pipeline, providing theoretical guarantees and demonstrating improved multilingual capabilities. See \Cref{appendix: additional related work} for further discussion on related work.\looseness=-1

\section{Conclusion}

We presented a meta-learning framework for aligning multilingual LLMs using human preference data. 
By treating each language as a task, we develop MAML-RLHF and MAML-DPO as two meta-learning algorithms that leverage preference data across languages to enable efficient adaptation to low-resource languages.
We provided theoretical guarantees and empirically demonstrated substantial improvements in sample efficiency across multiple languages and model sizes. In particular, meta-trained models achieve higher win rates than standard RLHF and DPO baselines in low-resource languages, while maintaining or improving performance as more adaptation data become available. 
Importantly, these gains remain consistent across varying language distances between meta-training and target languages, and are not sensitive to the choice of meta-language combinations.
Overall, our work demonstrates meta-learning as an effective way to address data imbalance in multilingual alignment. While we focused on preference learning, extending our approach to reinforcement learning stages and reducing end-to-end pipeline compute are promising directions for future work.

\section*{Acknowledgments}
We thank the anonymous reviewers for their helpful comments and suggestions. Jiaying Lin was supported by the China Scholarship Council under Grant No. 202408310075.

\clearpage

\newpage

\bibliographystyle{plainnat}
\bibliography{refs}
\newpage

\appendix
\addcontentsline{toc}{section}{Appendix}

\part{Appendix}
\parttoc
\input{appendix}


\newpage
\input{checklist.tex}

\end{document}

%% file: arxiv_preamble.tex
\usepackage{etoolbox}

\usepackage{savesym,multirow,url,xspace,graphicx,hyperref,enumitem,color}
\usepackage[toc,page,header]{appendix}
\usepackage{minitoc}

\usepackage{dsfont}
\usepackage{multicol}
\usepackage{dsfont}
\usepackage{setspace}

\usepackage{tikz}
\usetikzlibrary{automata, positioning, arrows}
\tikzset{
	->, 
	every state/.style={thick, fill=gray!10}, 
	initial text=$ $, 
}
\usetikzlibrary{arrows.meta}
\usepackage{pgfplots}
\pgfplotsset{width=10cm,compat=1.9}
\usepackage{diagbox}
\usepackage{caption}

\makeatletter
\newif\if@restonecol
\makeatother
\usepackage{amsmath,amssymb,xfrac,amsthm}
\usepackage{mathtools}
\usepackage{wrapfig}
\setitemize{noitemsep,topsep=1pt,parsep=1pt,partopsep=1pt, leftmargin=12pt}
\makeatletter

\newcommand{\Rmnum}[1]{\expandafter\@slowromancap\romannumeral #1@}
\makeatother
\usepackage{caption}

\usepackage{enumitem}
\usepackage{wasysym}

\DeclareMathOperator*{\argmin}{arg\,min}
\DeclareMathOperator*{\argmax}{arg\,max}

\usepackage{xcolor}

\usepackage{bm}
\usepackage{bbm}
\usepackage{cleveref}

\newcommand{\norm}[1]{\left\lVert#1\right\rVert}

\newcommand{\abs}[1]{\left|#1\right|}

\newcommand{\set}[1]{\left\{ #1 \right\}}

\renewcommand{\bar}{\overline}

\newcommand{\E}{\mathbb{E}}
\newcommand{\calB}{\mathcal{B}}
\newcommand{\calI}{\mathcal{I}}
\newcommand{\calL}{\mathcal{L}}

\newcommand{\calD}{\mathcal{D}}

\newcommand{\R}{\mathbb{R}}

\usepackage{xargs}                      
\PassOptionsToPackage{pdftex,dvipsnames}{xcolor}  
\newcommandx{\task}[2][1=]{\todo[linecolor=red,backgroundcolor=red!25,bordercolor=red,#1]{#2}}
\newcommandx{\change}[2][1=]{\todo[linecolor=blue,backgroundcolor=blue!25,bordercolor=blue,#1]{#2}}
\newcommandx{\info}[2][1=]{\todo[linecolor=green,backgroundcolor=green!25,bordercolor=green,#1]{#2}}

\usepackage[toc,page,header]{appendix}
\usepackage{minitoc}

\newtheorem{theorem}{Theorem}[section]

\newtheorem{lemma}[theorem]{Lemma}

\theoremstyle{definition}

\newtheorem{assumption}[theorem]{Assumption}
\theoremstyle{remark}
\newtheorem{remark}[theorem]{Remark}

\newtheorem*{statement*}{Statement}

%% file: appendix.tex
\appendix
\newpage

\section{Further Details of Experimental Setting}
\label{appendix: Further Details of Experimental Setting}
In this section, we explain the remaining details of the experimental setting.

\textbf{Datasets.} For each language task, the Okapi dataset provides 52K instruction–response pairs for supervised fine-tuning, 42K preference comparisons for reward modeling or DPO, and 64K instructions for policy optimization. 
For meta-training with Okapi, we use 12K samples from each language for DPO and 16K samples from each language for reward modeling. 
For the target languages, we use 100 or more samples for each of SFT, reward modeling, and policy optimization stage. Unless specified, MAML-based or Multitask- approaches use one of French, Italian, Spanish, Catalan, Romanian as the target language, while the others are used for MAML/Multitask training stage.
For \texttt{multilingual/orca\_dpo\_pairs}, we use 3 sets of 1K samples for each of SFT, preference optimization for policies, preference learning for judge reward functions. 2 sets of 200 samples are used for each of preference learning validation and win rate evaluation using judge reward models. For MAML/Multitask training, we use Arabic, German, Russian and Chinese, while using the remaining French, Spanish, Turkish for adaptation.

\begin{wrapfigure}{r}{0.5\textwidth}
    \includegraphics[width=\linewidth]{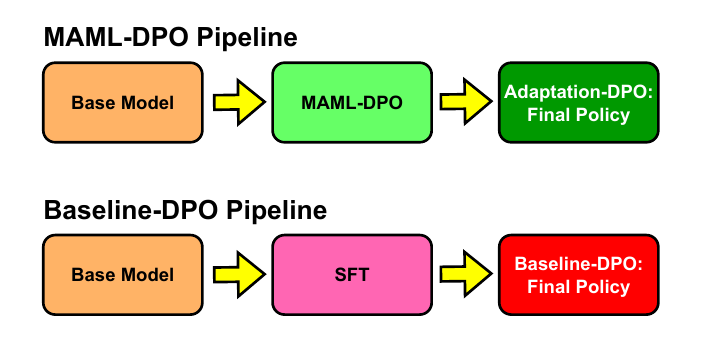}
    \caption{Visualization of DPO training pipelines. MAML-DPO pipeline does not use SFT, instead initializing the Adaptation-DPO policy with MAML-DPO algorithm on several non-target languages with sufficient data.}
    \label{fig:pipeline_DPO}
\end{wrapfigure}

\textbf{MAML-DPO.} The base model is directly trained with MAML-DPO using 4 out of 5 languages being considered in \Cref{sec:experiment_settings}. 
For BLOOM with 7.1B parameters, MAML-DPO is trained for 400 steps, while each step accumulates gradients over 40 iterations of choosing a task and sampling 6 datapoints. We use $3 \times 10^{-5}$ for learning rate, while 3 is multiplied to the learning rate for inner loop updates. For each outer loop update, we use 2 consecutive inner loop steps. Once meta-training is complete, we apply adaptation DPO. We use batch size 20 and 2 gradient accumulation steps with $1 \times 10^{-5}$ as the learning rate. For experiments with 100 and 8K samples, we train the models over 200 steps and use the model showing the lowest validation loss. For experiments with 40K samples, we train the model for 2K steps.

For Gemma-3 with 270M parameters, MAML-DPO is meta-trained for 300 steps. At each outer step, the model accumulates gradients over 40 iterations by selecting 2 tasks and sampling 2 data points per task. We use a single inner loop step for each outer loop update. The learning rate configuration is chosen from the pairs $(5 \times 10^{-5},\, 5)$ and $(1 \times 10^{-5},\, 3)$ for the outer learning rate and inner-loop learning coefficient.
After meta-training, the model is adapted using DPO. For experiments with 100 samples, adaptation is performed for 100 steps using a batch size of 2 with 20 gradient accumulation steps. For experiments with 4k and 40k samples, adaptation is trained over 1000 steps with a batch size of 2, using 2 and 20 gradient accumulation steps, respectively. The adaptation learning rate is selected between $1 \times 10^{-4}$ and $ 5 \times 10^{-6}$.

We visually describe the pipeline of our DPO training pipeline in \Cref{fig:pipeline_DPO}.

\begin{figure}[h!]
    \centering
    \includegraphics[width=0.45\linewidth]{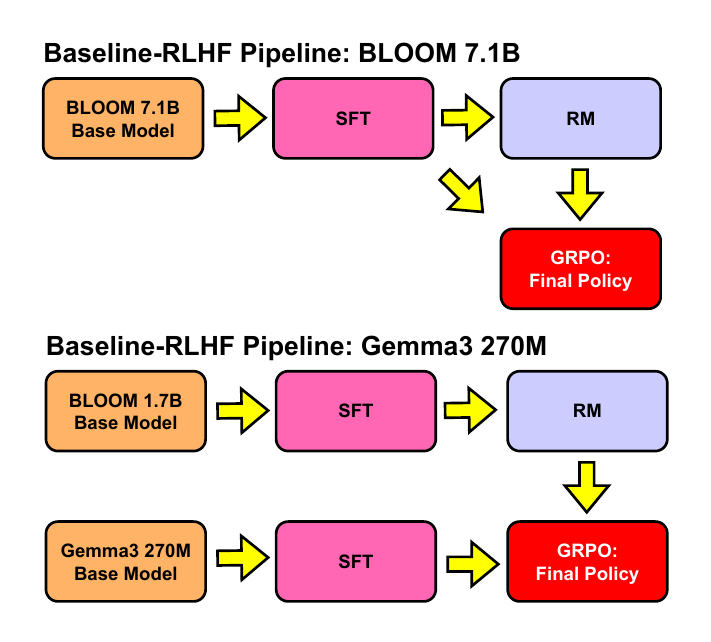}
    \caption{Visualization of baseline RLHF pipeline for both model architectures. For Gemma3 270M policies, BLOOM 1.7B is used for initializing the reward function.}
    \label{fig:pipeline_baselineRLHF}
\end{figure}

\textbf{MAML-RLHF.} Before reward modeling, MAML-SFT is first applied to the base model, followed by MAML-RM and then adaptation reward modeling.  
For BLOOM with 7.1B parameters, we use 100 outer steps for MAML-SFT with $3 \times 10^{-5}$ as the learning rate and multiplying 3 to it for a single inner loop update. Each outer step includes 40 gradient accumulation steps of sampling 2 datapoints from each of 2 tasks. MAML-RM uses the same learning rate and datapoint sampling hyperparameters as MAML-SFT, except that MAML-RM is run for 400 steps. 
Adaptation reward modeling fine-tunes the MAML-RM trained model with $1 \times 10^{-5}$ as the learning rate and 40 as the batch size without gradient accumulation. For 100 samples, we use 100 steps of reward modeling while we use 200 steps for 8K samples. 
The policy is first trained with SFT, using $1\times 10^{-4}$ as the learning rate and 40 as the batch size without gradient accumulation. For experiments with 8K samples, SFT is run for 200 steps.
The final fine-tuning of MAML-RLHF pipeline is applying GRPO to the SFT model using the adapted reward model. We use 400 prompts for training, while each of 100 training steps consists of generating 4 responses from each prompt. We use batch size 4 and 4 gradient accumulation steps with learning rate of $1\times10^{-4}$.

For Gemma 3 with 270M parameters, MAML-SFT is trained for 3000 outer steps with an outer learning rate of $5 \times 10^{-4}$, using an inner-loop learning rate scaled by a factor of 3. Each outer step samples 2 data points from each of 2 tasks per step, without extra accumulation steps. Initialized from the MAML-SFT model, the SFT policy is fine-tuned with a learning rate of $1 \times 10^{-4}$ and a batch size of 4 in 3000 steps, without extra accumulation steps. MAML-RM is trained for 300 steps with 40 gradient accumulation steps, sampling 4 data points from each of 2 tasks per accumulation step. The learning rate configuration is selected from $(5 \times 10^{-4},\, 3)$ and $(3 \times 10^{-4},\, 1)$ for the outer learning rate and inner-loop coefficient. We use 8 data points with 5 gradient accumulation steps for adapting reward models. It is trained for 100, 300, and 1000 steps in the 100-, 12k-, and 40k-sample experiments, respectively, with adaptation learning rates selected between $5 \times 10^{-4}$ and $5 \times 10^{-5}$. Using the adapted reward model for the target language, the SFT policy is further optimized via GRPO for 1000 steps with batch size 4 and no gradient accumulation, where the GRPO learning rate is selected between $5 \times 10^{-5}$ and $5 \times 10^{-6}$.

Visualization of the baseline RLHF pipeline for both models is presented in \Cref{fig:pipeline_baselineRLHF}. See \Cref{fig:pipeline_MRLHF} for the MAML-RLHF pipelines for each model architecture.

\begin{figure}
    \centering
    \includegraphics[width=0.55\linewidth]{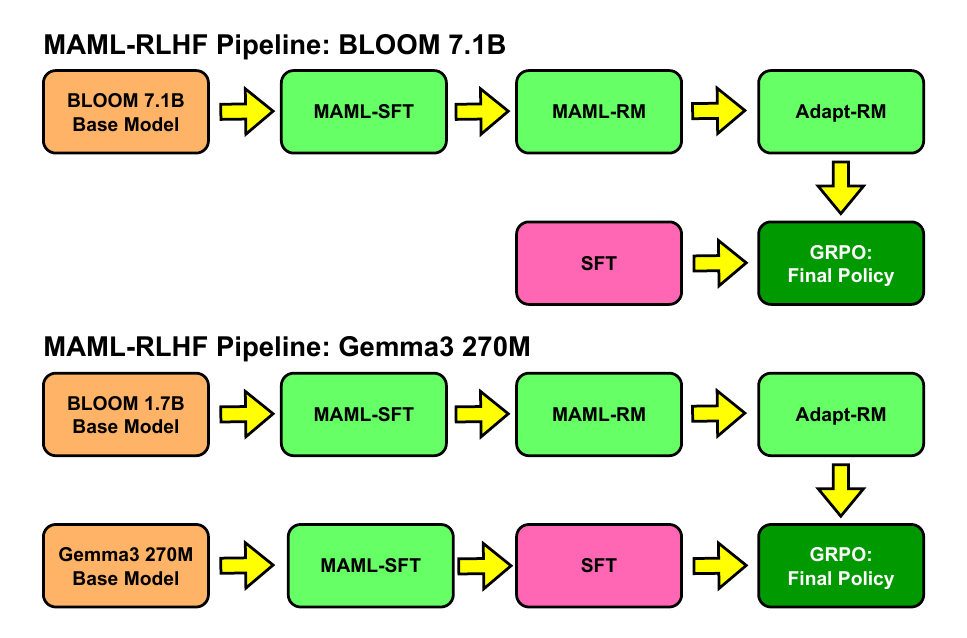}
    \caption{Visualization of MAML-RLHF pipeline for BLOOM 7.1B model and Gemma3 270M model. The main difference between the models is that for Gemma3 270M model, we use a separate BLOOM 1.7B model as the base model of reward functions.}
    \label{fig:pipeline_MRLHF}
\end{figure}

\textbf{Compute Resources.} For experiments with BLOOM 7.1B models, we use one NVIDIA H200 GPU. For experiments with Gemma3 270M models, we use NVIDIA Quadro RTX 6000 GPU.

\newpage
\section{Additional Experimental Results}
\label{appendix: additional experimental results}

\subsection{Additional Results of MAML-DPO with LoRA adapters}
\paragraph{Limited adaptation data.} Based on Gemma3 270M, \Cref{tab:winrate_DPO_100samples_gemma_cross_language_family} provides the numerical results of adaptation performance for MAML-DPO using only 100 samples from non-Romance target languages, where MAML-DPO substantially outperforms both the multitask and standard DPO baselines. Since our meta-DPO policies are trained with 4 Romance languages, these results suggest that the proposed meta-learning approach can retain its adaptation advantage even when the target languages are linguistically distant from the meta-training tasks. This is particularly important because larger linguistic distances can reduce the transferability of shared multilingual representations, thereby weakening the effectiveness of standard multitask training.

\begin{table}[H]
\centering
\begin{small}
\centering
\caption{Win rate of Gemma3 270M adapted to 100 samples from each non-Romance target language in the DPO pipeline.}
\label{tab:winrate_DPO_100samples_gemma_cross_language_family}
\begin{tabular}{cccc}
    \toprule
    Target & Algorithm & Samples & WR (\%) \\
    \midrule
    \multirow{3}{*}{Arabic} 
        & MAML-DPO & 100 & $\mathbf{75.36} \pm 1.39$ \\
        & Multitask & 100 & $61.07 \pm 0.54$ \\
        & DPO & 100 & $53.72 \pm 3.35$ \\
    \midrule
    \multirow{3}{*}{Slovak} 
        & MAML-DPO & 100 & $\mathbf{79.68} \pm 2.11$ \\
        & Multitask & 100 & $70.00 \pm 0.12$ \\
        & DPO & 100 & $64.03 \pm 2.55$ \\
    \midrule
    \multirow{3}{*}{Indonesian} 
        & MAML-DPO & 100 & $\mathbf{76.91} \pm 4.90$ \\
        & Multitask & 100 & $63.57 \pm 1.76$ \\
        & DPO & 100 & $48.95 \pm 3.31$ \\
    \midrule
    \multirow{3}{*}{Malayalam} 
        & MAML-DPO & 100 & $\mathbf{64.95} \pm 0.43$ \\
        & Multitask-DPO & 100 & $57.25 \pm 1.45$ \\
        &  DPO & 100 & $54.03 \pm 2.72$ \\
    \midrule
    \multirow{3}{*}{Kannada} 
        & MAML-DPO & 100 & $\mathbf{78.37} \pm 0.07$ \\
        & Multitask-DPO & 100 & $57.40 \pm 0.83$ \\
        &  DPO & 100 & $59.27 \pm 3.51$ \\
    \bottomrule
\end{tabular}
\end{small}
\end{table}

\Cref{tab:winrate_DPO_100samples_multitask_gemma} and \Cref{tab:winrate_DPO_100samples_multitask_bloom} present the adaptation performance of MAML-DPO on Romance target languages based on Gemma with both Gemma 270M and Bloom 7.1B. Across all three target languages, MAML-DPO consistently demonstrates strong adaptation ability and substantially outperforms conventional DPO. Compared with the multitask baseline, MAML-DPO achieves either comparable or better performance. The ranking among these three algorithms is preserved after controlling for response length in the Romance-language evaluations. This indicates that our method remains competitive even when the target language is linguistically close to the meta-training languages, where multitask training is expected to be a relatively strong baseline.


\begin{table}[H]
    \centering
    \small
    \caption{Raw and length-controlled win rates of Gemma3 270M adapted to 
    100 samples from each target language in the DPO pipeline. 
    $\Delta$ denotes (LC-WR - raw WR), and $\Delta$length denotes response-length difference between the adapted model and its paired comparison model in tokens.}
    \label{tab:winrate_DPO_100samples_multitask_gemma}
    \begin{tabular}{llccccc}
        \toprule
        Target & Algorithm & Samples & Raw WR (\%) & LC-WR (\%) 
        & $\Delta$ & $\Delta$length \\
        \midrule

        \multirow{3}{*}{French}
        & MAML-DPO 
        & 100 
        & $\mathbf{66.88} \pm 1.47$ 
        & $\mathbf{67.03}$ 
        & $+0.15$ 
        & $+9.8$ \\

        & Multitask-DPO 
        & 100 
        & $66.28 \pm 0.84$ 
        & $65.38$ 
        & $-0.90$ 
        & $+11.6$ \\

        & DPO baseline 
        & 100 
        & $53.81 \pm 5.76$ 
        & $53.78$ 
        & $-0.03$ 
        & $+3.8$ \\

        \midrule

        \multirow{3}{*}{Catalan}
        & MAML-DPO 
        & 100 
        & $\mathbf{70.55} \pm 0.72$ 
        & $\mathbf{74.23}$ 
        & $+3.68$ 
        & $-4.1$ \\

        & Multitask-DPO 
        & 100 
        & $69.75 \pm 0.29$ 
        & $72.18$ 
        & $+2.43$ 
        & $+0.0$ \\

        & DPO baseline 
        & 100 
        & $63.45 \pm 1.76$ 
        & $58.98$ 
        & $-4.47$ 
        & $+29.3$ \\

        \midrule

        \multirow{3}{*}{Romanian}
        & MAML-DPO 
        & 100 
        & $\mathbf{79.83} \pm 0.26$ 
        & $\mathbf{78.93}$ 
        & $-0.90$ 
        & $-43.8$ \\

        & Multitask-DPO 
        & 100 
        & $59.51 \pm 0.94$ 
        & $58.63$ 
        & $-0.88$ 
        & $-11.2$ \\

        & DPO baseline 
        & 100 
        & $51.56 \pm 2.97$ 
        & $51.85$ 
        & $+0.28$ 
        & $-3.6$ \\

        \bottomrule
    \end{tabular}
\end{table}

\begin{table}[H]
    \centering
    \small
    \caption{Raw and length-controlled win rates of BLOOM-7.1B adapted to 
    100 samples from each target language in the DPO pipeline. 
    $\Delta$ denotes (LC-WR - raw WR), and $\Delta$length denotes response-length difference between the adapted model and its paired comparison model in tokens.}
    \label{tab:winrate_DPO_100samples_multitask_bloom}
    \begin{tabular}{llccccc}
        \toprule
        Target & Algorithm & Samples & Raw WR (\%) & LC-WR (\%) 
        & $\Delta$ & $\Delta$length \\
        \midrule

        \multirow{3}{*}{French}
        & MAML-DPO 
        & 100 
        & $\mathbf{59.90}$ 
        & $\mathbf{61.02}$ 
        & $+1.13$ 
        & $-19.2$ \\

        & Multitask-DPO 
        & 100 
        & $58.48$ 
        & $59.63$ 
        & $+1.15$ 
        & $-19.6$ \\

        & DPO baseline 
        & 100 
        & $52.05$ 
        & $52.19$ 
        & $+0.14$ 
        & $-5.8$ \\

        \midrule

        \multirow{3}{*}{Catalan}
        & MAML-DPO 
        & 100 
        & $\mathbf{55.85}$ 
        & $\mathbf{57.34}$ 
        & $+1.49$ 
        & $-22.4$ \\

        & Multitask-DPO 
        & 100 
        & $53.57$ 
        & $52.45$ 
        & $-1.12$ 
        & $+5.3$ \\

        & DPO baseline 
        & 100 
        & $50.37$ 
        & $49.93$ 
        & $-0.44$ 
        & $+2.6$ \\

        \midrule

        \multirow{3}{*}{Romanian}
        & MAML-DPO 
        & 100 
        & $\mathbf{62.75}$ 
        & $\mathbf{63.26}$ 
        & $+0.51$ 
        & $-8.0$ \\

        & Multitask-DPO 
        & 100 
        & $58.15$ 
        & $59.18$ 
        & $+1.03$ 
        & $-19.0$ \\

        & DPO baseline 
        & 100 
        & $57.33$ 
        & $57.64$ 
        & $+0.32$ 
        & $-6.4$ \\

        \bottomrule
    \end{tabular}
\end{table}

\begin{table}[H]
    \centering
    \small
    \caption{Refusal rates (\%) of BLOOM-7.1B and Gemma3-270M in the 
    100-sample setting, evaluated over 1,000 prompts per target language. 
    Refusals are classified by GPT-5.4-mini. All results are from a single seed.}
    \label{tab:refusal_rate_100samples}
    \begin{tabular}{llcccc}
        \toprule
        Model & Target 
        & MAML-DPO 
        & Multitask-DPO 
        & DPO baseline 
        & Base model \\
        \midrule

        \multirow{3}{*}{BLOOM-7.1B}
        & French  
        & 0.81 
        & 0.81 
        & 1.71 
        & 0.00 \\

        & Catalan 
        & 1.61 
        & 0.71 
        & 3.73 
        & 2.72 \\

        & Romanian 
        & 0.10 
        & 0.30 
        & 0.20 
        & 0.10 \\

        \midrule

        \multirow{3}{*}{Gemma3-270M}
        & French  
        & 0.50 
        & 0.40 
        & 0.20 
        & 0.20 \\

        & Catalan 
        & 0.00 
        & 0.00 
        & 0.00 
        & 0.30 \\

        & Romanian 
        & 0.00 
        & 0.40 
        & 0.20 
        & 0.20 \\

        \bottomrule
    \end{tabular}
\end{table}

In \Cref{tab:refusal_rate_100samples}, refusal rates are low across both model families, remaining below 4\% for BLOOM-7.1B and below 1\% for Gemma3-270M. For BLOOM-7.1B, MAML-DPO refuses less often than the DPO baseline in all three target languages, while in Catalan both MAML-DPO and Multitask-DPO also show lower refusal rates than the base model. For Gemma3-270M, refusal rates remain similarly small, with no systematic increase under MAML-DPO relative to the other methods or the base model. These results suggest that the higher win rates achieved by MAML-DPO are not driven by increased refusal behaviour.

\paragraph{Out-of-distribution evaluation.} To examine robustness beyond the training distribution, we follow the Okapi evaluation protocol in \citet{dac2023okapi} and report zero-shot performance on ARC \cite{clark2018think}, HellaSwag \cite{zellers-etal-2019-hellaswag}, and MMLU \cite{hendryckstest2021}, which are multiple-choice benchmarks and assess robustness to out-of-distribution shifts in format and content.

In \Cref{tab:OOD_benchmark_results}, we evaluate BLOOM 7.1B models trained with DPO algorithms in OOD benchmarks. While both MAML-DPO and baseline DPO models have used 10K samples for adaptation to the target language, baseline DPO models show higher accuracy than the base model in fewer columns than MAML-DPO. This shows the robustness of MAML-DPO training pipeline to OOD evaluations.

We further provide OOD benchmark evaluation of DPO models based on Gemma3 270M architecture in \Cref{tab:OOD_benchmark_results_DPO_gemma}. Similarly to \Cref{tab:OOD_benchmark_results}, MAML-DPO models show improvements in more benchmarks and metrics than baseline DPO models. This indicates that MAML-DPO provides stable improvement across different model architectures and data domains.

\begin{table}[H]
    \centering
    \tiny
    \caption{Out-of-distribution benchmark evaluation results of BLOOM 7.1B models. Bold numbers indicate where the trained model outperformed the base model. With 10k samples given for adaptation, MAML-DPO maintains its robustness to OOD datapoints and gains improvement in more diverse benchmarks and metrics than the baseline method.}
    \begin{tabular}{cccc c c c c c}
        \toprule
        \multirow{2}{*}{Language} & \multirow{2}{*}{Algorithm} & \multirow{2}{*}{Samples} & \multicolumn{2}{c}{ARC} & \multicolumn{2}{c}{Hellaswag} & \multicolumn{2}{c}{MMLU} \\
        \cmidrule(lr){4-9}
         &  &  & Acc. & Norm. Acc. & Acc. & Norm. Acc.& Acc. & Norm. Acc. \\
        \midrule
        \multirow{3}{*}{French} & Base model & - & 0.3293 & 0.3644 & 0.4259 & 0.5452 & 0.2884 & 0.297 \\
        & MAML-DPO & 10000 & \textbf{0.3362} & \textbf{0.3704} & \textbf{0.434} & \textbf{0.5464} & \textbf{0.2885} & \textbf{0.2977} \\
        & DPO & 10000 & 0.3251 & \textbf{0.3721} & 0.4339 & \textbf{0.5477} & 0.2820 & 0.2887 \\
        \midrule
        \multirow{3}{*}{Catalan} & Base model & - & 0.3199 & 0.3482 & 0.3956 & 0.4992 & 0.2796 & 0.2887 \\
        & MAML-DPO & 10000 & \textbf{0.3242} & \textbf{0.3611} & \textbf{0.401} & 0.4981 & \textbf{0.2811} & \textbf{0.2898} \\
        & DPO & 10000 & \textbf{0.3344} & \textbf{0.3670} & \textbf{0.3981} & 0.4958 & 0.2730 & 0.2829 \\
        \midrule
        \multirow{3}{*}{Romanian} & Base model & - & 0.2091 & 0.2708 & 0.2814 & 0.3186 & 0.2557 & 0.2736 \\
        & MAML-DPO & 10000 & \textbf{0.2134} & 0.2639 & \textbf{0.2864} & \textbf{0.3254} & \textbf{0.2599} & \textbf{0.2739} \\
        & DPO & 10000 & 0\textbf{.2125} & 0.2699 & \textbf{0.2864} & \textbf{0.3266} & \textbf{0.2596} & 0.2719 \\
        \bottomrule
    \end{tabular}
    \label{tab:OOD_benchmark_results}
\end{table}

\begin{table}[h]
    \onecolumn
    \centering
    \tiny
    \caption{Out-of-distribution (OOD) benchmark evaluation results of Gemma 270M models under baseline and MAML-DPO. Bold numbers indicate where the trained model outperformed the base model.}
    \begin{tabular}{cccc c  c c  c c}
        \toprule
        \multirow{2}{*}{Language} & \multirow{2}{*}{Algorithm} & \multirow{2}{*}{Samples} & \multicolumn{2}{c}{ARC} & \multicolumn{2}{c}{Hellaswag} & \multicolumn{2}{c}{MMLU} \\
        \cmidrule(lr){4-9}
         &  &  & Acc. & Norm. Acc. & Acc.  & Norm. Acc.& Acc. & Norm. Acc. \\
        \midrule
        \multirow{3}{*}{French} & Base model & - & 0.210 & 0.249 & 0.271 & 0.261 & 0.251 & 0.266 \\
        & MAML-DPO & 40000 & \textbf{0.222} & \textbf{0.253} & \textbf{0.274 }&\textbf{ 0.263 }& 0.250 &0.261\\
        & DPO & 40000 & \textbf{0.212 }& 0.243 & 0.271 & 0.259 & 0.249 & 0.262 \\
        \midrule
        \multirow{3}{*}{Catalan} & Base model & - & 0.212 & 0.239 & 0.265 & 0.281 & 0.242 & 0.252 \\
        & MAML-DPO & 40000 & \textbf{0.214} & 0.225 & \textbf{0.282} & \textbf{0.29} & \textbf{0.257} & \textbf{0.257}\\
        & DPO & 40000 & 0.200 & 0.239 & \textbf{0.274} & 0.281 & 0.24 & 0.244 \\
        \midrule
        \multirow{3}{*}{Romanian} & Base model & - & 0.191 & 0.240 &0.283  & 0.276 & 0.260 & 0.268 \\
        & MAML-DPO & 40000 & \textbf{0.194 }& \textbf{0.246} & \textbf{0.294} & 0.259 & \textbf{0.262} & 0.264 \\
        & DPO & 40000 & 0.186 & 0.233 &\textbf{0.284} &0.263 &0.256 &0.259 \\
        \bottomrule
    \end{tabular}
    \label{tab:OOD_benchmark_results_DPO_gemma}
\end{table}

\begin{table}[H]
    \centering
    \small
    \caption{Accuracy on the French subsets of MGSM and MMMLU in the 
    100-sample setting. BLOOM-7.1B results are averaged over three random seeds.}
    \label{tab:mgsm_mmmlu}
    \begin{tabular}{llcc}
        \toprule
        Model & Method & MGSM (Acc. \%) & MMMLU (Acc. \%) \\
        \midrule

        \multirow{4}{*}{Gemma-270M}
        & MAML-DPO       & $\mathbf{3.20}$ & $\mathbf{28.90}$ \\
        & Multitask-DPO  & 0.40             & 26.50 \\
        & DPO baseline   & 1.20             & 25.90 \\
        & Base model     & 1.60             & 26.60 \\

        \midrule

        \multirow{4}{*}{BLOOM-7.1B}
        & MAML-DPO       & 4.27             & 28.68 \\
        & Multitask-DPO  & 4.00             & $\mathbf{28.70}$ \\
        & DPO baseline   & 4.00             & 28.48 \\
        & Base model     & $\mathbf{4.40}$  & 28.15 \\

        \bottomrule
    \end{tabular}
\end{table}

We additionally evaluate our methods on MGSM and MMMLU and results are shown in \Cref{tab:mgsm_mmmlu}. For Gemma-270M, MAML-DPO obtains the highest accuracy on both benchmarks, indicating that few-shot adaptation does not sacrifice OOD capability. For BLOOM-7.1B the three methods are statistically indistinguishable: the spread between them is at most 0.3\% against a standard error near 1.0\%, and every method sits within 0.6\% of the untrained model on both benchmarks — marginally above it on MMMLU, marginally below on MGSM.

\subsubsection{Text generation examples}
We further provide a representative example to illustrate the qualitative
differences between adaptation methods at the Gemma-3-270M scale. The prompt is
in Romanian:

\begin{tcolorbox}[
    colback=gray!8,
    colframe=gray!45,
    boxrule=0.4pt,
    arc=2pt,
    left=6pt,
    right=6pt,
    top=5pt,
    bottom=5pt
]
\textbf{Prompt (Romanian):}
\textit{``Sugerați 3 strategii de marketing pe care companiile ar trebui să
implementeze pentru a crește conversiile clienților pe site-ul lor.''}

\smallskip
\textit{English translation: ``Suggest three marketing strategies companies
should implement to increase customer conversions on their website.''}
\end{tcolorbox}

We preserve the original wording of the model outputs, remove HTML tags, and
use [...] to indicate omitted text.

\begin{tcolorbox}[
    colback=gray!8,
    colframe=gray!45,
    boxrule=0.4pt,
    arc=2pt,
    title=\textbf{Base model (Gemma-3-270M)},
    coltitle=black,
    colbacktitle=gray!18,
    fonttitle=\bfseries,
    left=6pt,
    right=6pt,
    top=5pt,
    bottom=5pt
]
\textit{``1. Provide three strategies on how your company should achieve this
2. How will you implement it? [...] The first step is to research your product
[...]''}

\smallskip
The base model ignores the Romanian instruction, responds in English, and
largely returns questions or generic guidance rather than three concrete
marketing strategies.
\end{tcolorbox}

\begin{tcolorbox}[
    colback=gray!8,
    colframe=gray!45,
    boxrule=0.4pt,
    arc=2pt,
    title=\textbf{DPO baseline},
    coltitle=black,
    colbacktitle=gray!18,
    fonttitle=\bfseries,
    left=6pt,
    right=6pt,
    top=5pt,
    bottom=5pt
]
\textit{``Strategii de marketing includ strategiile de marketing corectă.
[...] personalizarea proiectelor și meniuul online [...]''}

\smallskip
The model attempts to answer in Romanian, but the output is grammatically
degraded and does not provide three clearly distinguishable,
conversion-oriented strategies.
\end{tcolorbox}

\begin{tcolorbox}[
    colback=gray!8,
    colframe=gray!45,
    boxrule=0.4pt,
    arc=2pt,
    title=\textbf{Multitask-DPO},
    coltitle=black,
    colbacktitle=gray!18,
    fonttitle=\bfseries,
    left=6pt,
    right=6pt,
    top=5pt,
    bottom=5pt
]
\textit{``Strategii de marketing includ trei tipuri de marketing:
Marketing înlocuitor / Marketing reprezentant / Marketing personalizare
[...] precum festivaluri, concerte [...]''}

\smallskip
Multitask-DPO better follows the requested three-item structure, but several
proposed strategies are semantically inappropriate for website conversion and
the response drifts toward unrelated offline activities.
\end{tcolorbox}

\begin{tcolorbox}[
    colback=gray!8,
    colframe=gray!45,
    boxrule=0.4pt,
    arc=2pt,
    title=\textbf{MAML-DPO (ours; 100 Romanian examples)},
    coltitle=black,
    colbacktitle=gray!18,
    fonttitle=\bfseries,
    left=6pt,
    right=6pt,
    top=5pt,
    bottom=5pt
]
\textit{``1. Promovă o socializare inteligente. [...] creșterea conversiției
de clienți [...] Promovând newsletter clienților.''}

\smallskip
MAML-DPO remains primarily in Romanian and makes a more explicit connection
between online promotional actions, newsletters, and customer conversion,
although malformed phrasing and repetition remain.
\end{tcolorbox}

Overall, this example illustrates the relative behaviour of the four models:
the base model fails target-language following, DPO produces vague and degraded
Romanian, Multitask-DPO better follows the requested format but proposes
semantically unsuitable strategies, while MAML-DPO more directly connects its
recommendations to conversion-oriented actions after adaptation on only 100
Romanian examples. This qualitative example is intended as an illustration and our aggregate evaluation remains the primary evidence.

\subsection{Additional Results of MAML-RLHF with LoRA adapters}\label{appendix:RLHF_results}

\paragraph{Limited adaptation data.}

\Cref{tab:winrate_RLHF_100samples_multitask_gemma} reports the win rates of the corresponding policies optimized with these reward models in the RLHF pipeline. MAML-RLHF achieves the highest mean win rate across all three Romance target languages, suggesting that the benefits of our approach are not limited to reward modeling and can also be reflected to policy. At the same time, the margins over the multitask-RLHF are relatively small in Catalan and French. In both cases, the multitask-RLHF substantially improves over conventional RLHF, suggesting that pretraining on Romance preference data provides a strong basis for adapting to these targets. Against this competitive baseline, MAML-RLHF still performs best on average.

\begin{table}[H]
    \centering
    \small
    \captionof{table}{Win rate of Gemma3 270M adapted to 100 samples from each target language in the RLHF pipeline.}   
    \label{tab:winrate_RLHF_100samples_multitask_gemma}
    \begin{tabular}{cccc}
        \toprule
        Target & Algorithm & Samples & WR (\%) \\
        \midrule
        \multirow{3}{*}{French} & MAML-RLHF & 100 & $\mathbf{66.23} \pm 1.32$ \\
        & Multitask-RLHF & 100 & $64.55 \pm 1.40$ \\
        & RLHF & 100 & $49.12 \pm 2.21$ \\
        \midrule
        \multirow{3}{*}{Catalan} & MAML-RLHF & 100 & $\mathbf{63.28} \pm 1.35$ \\
        & Multitask-RLHF & 100 & $62.97 \pm 1.07$ \\
        & RLHF & 100 & $50.03 \pm 3.30$ \\
        \midrule
        \multirow{3}{*}{Romanian} & MAML-RLHF & 100 & $\mathbf{85.23} \pm 0.32$ \\
        & Multitask-RLHF & 100 & $63.41 \pm 2.34$ \\
        & RLHF & 100 & $60.28 \pm 6.17$ \\
        \bottomrule
    \end{tabular}
\end{table}

In \Cref{fig:Gemma_RM100_avg_accuracies_ar_id_sk}, we evaluate Gemma 270M settings on Arabic, Slovak, and Indonesian, which come from language families different from the Romance meta-training languages. The averaged curves with standard deviation bands show that MAML-RM continues to outperform the baseline across all three targets. This suggests that the benefit of meta-training is not limited to closely related Romance languages, but also transfers to linguistically distant target languages while preserving a similar adaptation trend. 

\begin{figure*}[t!]
    \centering
    \includegraphics[width=0.3\linewidth]{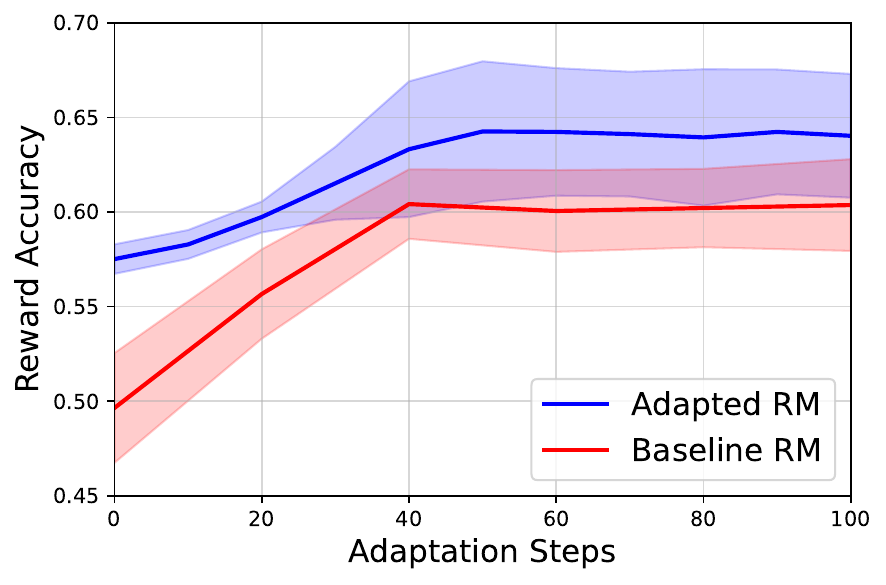}
    \includegraphics[width=0.3\linewidth]{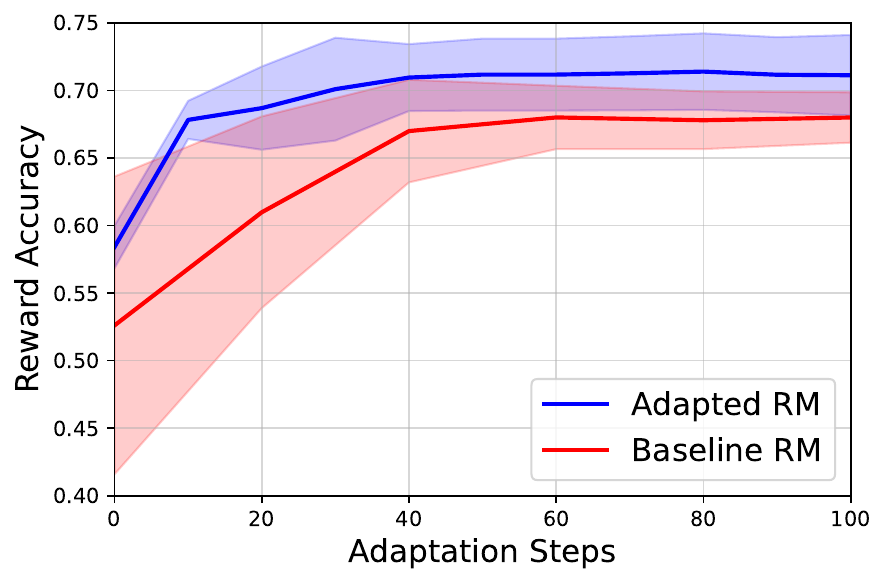}
    \includegraphics[width=0.3\linewidth]{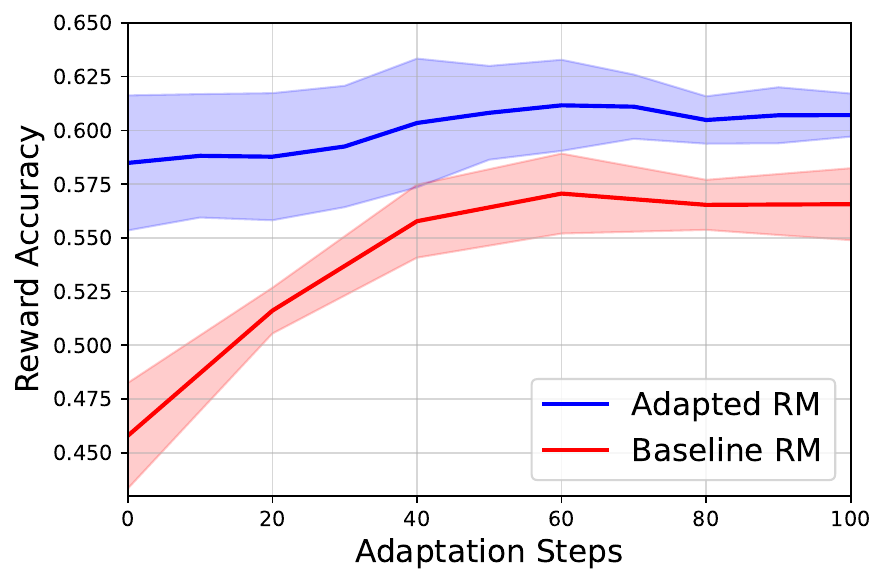}
    \caption{Average reward accuracies with standard deviation bands of BLOOM 1.7B reward functions over 3 random seeds during adaptation on the target language with 100 samples. [Left] Arabic. [Center] Slovak. [Right] Indonesian. In all settings, MAML reward functions show consistently higher reward accuracy than the baseline reward functions in adaptation.\looseness=-1}
    \label{fig:Gemma_RM100_avg_accuracies_ar_id_sk}
\end{figure*}

\Cref{tab:winrate_RLHF_100samples_multitask_gemma_ar_sk_id} reports the corresponding policy evaluation results for these non-Romance languages. Overall, MAML-RLHF improves over both baselines on Arabic and Indonesian, with larger margins over multitask RLHF than those observed for the Romance targets. Similar to MAML-DPO on non-Romance languages, this comparison suggests that our approach yields more robust transfer across language families, whereas multitask appears more sensitive to the language composition during preference pretraining. The improvement is more modest in Slovak, where conventional RLHF already achieves a strong win rate, resulting limited room for further improvement through GRPO even when MAML-RLHF provides stronger reward-model adaptation.

\begin{table}[H]
    \centering
    \small
    \captionof{table}{Win rate of Gemma3 270M adapted to 100 samples from each non-Romance target language in the RLHF pipeline.}   \label{tab:winrate_RLHF_100samples_multitask_gemma_ar_sk_id}
    \begin{tabular}{cccc}
        \toprule
        Target & Algorithm & Samples & WR (\%) \\
        \midrule
        \multirow{3}{*}{Arabic} & MAML-RLHF & 100 & $\mathbf{76.90} \pm 6.29$ \\
        & Multitask-RLHF & 100 & $59.30 \pm 6.80$\\
        
        & RLHF & 100 & $50.71 \pm 4.35$ \\
        
        \midrule
        \multirow{3}{*}{Slovak} & MAML-RLHF & 100 & $\mathbf{77.30} \pm 3.81$\\
        & Multitask-RLHF & 100 & $75.60 \pm 0.92$ \\
        & RLHF & 100 & $76.22 \pm 0.65$ \\
        \midrule
        \multirow{3}{*}{Indonesian} & MAML-RLHF & 100 & $\mathbf{82.07} \pm 1.66$ \\
        & Multitask-RLHF & 100 & $44.17 \pm 0.42$ \\
        & RLHF & 100 & $48.67 \pm 1.47$ \\
        \bottomrule
    \end{tabular}
\end{table}

\paragraph{Sufficient adaptation data.} \Cref{tab:WR_gemma_12k_40k_RLHF} and \Cref{tab:WR_bloom_RLHF_8k_40k_samples} report the win-rate results for models fine-tuned using the RLHF pipeline using 12k/8k and 40k target language samples. With both model configurations, MAML-RLHF consistently outperforms conventional RLHF in all cases, showing the effectiveness of the meta-learned initialization for RLHF-based adaptation. With Gemma3 270M, MAML-RLHF achieves higher win rates than multitask RLHF on Catalan and Romanian under both data settings, while multitask RLHF performs better on French. This suggests that MAML-RLHF provides a stronger adaptation advantage in most target-language settings, although multitask training can remain competitive when the target language is closely aligned with the meta-training languages.

\begin{table}[H]
    \centering
    \small
    \captionof{table}{Win rate evaluation of Gemma 270M models adapted with 12k and 40k samples from each target language under baseline, multi-task, and MAML RLHF settings.}   \label{tab:WR_gemma_12k_40k_RLHF}
    \begin{tabular}{cccc}
        \toprule
        Target & Algorithm & Samples &  WR(\%) \\
        \midrule
        \multirow{6}{*}{French} 
        
        & MAML-RLHF & 12000 & $72.05$ \\
        & Multitask-RLHF & 12000 & $\mathbf{73.45}$\\
          & RLHF & 12000 &  $69.25$\\ 
         
        \cmidrule(lr){2-4}
        
        & MAML-RLHF & 40000 & $72.30$\\ 
        & Multitask-RLHF & 40000 &  $\mathbf{74.90}$\\
          & RLHF & 40000 & $64.90$ \\
         \midrule

         \multirow{6}{*}{Catalan}  
        
     & MAML-RLHF & 12000 & $\mathbf{63.00}$\\
     & Multitask-RLHF & 12000 & $57.70$ \\
          & RLHF & 12000 & $50.30$ \\
         
        \cmidrule(lr){2-4}
        
         & MAML-RLHF & 40000 & $\mathbf{69.65}$ \\
         & Multitask-RLHF & 40000 & $65.80$ \\
          & RLHF & 40000 & $58.20$ \\
         \midrule       
    \multirow{6}{*}{Romanian}  
        
     & MAML-RLHF & 12000 & $\mathbf{87.25}$ \\
     & Multitask-RLHF & 12000 & $85.10$ \\
          & RLHF & 12000 & $76.15$ \\
         
        \cmidrule(lr){2-4}
        
         & MAML-RLHF & 40000 & $\mathbf{88.95}$  \\
         & Multitask-RLHF & 40000 & $87.00$ \\
          & RLHF & 40000 & $82.60$ \\
         \bottomrule
    \end{tabular} 
\end{table}

\begin{table}[h]
    \centering
    \small
    \caption{Win rate evaluation of BLOOM 7.1B models trained with RLHF pipelines. MAML-RLHF models achieve at least 9\% in performance improvement over baseline models.}
    \begin{tabular}{cccc}
        \toprule
        Target & Algorithm & Samples &  WR(\%) \\
        \midrule
        \multirow{2}{*}{French} 
          & MAML-RLHF & 8000 & $\textbf{69.6}$\\
          & RLHF & 8000 &  $60.1$\\ 
        \midrule
         \multirow{2}{*}{Catalan}  
         & MAML-RLHF & 8000 & $\textbf{73.2}$\\
          & RLHF & 8000 &  $56.8$\\ 
    \midrule
        \multirow{2}{*}{Romanian}  
         & MAML-RLHF & 8000 & $\textbf{65.0}$ \\
          & RLHF & 8000 &  $52.7$ \\ 
        \bottomrule
    \end{tabular}  \label{tab:WR_bloom_RLHF_8k_40k_samples}
\end{table}

\paragraph{Out-of-distribution evaluation. } We evaluate the OOD benchmark performances of Gemma 3 270M models trained with the RLHF pipeline in \Cref{tab:OOD_benchmark_results_RLHF_gemma}. While MAML-RLHF models show improvements in more benchmarks and metrics than the baseline RLHF models in French and Catalan, the same pattern is not present in Romanian. We attribute this to the diminishing advantage  of MAML-RLHF as the amount of adaptation dataset becomes larger. We also note that compared to DPO models, RLHF models are more unstable, due to training with reward functions increasing the possibility of overfitting and losing generalization.

\begin{table}[h]
    \onecolumn
    \centering
    \tiny
    \caption{Out-of-distribution benchmark evaluation results of Gemma 270M models under baseline and MAML-RLHF. Bold numbers indicate where the trained model outperformed the base model.}
    \begin{tabular}{cccc c  c c  c c}
        \toprule
        \multirow{2}{*}{Language} & \multirow{2}{*}{Algorithm} & \multirow{2}{*}{Samples} & \multicolumn{2}{c}{ARC} & \multicolumn{2}{c}{Hellaswag} & \multicolumn{2}{c}{MMLU} \\
        \cmidrule(lr){4-9}
         &  &  & Acc.  & Norm. Acc. & Acc. & Norm. Acc.& Acc.  & Norm. Acc. \\
        \midrule
        \multirow{3}{*}{French} & Base model & - & 0.210 & 0.249 & 0.271 & 0.261 & 0.251 & 0.266 \\
        & MAML-RLHF & 40000 & \textbf{0.213} & \textbf{0.250} & \textbf{0.273} & \textbf{0.291} & 0.239 & 0.251\\
        & RLHF & 40000 & 0.210 & 0.248 & 0.270 & \textbf{0.264} & \textbf{0.260} & 0.258 \\
        \midrule
        \multirow{3}{*}{Catalan} & Base model & - & 0.212 & 0.239 & 0.265 & 0.281 & 0.242 & 0.252 \\
        & MAML-RLHF & 40000 & \textbf{0.214} & \textbf{0.257} & \textbf{0.279} & \textbf{0.287} & \textbf{0.247} & \textbf{0.253}\\
        & RLHF & 40000 & \textbf{0.230} & \textbf{0.271} & 0.263 & 0.265 & 0.238 & \textbf{0.268} \\
        \midrule
        \multirow{3}{*}{Romanian} & Base model & - & 0.191 & 0.240 &0.283  & 0.276 & 0.260 & 0.268 \\
        & MAML-RLHF & 40000 &  0.188 & \textbf{0.244} & 0.255 & 0.251 & 0.230 & 0.246 \\
        & RLHF & 40000 &\textbf{0.207} & \textbf{0.241}& 0.271 &0.263 &0.241 & 0.247\\
        \bottomrule
    \end{tabular}   \label{tab:OOD_benchmark_results_RLHF_gemma}
\end{table}

\begin{table}[H]
\centering
\caption{Resource consumption of different training stages based on Bloom 7.1B model. Note that the \textit{Adaptation-DPO/RM} is also used for conventional DPO/RM. All the processes run with a single H200 GPU.
}

\label{tab:resource_consumption}
\begin{tabular}{lcccc}
\toprule
Method & Training steps & Samples & Wall-clock time & GPU memory consumption \\
\midrule
MAML-DPO        & 400 & 12,000 & 23 hours   & max. 137.2GB \\
Multitask-DPO   & 400 & 12,000 & 23 hours   & max. 128.8GB \\
Adaptation-DPO  & 100 & 100    & 0.5 hour   & max. 121.8GB \\
\midrule
MAML-RM         & 400 & 16,000 & 10 hours   & max. 135.8GB \\
Multitask-RM    & 400 & 16,000 & 8 hours    & max. 131.6GB \\
Adaptation-RM   & 100 & 100    & 0.25 hour  & max. 114.8GB \\
\bottomrule
\end{tabular}
\end{table}

\subsection{Resource Consumption of MAML-DPO} \label{sec:resource_consumption}
We further report the resource consumption of different approaches in \Cref{tab:resource_consumption}. With the first-order approximation to the MAML objective, the additional cost of our MAML-based methods remains moderate compared with the corresponding multitask models trained on the same amount of multilingual preference data. Specifically, MAML-DPO mainly increases peak GPU memory usage, while MAML-RM incurs a small increase in both wall-clock time and memory consumption. These results suggest that the cost of our approaches remains manageable with respect to both training time and GPU memory requirements.

\newpage
\section{Convergence of Meta-RLHF}
\label{appendix: Proof of Convergerence for Meta-RLHF}
\subsection{Proof of Theorem \ref{theorem: meta-RLHF convergence}}
We first show that the task-specific loss function has Lipschitz-continuous gradient. Recall the definition of $\calL(\phi;\tau)$ and its gradient $\nabla_\phi \calL(\phi;\tau)$.
\begin{align*}
\calL(\phi;\tau) &= -\E_{(x,y_w,y_\ell) \sim \calD^\tau} \left[ \log \sigma \left( \phi^\top \Delta \psi(x, y_w,y_\ell) \right)\right]\\
\nabla_\phi \calL(\phi; \tau) &= -\E_{(x,y_w,y_\ell) \sim \calD^\tau} \left[ 
        \sigma(-\phi^\top \Delta \psi(x, y_w,y_\ell)) \cdot \Delta \psi(x, y_w,y_\ell)
    \right]
\end{align*}
Now,
\begin{align*}
    \abs{\nabla_\phi \calL(\phi;\tau) - \nabla_\phi \calL(\phi';\tau)} \le \mathop{\E}_{(x,y_w,y_\ell) \sim \calD^\tau} &\left[ \left \lVert \sigma(-\phi^\top \Delta \psi(x, y_w,y_\ell)) \cdot\Delta \psi(x, y_w,y_\ell) \right. \right.\\
    &-\left. \left. \sigma(-\phi'^\top \Delta \psi(x, y_w,y_\ell)) \cdot \Delta \psi(x, y_w,y_\ell) \right \rVert \right]
\end{align*}
\begin{align*}
    &\le \mathop{\E}_{(x,y_w,y_\ell) \sim \calD^\tau} \left[ \abs{ \sigma(-\phi^\top \Delta \psi(x, y_w,y_\ell)) - \sigma(-\phi'^\top \Delta \psi(x, y_w,y_\ell))} \cdot \norm{\Delta \psi(x, y_w,y_\ell) } \right]\\
    &\le 2 \mathop{\E}_{(x,y_w,y_\ell) \sim \calD^\tau} \left[ \norm{ (\phi-\phi')^\top \Delta \psi(x, y_w,y_\ell)) } \right]\\
    &\le 2 \mathop{\E}_{(x,y_w,y_\ell) \sim \calD^\tau} \left[ \norm{ \phi-\phi'} \norm{ \Delta \psi(x, y_w,y_\ell) } \right]\\
    &\le 4 \norm{ \phi-\phi'}.
\end{align*}

 Additionally, let $\widetilde{\nabla}_\phi \calL(\phi; \tau) = \frac{1}{\abs{\calD^\tau_{\text{in}}} }\sum_{j\in \calD^\tau_{\text{in}}} \nabla_\phi \calL(\phi; \tau, j)$ be a stochastic gradient estimator for task $\tau$. Since $\norm{\nabla_\phi \calL(\phi;\tau)}_2 \le 2$, standard concentration inequality gives us,
 \[
 {\E}_{\calD^\tau_{\text{in}}} \left[ \norm{\widetilde{\nabla}_\phi \calL(\phi; \tau) - {\nabla}_\phi \calL(\phi; \tau)}_2^2 \right] \le O\left( \frac{1}{\abs{\calD^\tau_{\text{in}}}}\right) \triangleq \widetilde{\sigma}^2.
 \]
 We can now apply Theorem 4.15 from \cite{fallah2020convergence} to conclude that as long as the number of iterations $T \ge \Delta \min \set{\frac{4}{\varepsilon^2}, \frac{4}{\sigma^2(16\alpha^2 + B^{-1})} + \frac{4(BD_o + D_{\text{in}})}{\widetilde{\sigma}^2}}$ with $\Delta = \calL_M(\phi_1) - \min_\phi \calL_M(\phi)$, we have
 \[
 \E\left[ \norm{\nabla_\phi \calL_M(\widehat{\phi}_M)}_2 \right] \le O\left(\sqrt{\sigma^2 \left(16\alpha^2 + \frac{1}{B}\right) + \frac{\widetilde{\sigma}^2}{B D_o} + \frac{\widetilde{\sigma}^2}{D_{\text{in}}}} \right) + \varepsilon
 \]
 Now we substitute the value of $\widetilde{\sigma}$ and choose $D_{\text{in}} = 1/\sigma,  D_o = O(1), B = 1/\alpha^2$, and $\varepsilon = \alpha \sigma$. Then we obtain \[\E\left[ \norm{\nabla_\phi \calL_M(\widehat{\phi}_M)}_2 \right] \le O(\alpha \sigma).\] Moreover, $\Delta = \calL_M(\phi_1) - \min_\phi \calL_M(\phi) \le O(L)$ as feature norms are bounded by $L$. Then the number of iterations $T$ is bounded by $O\left( \frac{L}{\alpha^2 \sigma^2}\right)$.

By lemma~\ref{lem:strong-convexity-MAML-Reward}, $\calL_M(\phi)$ is a $m_M$-strongly convex function for $m_M = \frac{c\cdot m_0}{4+c+m_0}$.
\[
\calL_M(\phi^\star_M) \ge \calL_M(\widehat{\phi}_M) + \left \langle \nabla_\phi \calL_M(\widehat{\phi}_M), \phi^\star_M  - \widehat{\phi}_M \right \rangle + \frac{m_M}{2} \norm{\phi^\star_M  - \widehat{\phi}_M }_2^2
\]
After rearranging and using Cauchy-Schwartz inequality, we obtain the following bound.
\begin{align*}
    \frac{m_M}{2} \norm{\phi^\star_M  - \widehat{\phi}_M }_2^2 &\le \calL_M(\phi^\star_M) - \calL_M(\widehat{\phi}_M) + \norm{\nabla_\phi \calL_M(\widehat{\phi}_M)}_2 \norm{\phi^\star_M  - \widehat{\phi}_M}_2\\
    &\le \norm{\nabla_\phi \calL_M(\widehat{\phi}_M)}_2 \norm{\phi^\star_M  - \widehat{\phi}_M}_2
\end{align*}
This implies $\mathbb E\left[\norm{\phi^\star_M  - \widehat{\phi}_M}_2 \right] \le \mathbb E\left[ \frac{2}{m_M} \norm{\nabla_\phi \calL_M(\widehat{\phi}_M)}_2 \right]\le \frac{2\alpha \sigma}{m_M}$. \qed

\subsection{Technical Lemmas}
We dedicate this section to proving that the meta-objective $\mathcal L_M(\phi)$ is strongly convex under the conditions of Theorem~\ref{theorem: meta-RLHF convergence}. This property is formalized in the following lemma.
\begin{lemma}\label{lem:strong-convexity-MAML-Reward}
    Suppose Assumption \ref{asn:full-coverage} hold and $m_0 = \frac{\nu}{2(1+\exp(2B_\phi))}$. If $\alpha \le \alpha_0 = \frac{m_0}{2(4+m_0)}$, then $\nabla^2_\phi \calL_M(\phi)$ is positive semidefinite. Moreover, if $\alpha \le \frac{m_0}{2(4+c+m_0)}$ for some constant $c > 0$ then $\nabla^2_\phi \calL_M(\phi) \succcurlyeq \frac{c\cdot m_0}{4+ c + m_0}  I$.
\end{lemma}
\begin{proof}
We first derive the third derivative of the task specific loss function. Moreover, we will also write $z = \phi^T\Delta\psi$. 
\begin{align*}
    \nabla^3_\phi\mathcal{L}(\phi; \tau) &= \nabla_\phi\left(\mathbb{E}_{(x,y_w,y_\ell)}[\sigma(z)(1-\sigma(z))\Delta\psi\Delta\psi^T]\right) \\
    &= \mathbb{E}_{(x,y_w,y_\ell)}\left[\frac{d}{dz}(\sigma(z)(1-\sigma(z)))\nabla_\phi z \cdot \Delta\psi\Delta\psi^T\right]
\end{align*}

Since $\nabla_\phi z = \Delta\psi$ and $\frac{d}{dz}(\sigma(z)(1-\sigma(z))) = \sigma(z)(1-\sigma(z))(1-2\sigma(z))$, we have:
\begin{equation}
    \nabla^3_\phi\mathcal{L}(\phi; \tau) = \mathbb{E}_{(x,y_w,y_\ell)}\left[\sigma(z)(1-\sigma(z))(1-2\sigma(z))\Delta\psi \otimes \Delta\psi\Delta\psi^T\right]
\end{equation}

Where $\otimes$ represents the tensor product.

Let $\phi_\tau = \phi - \alpha\nabla_\phi\mathcal{L}(\phi; \tau)$ denote the adapted parameter for task $\tau$. The gradient of the meta-loss is given as follows.

\begin{equation}
    \nabla_\phi\mathcal{L}_M(\phi) = \mathbb{E}_{\tau\sim P_{\text{TASK}}}\left[(I - \alpha\nabla^2_\phi\mathcal{L}(\phi; \tau))\nabla_\phi\mathcal{L}(\phi_\tau; \tau)\right]
\end{equation}

Differentiating the gradient with respect to $\phi$ we get the Hessian of the meta-loss. 
\begin{align*}
    \nabla^2_\phi\mathcal{L}_M(\phi) &= \mathbb{E}_{\tau}\left[\nabla_\phi\left((I - \alpha\nabla^2_\phi\mathcal{L}(\phi; \tau))\nabla_\phi\mathcal{L}(\phi_\tau; \tau)\right)\right] \\
    &= \mathbb{E}_{\tau}\left[\nabla_\phi(I - \alpha\nabla^2_\phi\mathcal{L}(\phi; \tau)) \cdot \nabla_\phi\mathcal{L}(\phi_\tau; \tau) + (I - \alpha\nabla^2_\phi\mathcal{L}(\phi; \tau)) \cdot \nabla_\phi(\nabla_\phi\mathcal{L}(\phi_\tau; \tau))\right] \\
    &= \mathbb{E}_{\tau}\left[-\alpha\nabla_\phi(\nabla^2_\phi\mathcal{L}(\phi; \tau)) \cdot \nabla_\phi\mathcal{L}(\phi_\tau; \tau) + (I - \alpha\nabla^2_\phi\mathcal{L}(\phi; \tau)) \cdot \nabla_\phi(\nabla_\phi\mathcal{L}(\phi_\tau; \tau))\right]
\end{align*}

The first term involves the third derivative tensor $\nabla^3_\phi\mathcal{L}(\phi; \tau)$, and for the second term, we calculate $\nabla_\phi(\nabla_\phi\mathcal{L}(\phi_\tau; \tau))$.
\begin{align}
    \nabla_\phi(\nabla_\phi\mathcal{L}(\phi_\tau; \tau)) &= \nabla^2_\phi\mathcal{L}(\phi_\tau; \tau) \cdot \nabla_\phi\phi_\tau \\
    &= \nabla^2_\phi\mathcal{L}(\phi_\tau; \tau) \cdot \nabla_\phi(\phi - \alpha\nabla_\phi\mathcal{L}(\phi; \tau)) \\
    &= \nabla^2_\phi\mathcal{L}(\phi_\tau; \tau) \cdot (I - \alpha\nabla^2_\phi\mathcal{L}(\phi; \tau))
\end{align}

Substituting back we obtain.
\begin{align*}
    \nabla^2_\phi\mathcal{L}_M(\phi) &= \mathbb{E}_{\tau}\left[-\alpha\nabla^3_\phi\mathcal{L}(\phi; \tau) \cdot \nabla_\phi\mathcal{L}(\phi_\tau; \tau) \right. \\
    &\quad\quad \left. + (I - \alpha\nabla^2_\phi\mathcal{L}(\phi; \tau)) \cdot \nabla^2_\phi\mathcal{L}(\phi_\tau; \tau) \cdot (I - \alpha\nabla^2_\phi\mathcal{L}(\phi; \tau))\right]
\end{align*}

The following is the complete expression for the Hessian, which consists of two terms.
\begin{equation}
    \nabla^2_\phi\mathcal{L}_M(\phi) = \mathbb{E}_{\tau}[T_1(\tau) + T_2(\tau)]
\end{equation}
where
\begin{align}
    T_1(\tau) &= -\alpha\nabla^3_\phi\mathcal{L}(\phi; \tau) \cdot \nabla_\phi\mathcal{L}(\phi_\tau; \tau) \label{defn:T1}\\
    T_2(\tau) &= (I - \alpha\nabla^2_\phi\mathcal{L}(\phi; \tau)) \cdot \nabla^2_\phi\mathcal{L}(\phi_\tau; \tau) \cdot (I - \alpha\nabla^2_\phi\mathcal{L}(\phi; \tau)) \label{defn:T2}
\end{align}

Since $\nabla^2_\phi \calL_M(\phi) = \E_{\tau}[T_1(\tau) + T_2(\tau)]$, for the Hessian to be positive semidefinite, we need.
\begin{align*}
    \mathbb{E}[w^\top T_2(\tau) w] - |\mathbb{E}[w^\top T_1(\tau) w]| \geq 0
\end{align*}
For any unit vector $w$, Lemma \ref{lem: bound T_1, T_2 for RLHF} shows $|w^T T_1(\tau) w| \leq 8 \alpha$, and $w^\top T_2(\tau) w \ge (1-\alpha)^2 m_0 $.

 Then the sufficient condition is the following.
\begin{align}\label{eq: alpha_sufficient_condition_PSD_Hessian}
(1-\alpha )^2 m_0 > 8\alpha.
\end{align}
Therefore, a sufficient condition is $\alpha < \frac{m_0}{2(4+m_0)}$.

In order to show positive definiteness, the required condition is 
\begin{align}\label{eq: alpha_sufficient_condition_PD_Hessian}
(1-\alpha )^2 m_0 > 8\alpha + 2c  \alpha,
\end{align}
for some $c>0$.
The rest of the proof is almost identical to the positive-semidefinite setting.
\end{proof}

\begin{lemma} \label{lem: bound T_1, T_2 for RLHF}
Assume Assumption~\ref{asn:full-coverage}. Fix any task $\tau$ and stepsize $\alpha > 0$. Let $T_1(\tau)$ and $T_2(\tau)$ be defined as in the proof of Lemma~\ref{lem:strong-convexity-MAML-Reward}. Then,
\[
\|T_1(\tau)\|_2 \le 8\alpha,\quad\text{and},\quad  \lambda_{\min} (T_2(\tau))\geq (1-\alpha)^2 m_0,
\]
where $m_0 = \frac{\nu}{2(1+\exp(2B_\phi))}$.
\end{lemma}

\begin{proof}
\textit{Bound for $T_1(\tau)$.} We first expand $T_1(\tau)$ as follows:
\[
T_1(\tau)
= -\alpha f(z)\,(\Delta\psi^\top v)\,\Delta\psi\Delta\psi^\top,
\quad
v\triangleq \nabla_\phi \calL(\phi_\tau;\tau),
\quad
f(z)=\sigma(z)(1-\sigma(z))(1-2\sigma(z)),
\quad
z=\phi^\top \Delta\psi .
\]
Using $\|\Delta\psi\Delta\psi^\top\|_F=\|\Delta\psi\|_2^2$, and Cauchy--Schwarz,
\[
\|T_1(\tau)\|_F
\le \alpha|f(z)||\Delta\psi^\top v|\|\Delta\psi\|_2^2
\le \alpha |f(z)|\|\Delta\psi\|_2^3\|v\|_2.
\]
Moreover, from the gradient expression
$\nabla_\phi \calL (\cdot;\tau)= -\mathbb{E}[\sigma(-z)\Delta\psi]$, we have
\[
\|v\|_2=\|\nabla_\phi L(\phi_\tau;\tau)\|_2 \le \mathbb{E}\|\Delta\psi\|_2 \le 2.
\]
Finally, $|f(z)|\le 1/2$. Therefore,
\[
\|T_1(\tau)\|_2\leq \|T_1(\tau)\|_F \le \alpha \cdot \frac12 \cdot 2^3 \cdot 2 = 8\alpha,
\]
which completes the proof.

\textit{Bound for $T_2(\tau)$.}
By Lemma \ref{prop:psd-task-loss}, we have that $\nabla^2_\phi \calL(\phi;\tau) \succcurlyeq m_0$ and $\nabla^2_\phi \calL(\phi;\tau) \preccurlyeq I$.
For any unit $w$ let $u=(I-\alpha \nabla_\phi^2 \calL(\phi;\tau))w$, we have that $\|u \|_2^2 \leq (1-\alpha)^2$. Then
\[
w^\top T_2(\tau) w 
= u^\top \nabla_\phi^2 \calL(\phi_\tau;\tau) u
\geq m_0 \|u\|_2^2 \geq m_0(1-\alpha)^2.
\]
In other words, $\| T_2(\tau)\|_2 \geq m_0(1-\alpha)^2$ as claimed.
\end{proof}

\begin{lemma}\label{prop:psd-task-loss}
    For any task $\tau$, the loss function $\calL(\phi;\tau)$ is positive semidefinite i.e. $\nabla^2_\phi \calL(\phi;\tau) \succcurlyeq 0$. Additionally, if Assumption \ref{asn:full-coverage} hold, then it is positive definite i.e. 
    \[ \nabla^2_\phi \mathcal L(\phi)\|_2 \preccurlyeq I \quad \text{and}\quad \nabla^2_\phi \calL(\phi;\tau) \succcurlyeq \frac{\nu}{2(1+\exp(2B_\phi))} I,\]
    for any $\phi$ with $\norm{\phi}_2 \le B_\phi$.
\end{lemma}
\begin{proof}
First, as $\|\Delta\psi(\cdot)\|_2^2 \leq 2$, one can show that $\| \nabla^2_\phi \mathcal L(\phi)\|_2\leq 1$.

Since $\sigma(u) \ge 0 \ \forall u$, the multiplication of two sigmoid functions satisfy $\sigma(u)\sigma(-u) \ge 0$. The matrix $\psi\psi^\top$ is positive semidefinite. Hence,  $\nabla^2_\phi\calL(\phi;\tau) \succcurlyeq 0$.

Additionally, for any $u$, $\sigma(u) \sigma(-u) = \frac{1}{1+\exp(u)} \frac{1}{1+\exp(-u)} = \frac{1}{2+\exp(u) + \exp(-u)} \ge \frac{1}{2+2\exp(\abs{u})}$. Moreover, $\abs{z_\phi(x,y_w,y_\ell)} = \abs{\phi^\top (\psi(x,y_w) - \psi(x,y_\ell)} \le 2B_\phi$. 
 Moreover, under assumption \ref{asn:full-coverage} the matrix $\psi \psi^\top$ is positive-definite with constant $\nu$. Therefore, $\nabla^2_\phi \calL(\phi;\tau) \succcurlyeq \frac{\nu}{2(1+\exp(2B_\phi))} I$.
\end{proof}

\newpage
\section{Convergence of Meta-DPO} \label{appendix: Proof of Convergerence for Meta-DPO}
\subsection{Proof of Theorem \ref{theorem: meta-DPO convergence}}
The key step of the proof of Theorem \ref{theorem: meta-DPO convergence} is similar to that of Theorem \ref{theorem: meta-RLHF convergence}.

By Lemma \ref{lemma:task_smooth_sc_theta}, for all $\tau$, we have that the smoothness constant of task-wise gradient as
\begin{equation}
    \|\nabla_\theta \mathcal L(\theta,\tau) - \nabla_\theta \mathcal L(\theta',\tau) \|_2 \leq \beta^2 \|\theta -\theta' \|_2.
\end{equation}

Moreover, by Proposition \ref{lem:third_deriv_bound_theta}, we have that $\|\nabla_\theta \mathcal L(\theta,\tau) \|_2 \leq 2\beta$.

Additionally, let $\widetilde{\nabla}_\phi \calL(\theta; \tau) = \frac{1}{\abs{\calD^\tau_{\text{in}}} }\sum_{j\in \calD^\tau_{\text{in}}} \nabla_\theta \calL(\theta; \tau, j)$ be a stochastic gradient estimator for task $\tau$. Since by Proposition \ref{lem:third_deriv_bound_theta}, we have that $\|\nabla_\theta \mathcal L(\theta,\tau) \|_2 \leq 2\beta$, standard concentration inequality gives us,
 \[
 {\E}_{\calD^\tau_{\text{in}}} \left[ \norm{\widetilde{\nabla}_\theta \calL(\theta; \tau) - {\nabla}_\theta \calL(\theta; \tau)}_2^2 \right] \le O\left( \frac{\beta^2}{\abs{\calD^\tau_{\text{in}}}}\right) \triangleq \widetilde{\sigma}^2.
 \]

 We can now apply Theorem 4.15 from \cite{fallah2020convergence} to conclude that as long as the number of iterations $T \ge \Delta \min \set{\frac{4}{\varepsilon^2}, \frac{4}{\sigma^2(16\alpha^2 + B^{-1})} + \frac{4(BD_o + D_{\text{in}})}{\widetilde{\sigma}^2}}$ with $\Delta = \calL_M(\theta_1) - \min_\theta \calL_M(\theta)$, we have
 \[
 \E\left[ \norm{\nabla_\theta \calL_M(\widehat{\theta}_M)}_2 \right] \le O\left(\sqrt{\sigma^2 \left(16\alpha^2 + \frac{1}{B}\right) + \frac{\widetilde{\sigma}^2}{B D_o} + \frac{\widetilde{\sigma}^2}{D_{\text{in}}}} \right) + \varepsilon.
 \]
 
 Now we substitute the value of $\widetilde{\sigma}$ and choose $D_{\text{in}} = \beta/\sigma,  D_o = O(1), B = 1/\alpha^2$, and $\varepsilon = \alpha \sigma$. Then we obtain \[\E\left[ \norm{\nabla_\theta \calL_M(\widehat{\theta}_M)}_2 \right] \le O(\alpha \sigma).\] Moreover, $\Delta = \calL_M(\theta_1) - \min_\theta \calL_M(\theta) \le O(L)$ as feature norms are bounded by $L$. Then the number of iterations $T$ is bounded by $O\left( \frac{L}{\alpha^2 \sigma^2}\right)$.

By Theorem~\ref{theorem:meta_pl_theta}, $\calL_M(\theta)$ is $m_M$-strongly convex with $m_M = (1-\alpha\beta^2)^2\,\mu_\tau - 4\alpha\beta^4$, where $\mu_\tau$ is the per-task strong-convexity constant of \Cref{lemma:task_smooth_sc_theta}.
Standard calculation gives us that, with $\alpha = \frac{1}{\beta^2\left(2+\frac{8(1+\exp(Z_{\max}))}{\nu}\right)}$, we have $m_M \geq \beta^2 \frac{c\nu^3}{(1+\exp(Z_{\max}))^3}$ for some constant $c>0$. 
We have the following
\[
\calL_M(\theta^\star_M) \ge \calL_M(\widehat{\theta}_M) + \left \langle \nabla_\theta \calL_M(\widehat{\theta}_M), \theta^\star_M  - \widehat{\theta}_M \right \rangle + \frac{m_M}{2} \norm{\theta^\star_M  - \widehat{\theta}_M }_2^2.
\]
After rearranging and using Cauchy-Schwartz inequality, we obtain the following bound.
\begin{align*}
    \frac{m_M}{2} \norm{\theta^\star_M  - \widehat{\theta}_M }_2^2 &\le \calL_M(\theta^\star_M) - \calL_M(\widehat{\theta}_M) + \norm{\nabla_\theta \calL_M(\widehat{\theta}_M)}_2 \norm{\theta^\star_M  - \widehat{\theta}_M}_2\\
    &\le \norm{\nabla_\theta \calL_M(\widehat{\theta}_M)}_2 \norm{\theta^\star_M  - \widehat{\theta}_M}_2.
\end{align*}
This implies $\mathbb E\left[\norm{\theta^\star_M  - \widehat{\theta}_M}_2\right] \le \frac{2}{m_M} \mathbb E \left[\norm{\nabla_\theta \calL_M(\widehat{\theta}_M)}_2 \right] \le \frac{2\alpha \sigma}{m_M}$. \qed

\subsection{Technical Lemmas}
We first show the the meta loss function $\calL_M(\theta)$ is strongly convex under condition of Theorem \ref{theorem: meta-DPO convergence}.
\begin{theorem}[Strong convexity of DPO loss]
\label{theorem:meta_pl_theta}
Let $\mu_\tau$ be the per-task strong-convexity constant of \Cref{lemma:task_smooth_sc_theta}. Define
\begin{equation}
\label{eq:mM_theta}
m_M \triangleq (1-\alpha\beta^2)^2\,\mu_\tau - 4\alpha\beta^4.
\end{equation}
If $m_M>0$, then
\[
\nabla_\theta^2 \mathcal L_M(\theta)\succcurlyeq m_M I.
\]
\end{theorem}

\begin{proof}
Let $\theta_\tau = \theta - \alpha \nabla_\theta \mathcal L(\theta,\tau)$. Then for any unit $w$,
\begin{equation}
    \begin{aligned}
        \nabla_\theta \mathcal L_M(\theta) &=
\mathbb{E}_\tau\big[(I-\alpha \nabla_\theta^2  \mathcal L(\theta,\tau))\cdot\nabla_\theta \mathcal L(\theta_\tau,\tau)\big], \\
    \nabla_\theta^2 \mathcal L_M(\theta)&=\mathbb{E}_\tau[T_1(\tau)+T_2(\tau)],
    \end{aligned}
\end{equation}
where
\begin{align}
    T_1(\tau) &= -\alpha\nabla^3_\theta\mathcal{L}(\theta; \tau) \cdot \nabla_\theta\mathcal{L}(\theta_\tau; \tau) \label{defn:T1-new},\\
    T_2(\tau) &= (I - \alpha\nabla^2_\theta\mathcal{L}(\theta; \tau)) \cdot \nabla^2_\theta\mathcal{L}(\theta_\tau; \tau) \cdot (I - \alpha\nabla^2_\theta\mathcal{L}(\theta; \tau)).
\end{align}
Now, we have that
\[
w^\top\nabla^2 \mathcal L_M(\theta)w
=
\mathbb{E}_\tau\big[w^\top(T_1(\tau)+T_2(\tau))w\big]
\ge \mathbb{E}_\tau[w^\top T_2(\tau)w] - \mathbb{E}_\tau|w^\top T_1(\tau)w|.
\]
Apply Lemma~\ref{lem:T1T2_bounds_theta} to obtain
\[
w^\top\nabla^2 \mathcal L_M(\theta)w
\ge (1-\alpha\beta^2)^2 \mu_\tau - 4\alpha\beta^4
= m_M.
\]
Thus $\nabla^2 \mathcal L_M(\theta)\succcurlyeq m_M I$, i.e. $\mathcal L_M(\theta)$ is $m_M$-strongly convex.
\end{proof}

\begin{lemma}[Bounds on the meta-Hessian terms]
\label{lem:T1T2_bounds_theta}
Under Assumption \ref{asn:full-coverage}, we have that
\[
|w^\top T_1(\tau) w| \le 4\alpha\beta^4,
\qquad
w^\top T_2(\tau) w \ge (1-\alpha\beta^2)^2\,\mu_\tau.
\]
\end{lemma}

\begin{proof}
For $T_2$, by \Cref{lemma:task_smooth_sc_theta} we have that $\nabla_\theta^2\mathcal L(\theta,\tau)\preccurlyeq L_\tau I = \beta^2 I$ and $\nabla_\theta^2\mathcal L(\theta_\tau,\tau)\succcurlyeq \mu_\tau I$; thus,
\[
w^\top T_2(\tau) w
=
\big((I-\alpha \nabla_\theta^2\mathcal L(\theta, \tau))w\big)^\top \nabla_\theta^2\mathcal L(\theta_\tau, \tau)\big((I-\alpha \nabla_\theta^2\mathcal L(\theta, \tau))w\big)
\ge \mu_\tau\|(I-\alpha \nabla_\theta^2\mathcal L(\theta, \tau))w\|^2
\ge (1-\alpha\beta^2)^2 \mu_\tau.
\]
For $T_1$, by Cauchy-Schwarz for tensor contraction,
\[
|w^\top T_1(\tau) w|
\le \alpha\,\|\nabla^3_\theta L(\theta;\tau)\|_{\mathrm{op}}\;\|\nabla_\theta \calL(\theta_\tau;\tau)\|\;\|w\|^2.
\]
Apply Lemma~\ref{lem:third_deriv_bound_theta}: $\|\nabla^3 \calL(\cdot)\|_{\mathrm{op}}\le 2\beta^3$ and
$\|\nabla \calL(\theta_\tau;\tau)\|\le 2\beta$, hence $|w^\top T_1(\tau)w|\le 4\alpha\beta^4$.
\end{proof}

\begin{lemma}[Task-wise smoothness and strong convexity]
\label{lemma:task_smooth_sc_theta}
Define the task-wise smoothness and strong-convexity constants as
\[
L_\tau \;\triangleq\; \beta^2,
\qquad
\mu_\tau \;\triangleq\; \frac{\beta^2\,\nu}{2\bigl(1+e^{Z_{\max}}\bigr)}.
\]
Then $\nabla_\theta^2 \calL(\theta;\tau) \preccurlyeq L_\tau I$, and under Assumption~\ref{ass:bounded_ref_ratio} (so $|z_\theta|\le Z_{\max}$) and Assumption~\ref{asn:full-coverage} we also have $\nabla_\theta^2 \calL(\theta;\tau) \succcurlyeq \mu_\tau I$.
\end{lemma}

\begin{proof}
\textit{Upper bound.}
Note that $\sigma(u)\sigma(-u)\le 1/4$,
\[
\|\nabla^2 \calL(\theta;\tau)\|_2 = \beta^2\,\mathbb{E}\big[\sigma(z_\theta)\sigma(-z_\theta)\Delta\psi\Delta\psi^\top\big]
\le
\beta^2\,\mathbb{E}\big[\sigma(z_\theta)\sigma(-z_\theta)\,\|\Delta\psi\Delta\psi^\top\|\big]
\le \beta^2\cdot \frac14 \cdot 4 = \beta^2.
\]

\textit{Lower bound.} For any scalar $u$,
\[
\sigma(u)\sigma(-u)
=
\frac{1}{2+e^u+e^{-u}}
\ge
\frac{1}{2+2e^{|u|}}
=
\frac{1}{2(1+e^{|u|})}.
\]
Using \eqref{eq:Zmax_theta}, $|z_\theta|\le Z_{\max}$, hence
$\sigma(z_\theta)\sigma(-z_\theta)\ge \frac{1}{2(1+e^{Z_{\max}})}$.
Therefore
\[
\nabla^2 \calL(\theta;\tau)
\succcurlyeq
\beta^2\cdot \frac{1}{2(1+e^{Z_{\max}})}\;\mathbb{E}[\Delta\psi\Delta\psi^\top]
\succcurlyeq
\beta^2\cdot \frac{1}{2(1+e^{Z_{\max}})}\;\nu I,
\]
where the last step uses Assumption~\ref{asn:full-coverage}.
\end{proof}

\begin{lemma}[Uniform bound on the third derivative of the task loss]
\label{lem:third_deriv_bound_theta}
For all $\tau$,
\begin{align*}
    \|\nabla_\theta^3 \mathcal L(\theta;\tau)\|_{\mathrm{op}} & \le 2\beta^3, \\
    \|\nabla_\theta \mathcal L(\theta;\tau)\|_2 &\le 2\beta.
\end{align*}
\end{lemma}

\begin{proof}
For the scalar loss $\ell(z)=\log(1+e^{-z})$, we have $\ell'''(z)=-\sigma(z)\sigma(-z)(1-2\sigma(z))$ so $|\ell'''(z)|\le 1/4$.
For a single sample, $z_\theta=\beta\,\theta^\top\Delta\psi - J$ is affine in $\theta$ with $\nabla_\theta z_\theta=\beta\Delta\psi$.
Hence, for unit vectors $a,b,c$,
\[
|\nabla_\theta^3 \ell(z_\theta)[a,b,c]|
=
|\ell'''(z_\theta)|\cdot |\langle \beta\Delta\psi,a\rangle|\cdot|\langle \beta\Delta\psi,b\rangle|\cdot|\langle \beta\Delta\psi,c\rangle|
\le \frac14\,\beta^3\,\|\Delta\psi\|^3
\le \frac14\,\beta^3\cdot 8 = 2\beta^3,
\]
using $\|\Delta\psi\|\le 2$. Taking expectation over $\calD^\tau$ preserves the bound.

For the gradient, $\nabla_\theta \mathcal L(\theta;\tau)= -\mathbb{E}[\beta\sigma(-z_\theta)\Delta\psi]$, so
$\|\nabla_\theta \mathcal L(\theta;\tau)\|\le \beta\,\mathbb{E}\|\Delta\psi\|\le 2\beta$.
\end{proof}

\section{Convergence Analysis of the Adaptation Phase for MAML-DPO}\label{appendix:convergence_adaptation_MAML_DPO}

In this section, we provide a comparative convergence analysis of the adaptation phase. We prove that the meta-learned initialization $\hat{\theta}_M$ yields a faster expected convergence rate to a target task's optimal policy $\theta^*_{\tau}$ compared to a baseline multitask initialization $\theta_{\rm base}$.

\subsection{Preliminaries and Task Geometry}

Let the target task $\tau \sim P_{\rm TASK}$ be drawn from the distribution of languages. During adaptation, we optimize the empirical DPO loss using gradient descent with step size $\eta$:
\begin{equation}
    \theta_{t+1} = \theta_t - \eta \nabla_{\theta} \mathcal{L}^{\rm DPO}(\theta_t; \tau)
\end{equation}

From \Cref{lemma:task_smooth_sc_theta}, the task-wise DPO loss $\mathcal{L}^{\rm DPO}(\theta; \tau)$ is $L_{\tau}$-smooth and $\mu_{\tau}$-strongly convex.

We define the baseline initialization as the minimizer of the joint expected risk across all tasks:
\begin{equation}
    \theta_{\rm base}^* = \arg\min_\theta \mathbb{E}_{\tau \sim P_{\rm TASK}} [\mathcal{L}^{\rm DPO}(\theta; \tau)]
\end{equation}
We define the optimal MAML initialization as the minimizer of the meta-objective:
\begin{equation}
    \theta_{M}^* = \arg\min_\theta \mathbb{E}_{\tau \sim P_{\rm TASK}} [\mathcal{L}^{\rm DPO}(\theta - \alpha \nabla_{\theta}\mathcal{L}^{\rm DPO}(\theta; \tau); \tau)]
\end{equation}
By \Cref{theorem: meta-DPO convergence}, our gradient-based algorithm produces $\hat{\theta}_M$ such that $||\theta_M^* - \hat{\theta}_M||_2 \le \epsilon_{meta}$.

\subsection{Deriving the Initialization Advantage}

To prove faster convergence, we must first bound the expected squared distance of both initializations to the task-specific optimum $\theta_\tau^* = \arg\min_\theta \mathcal{L}^{\rm DPO}(\theta; \tau)$.

\begin{lemma}[Baseline Initialization Distance]
Let \Cref{ass:variance_of_task_theta}  hold, such that the variance of task gradients is bounded by $\sigma^2$. The expected distance from the baseline initialization $\theta_{\rm base}^*$ to a sampled task optimum $\theta_\tau^*$ is bounded by:
\begin{equation}
    \mathbb{E}_{\tau}[||\theta_{\rm base}^* - \theta_\tau^*||_2^2] \le \frac{\sigma^2}{\mu_{\tau}^2} \triangleq D_{\rm base}^2
\end{equation}
\end{lemma}

\begin{proof}
By the strong convexity of the task-specific loss $\mathcal{L}^{\rm DPO}(\theta; \tau)$, the gradient is co-coercive. For any $\theta$, the following inequality holds:
\begin{equation*}
    \langle \nabla \mathcal{L}^{\rm DPO}(\theta; \tau) - \nabla \mathcal{L}^{\rm DPO}(\theta_\tau^*; \tau), \theta - \theta_\tau^* \rangle \ge \mu_\tau ||\theta - \theta_\tau^*||_2^2
\end{equation*}
Because $\theta_\tau^*$ is the exact minimizer for task $\tau$, its gradient is zero ($\nabla \mathcal{L}^{\rm DPO}(\theta_\tau^*; \tau) = 0$). Evaluating this at the baseline initialization $\theta_{\rm base}^*$ yields:
\begin{equation*}
    \langle \nabla \mathcal{L}^{\rm DPO}(\theta_{\rm base}^*; \tau), \theta_{\rm base}^* - \theta_\tau^* \rangle \ge \mu_\tau ||\theta_{\rm base}^* - \theta_\tau^*||_2^2
\end{equation*}
Applying the Cauchy-Schwarz inequality to the left side:
\begin{equation*}
    ||\nabla \mathcal{L}^{\rm DPO}(\theta_{\rm base}^*; \tau)||_2 ||\theta_{\rm base}^* - \theta_\tau^*||_2 \ge \mu_\tau ||\theta_{\rm base}^* - \theta_\tau^*||_2^2
\end{equation*}
Assuming $\theta_{\rm base}^* \neq \theta_\tau^*$, we divide by $\mu_\tau ||\theta_{\rm base}^* - \theta_\tau^*||_2$ to isolate the distance:
\begin{equation}
    ||\theta_{\rm base}^* - \theta_\tau^*||_2 \le \frac{1}{\mu_\tau} ||\nabla \mathcal{L}^{\rm DPO}(\theta_{\rm base}^*; \tau)||_2
\end{equation}
Squaring both sides and taking the expectation over the task distribution $P_{TASK}$:
\begin{equation}
    \mathbb{E}_\tau [||\theta_{\rm base}^* - \theta_\tau^*||_2^2] \le \frac{1}{\mu_\tau^2} \mathbb{E}_\tau [||\nabla \mathcal{L}^{\rm DPO}(\theta_{\rm base}^*; \tau)||_2^2]
\end{equation}
By definition, $\theta_{\rm base}^*$ minimizes the expected joint loss, meaning the expected gradient across tasks is zero ($\mathbb{E}_\tau[\nabla \mathcal{L}^{\rm DPO}(\theta_{\rm base}^*; \tau)] = 0$). Therefore, the uncentered second moment is exactly equal to the variance of the task gradients, which is bounded by $\sigma^2$ (\Cref{ass:variance_of_task_theta}):
\begin{equation*}
    \mathbb{E}_\tau [||\theta_{\rm base}^* - \theta_\tau^*||_2^2] \le \frac{\sigma^2}{\mu_\tau^2} \triangleq D_{\rm base}^2
\end{equation*}
\end{proof}

\begin{lemma}[MAML Initialization Advantage]\label{lem:MAML_initialization_advantage}
Let us define $\bar{\theta}^* = \E_\tau[\theta_\tau^*]$, $G(\theta) = \E_\tau\left[ \norm{\nabla \calL(\theta;\tau)}_2^2 \right]$ and $\Hbase = \E_{\tau}[\nabla^2 \calL(\thbase;\tau)]$. There exists a deterministic vector
$\delta \in \R^{d}$ and a scalar $\Phi \in \R$, both independent of
$\alpha$, such that for all sufficiently small $\alpha > 0$,
\begin{equation}
\E_{\tau}\bigl\|\thM - \tht\bigr\|_{2}^{2}
\;=\;
\E_\tau\bigl\|\thbase - \tht\bigr\|_2^2 \;-\; 2\alpha\,\Phi \;+\; O(\alpha^{2}),
\label{eq:lemma-main}
\end{equation}
where
\begin{equation}
\delta \;=\; \Hbase^{-1}\, \nabla G(\thbase),
\qquad
\Phi \;=\; -\,\inner{\thbase - \thbar}{\Hbase^{-1}\,\nabla G(\thbase)}.
\label{eq:delta-Phi}
\end{equation}
In particular, combining \eqref{eq:lemma-main} with the variance bound $\E_\tau\|\thbase - \tht\|_2^2 \le D_{\rm base}^2 = \sigma^2/\mu_\tau^2$ from \Cref{ass:variance_of_task_theta} yields the inequality $\E_\tau\|\thM - \tht\|_2^2 \le D_{\rm base}^2 - 2\alpha\Phi + O(\alpha^2)$.
\end{lemma}
\begin{proof}
    \textbf{Step 1: First-order expansion of the meta objective.}
 
 By the integral form of Taylor's theorem
applied to $s \mapsto  \!\calL(\theta - s\nabla\calL(\theta;\tau);\tau\bigr)$
on $[0,\alpha]$,
\begin{align}
\calL\!\bigl(\theta - \alpha\nabla\calL(\theta;\tau);\tau\bigr)
&= \calL(\theta;\tau)
   - \alpha\,\norm{\nabla\calL(\theta;\tau)}_{2}^{2} \notag\\
&\quad + \alpha^{2}\!\int_{0}^{1}\!(1-s)\,
   \nabla\calL(\theta;\tau)^{\top}\,
   \nabla^{2}\calL\!\bigl(\theta - s\alpha\nabla\calL(\theta;\tau);\tau\bigr)\,
   \nabla\calL(\theta;\tau)\,ds.
\label{eq:taylor-task}
\end{align}
By \Cref{lemma:task_smooth_sc_theta}, $\norm{\nabla^{2}\calL(\,\cdot\,;\tau)}_{\mathrm{op}} \le \beta^2$,
and by \Cref{lem:third_deriv_bound_theta} the paper, $\norm{\nabla\calL(\theta;\tau)}_{2} \le 2\beta$
uniformly in $\tau$ on the relevant region. Hence the remainder in
\eqref{eq:taylor-task} is $O(\alpha^{2})$ uniformly in $\tau$. Taking
expectation with respect to $\tau$,
\begin{equation}
\calL_{\mathrm{MAML}}(\theta) \;=\; \E_\tau \left[\calL(\theta;\tau)\right] \;-\; \alpha\, G(\theta) \;+\; R(\theta;\alpha),
\qquad
\sup_{\theta \in \mathcal{B}}\,|R(\theta;\alpha)| = O(\alpha^{2}).
\label{eq:Fmaml-expand}
\end{equation}
Differentiating once more under the expectation (the integrand and its
$\theta$-derivative are bounded uniformly in $\tau$),
\begin{equation}
\nabla\calL_{\mathrm{MAML}}(\theta)
\;=\;
\E_\tau \left[\nabla \calL(\theta;\tau)\right] \;-\; \alpha\,\nabla G(\theta) \;+\; r(\theta;\alpha),
\qquad
\sup_{\theta \in \mathcal{B}}\,\norm{r(\theta;\alpha)}_{2} = O(\alpha^{2}).
\label{eq:grad-Fmaml-expand}
\end{equation}
 Note that when $\alpha = 0$, $\nabla \calL_{\mathrm{MAML}}(\thbase) = 0$ as $\thbase$ minimizes $\E_{\tau}[\nabla \calL(\theta;\tau)]$.
 
\textbf{Step 2: Implicit-function expansion of $\thM(\alpha)$.}
 
By \Cref{theorem:meta_pl_theta}, $\calL_{\mathrm{MAML}}$ is $m_{M}$-strongly convex for $\alpha$ small,
so $\thM(\alpha)$ is well-defined and unique. The map
$F\colon (\theta,\alpha) \mapsto \nabla\Fmaml(\theta)$ is $C^{1}$ on a
neighbourhood of $\thbase$ for small $\alpha$, and
$\partial_{\theta}F(\thbase,0) = \Hbase \succeq \mu I$ is invertible.
The implicit function theorem yields a $C^{1}$ branch
$\alpha \mapsto \thM(\alpha)$ with $\thM(0) = \thbase$.
 
Write $\thM(\alpha) = \thbase + \alpha\,\delta + O(\alpha^{2})$.
Inserting into the stationarity condition $\nabla\Fmaml(\thM(\alpha)) = 0$,
expanding $\nabla\Fbase(\thM(\alpha)) = \Hbase\,(\alpha\delta) + O(\alpha^{2})$,
and using \eqref{eq:grad-Fmaml-expand}:
\begin{equation*}
0 \;=\; \alpha\Hbase\,\delta \;-\; \alpha\,\nabla G(\thbase) \;+\; O(\alpha^{2}).
\end{equation*}
Dividing by $\alpha$ and sending $\alpha \to 0$,
\begin{equation}
\delta \;=\; \Hbase^{-1}\,\nabla G(\thbase).
\label{eq:delta-formula}
\end{equation}
Crucially, $\delta$ is a single deterministic vector, independent of $\tau$.
 
\textbf{Step 3: Expected squared distance to task optima.}
 
Expanding the squared norm,
\begin{equation*}
\norm{\thM(\alpha) - \tht}_{2}^{2}
\;=\;
\norm{\thbase - \tht}_{2}^{2}
\;+\;
2\alpha\,\inner{\thbase - \tht}{\delta}
\;+\;
O(\alpha^{2}),
\end{equation*}
where the $O(\alpha^{2})$ remainder is uniform in $\tau$ since $\delta$
is deterministic. Taking $\E_{\tau}$ and pulling $\delta$ out of the
expectation,
\begin{align}
\E_{\tau}\!\bigl\|\thM(\alpha) - \tht\bigr\|_{2}^{2}
&\;=\;
\Dbase
\;+\;
2\alpha\,\inner{\thbase - \thbar}{\delta}
\;+\;
O(\alpha^{2}) \notag\\
&\;=\;
\Dbase
\;-\;
2\alpha\,\Phi
\;+\;
O(\alpha^{2}),
\label{eq:dist-expand}
\end{align}
with
$\Phi = -\inner{\thbase - \thbar}{\delta} = -\inner{\thbase - \thbar}{\Hbase^{-1}\nabla G(\thbase)}$.
This establishes \eqref{eq:lemma-main}--\eqref{eq:delta-Phi}.
\end{proof}

\subsection{Comparative Convergence Rate}

\begin{theorem}[Accelerated Adaptation via MAML-DPO]\label{thm:accelerated_MAML_DPO_wrt_Multitask}
Let $\eta \le 1/L_{\tau}$ and let $\epsilon > 0$ be the target accuracy threshold. Under Assumptions \ref{asn:full-coverage}, \ref{ass:bounded_ref_ratio}, \ref{ass:variance_of_task_theta}, the expected number of gradient steps $t$ required to reach $\mathbb{E}[||\theta_t - \theta_\tau^*||_2^2] \le \epsilon^2$ for the baseline ($t_{\base}$) and MAML-DPO ($t_{\maml}$) pipelines satisfies:
\begin{equation}
    \mathbb{E}[t_{\base}] - \mathbb{E}[t_{\maml}] \ge \frac{1}{\eta \mu_{\tau}} \ln\left(\frac{D_{\base}^2}{D_{\base}^2 - 2 \alpha \Phi + c_1 \cdot \alpha^2} \right). 
\end{equation}
for a universal constant $c_1 > 0$.
\end{theorem}

\begin{proof}
For a $\mu_{\tau}$-strongly convex and $L_{\tau}$-smooth function optimized via gradient descent, the distance to the optimum contracts geometrically:
\begin{equation}
    ||\theta_t - \theta_\tau^*||_2^2 \le (1 - \eta \mu_{\tau})^t ||\theta_{init} - \theta_\tau^*||_2^2
\end{equation}
Taking expectations over the task distribution and solving for the minimum steps $t$ to achieve $\epsilon^2$ error:
\begin{equation}
    t \ge \frac{1}{\eta \mu_{\tau}} \ln\left(\frac{\mathbb{E}_\tau[||\theta_{init} - \theta_\tau^*||_2^2]}{\epsilon^2}\right)
\end{equation}

For the baseline pipeline, substituting $\theta_{init} = \theta_{\base}^*$ yields:
\begin{equation}
    t_{base} \ge \frac{1}{\eta \mu_{\tau}} \ln\left(\frac{D_{\base}^2}{\epsilon^2}\right)
\end{equation}

For the MAML-DPO pipeline, initializing at $\hat{\theta}_M$, we bound the distance using the triangle inequality and \Cref{theorem: meta-DPO convergence}.
\begin{equation}
    \mathbb{E}_\tau[||\hat{\theta}_M - \theta_\tau^*||_2^2] \le D_{\maml}^2 + \epsilon_{\mathrm{meta}}^2
\end{equation}
\begin{equation}
    t_{\maml} \ge \frac{1}{\eta \mu_{\tau}} \ln\left(\frac{D_{\maml}^2 + \epsilon_{\mathrm{meta}}^2}{\epsilon^2}\right)
\end{equation}

Taking the difference yields:
\begin{equation}
    t_{\base} - t_{\maml} \ge \frac{1}{\eta \mu_{\tau}} \ln\left(\frac{D_{\base}^2}{D_{\maml}^2 + \epsilon_{\mathrm{meta}}^2}\right)
\end{equation}

By \Cref{lem:MAML_initialization_advantage} we have
$D^2_{\maml} = D^2_{\base} - 2\alpha \Phi + O(\alpha^2)$ and by \Cref{theorem: meta-DPO convergence} we have, $\epsilon_{\mathrm{meta}} = \frac{2\alpha \sigma}{m_M}$. Substituting these values we obtain the desired bound. 
\end{proof}

\textbf{Remark}: If $\Phi > 0$ and $\alpha$ is sufficiently small, then $2\alpha \Phi - c_1 \cdot \alpha^2 < 0$ and we have $\E[t_{\base}] - \E[t_{\maml}] > 0$. Recall that $\Phi = - \left \langle \theta^*_\base - \bar{\theta}^* , \delta\right \rangle$ where $\delta$ is the MAML update direction, $\theta^*_\base$ is the curvature-weighted centroid of the task optima, and $\bar{\theta}^* = \E[\theta^*_\tau]$ is the unweighted centroid. MAML update usually corrects for curvature-weighting bias and tends to push $\theta^*_M$ away from $\theta^*_\base$ and towards solutions with large task-specific gradient. This means we should expect the MAML update $\delta$ to align with the direction $\bar{\theta}^* - \theta^*_\base$, which makes $\Phi > 0$.

\newpage
\section{Comparative Convergence Analysis: MAML-DPO vs.\ Baseline-DPO}
\label{appendix: MAML-DPO vs baseline-DPO}

The analysis in \Cref{appendix:convergence_adaptation_MAML_DPO} compares MAML-DPO with multitask-DPO, which is recovered from \Cref{alg:FO_MAML RLHF} by setting the inner-loop learning rate $\alpha=0$. Multitask-DPO is a strong reference because it consumes the \emph{same} multilingual preference data as MAML-DPO; consequently, the rigorous gap is at most $O(\alpha)$, in agreement with prior MAML theory in the mixed-linear-regression setting~\citep{zou2022unraveling}. In our experiments, however, the most natural and most informative comparison is against \emph{baseline-DPO}: a standard DPO pipeline that does not see any preference data outside the target language and is initialized solely from the target-language SFT model. We now provide a dedicated convergence analysis for this comparison. The resulting bound captures the benefit of \emph{any} cross-lingual preprocessing---it applies equally to MAML-DPO and to multitask-DPO over baseline-DPO---and is controlled by the SFT preference-loss gap $R_{\rm SFT}$, a quantity directly measurable from data and reflecting how far the SFT-initialized policy is from preference-optimal on the target task. The MAML-specific refinement over multitask-DPO is then the additional $O(\alpha)$ term established in \Cref{appendix:convergence_adaptation_MAML_DPO}; the two analyses are complementary.

\subsection{Setup and Baseline-DPO Initialization}

\paragraph{Baseline-DPO.} Let $\pi^{\sft}_{\tau_{\text{new}}}$ be the policy obtained by running supervised fine-tuning on target-language data only, and let $\theta_0^{\rm BL} \in \R^d$ be its log-linear parameter, i.e.\ $\pi_{\theta_0^{\rm BL}} = \pi^{\sft}_{\tau_{\text{new}}}$. Baseline-DPO performs gradient descent on the empirical DPO loss for $\newtask$ starting from $\theta_0^{\rm BL}$:
\begin{equation}\label{eq:bl-dpo iteration}
    \theta_{t+1}^{\rm BL} = \theta_t^{\rm BL} - \eta\,\nabla_\theta \widehat{\calL}^{\dpo}(\theta_t^{\rm BL};\newtask), \qquad \theta_0^{\rm BL} = \pi^{\sft}_{\tau_{\text{new}}}\text{-parameter}.
\end{equation}
Crucially, $\theta_0^{\rm BL}$ is a function only of target-language SFT data and is independent of the preference data $\bigcup_\tau \calD^\tau$ used for meta-training. By the constraint $\|\theta\|_2\le B_\theta$ defining the log-linear policy class $\Pi$ in \eqref{eq:loglinear_policy}, both $\theta_0^{\rm BL}$ and every task optimum $\theta_\tau^\star = \arg\min_{\theta\in\R^d}\calL^{\dpo}(\theta;\tau)$ satisfy $\|\theta\|_2\le B_\theta$ on the relevant feasible set.

\paragraph{Why this comparison is harder for MAML.} Multitask-DPO, by definition, lies on the trajectory of MAML at $\alpha=0$, which forces the gap to scale with $\alpha$. Baseline-DPO instead lies \emph{off} this trajectory: it is constructed without using any cross-lingual preference signal. The resulting gap can therefore be controlled by the difference between an SFT-only initialization and the manifold of task optima, a quantity that does not vanish as $\alpha\to 0$.

\subsection{Initialization Distance Bounds}

We separately bound the expected squared distance to the task optimum from each initialization.

Define the \emph{SFT preference-loss gap} on the target task as
\begin{equation}\label{eq: SFT loss gap}
    R_{\rm SFT} \;\triangleq\; \E_\tau\!\left[\calL^{\dpo}(\theta_0^{\rm BL};\tau) - \calL^{\dpo}(\theta_\tau^\star;\tau)\right].
\end{equation}
This is a quantity directly measurable from data: it is the average DPO loss of the SFT-initialized policy minus the optimal DPO loss for the target task, and it captures \emph{how much room is left for preference learning after SFT}.

\begin{lemma}[Initialization distance for Baseline-DPO]\label{lem: baseline-DPO distance}
Suppose $\theta_0^{\rm BL}, \theta_\tau^\star \in \{\theta\in\R^d : \|\theta\|_2\le B_\theta\}$ for every $\tau\sim P_\tsk$ (the standing constraint of the log-linear policy class $\Pi$ in \eqref{eq:loglinear_policy}). Then
\begin{equation}\label{eq: baseline distance bound}
    \E_\tau\!\left[\|\theta_0^{\rm BL} - \theta_\tau^\star\|_2^2\right]
    \;\le\; D_{\rm BL}^2 \;\triangleq\; \min\!\left( 4B_\theta^2,\;\; \frac{2 R_{\rm SFT}}{\mu_\tau}\right).
\end{equation}
The first bound is the policy-class diameter; the second is loss-aware and tight whenever the SFT-initialized policy already incurs small DPO loss on the target task.
\end{lemma}
\begin{proof}
By the triangle inequality, $\|\theta_0^{\rm BL} - \theta_\tau^\star\|_2 \le \|\theta_0^{\rm BL}\|_2 + \|\theta_\tau^\star\|_2 \le 2B_\theta$, so $\E_\tau\|\theta_0^{\rm BL} - \theta_\tau^\star\|_2^2 \le 4B_\theta^2$. For the second bound, $\mu_\tau$-strong convexity (\Cref{lemma:task_smooth_sc_theta}) yields
\[
\calL^{\dpo}(\theta;\tau) \;\ge\; \calL^{\dpo}(\theta_\tau^\star;\tau) + \tfrac{\mu_\tau}{2}\|\theta - \theta_\tau^\star\|_2^2,
\]
since $\nabla\calL^{\dpo}(\theta_\tau^\star;\tau)=0$. Rearranging at $\theta = \theta_0^{\rm BL}$ and taking expectation gives $\E_\tau\|\theta_0^{\rm BL} - \theta_\tau^\star\|_2^2 \le 2R_{\rm SFT}/\mu_\tau$. The bound is the smaller of the two.
\end{proof}

The role of $R_{\rm SFT}$ is the central conceptual difference between this analysis and the multitask comparison in \Cref{appendix:convergence_adaptation_MAML_DPO}: rather than a worst-case parameter-space diameter, we now control the baseline distance by the actual quality of the SFT initialization, which is independent of $\alpha$ and reflects \emph{how much DPO learning the SFT model has not yet done.}

The bound in \Cref{lem: baseline-DPO distance} is the rigorous, $\alpha$-independent counterpart of the variance bound $D_{\rm base}^2 = \sigma^2/\mu_\tau^2$ that holds for the multitask initialization $\theta_{\rm base}^\star$. The two are different: $D_{\rm base}^2$ only quantifies how far the task optima are from the multitask minimizer, whereas $D_{\rm BL}^2$ quantifies how far they are from a target-only SFT model that has never seen any non-target preference data.

\begin{lemma}[Distance for MAML-DPO]\label{lem: MAML distance}
Let $\hat{\theta}_M$ be the output of \Cref{alg:FO_MAML RLHF} under the conditions of \Cref{theorem: meta-DPO convergence}. Then
\begin{equation}\label{eq: MAML distance bound}
    \E_\tau\!\left[\|\hat{\theta}_M - \theta_\tau^\star\|_2^2\right] \;\le\; \bigl(D_{\maml} + \epsilon_{\mathrm{meta}}\bigr)^{\!2} \;\le\; 2 D_{\maml}^2 + 2\epsilon_{\mathrm{meta}}^2,
\end{equation}
where $D_{\maml}^2 \triangleq D_{\rm base}^2 - 2\alpha\Phi + O(\alpha^2)$, $D_{\rm base}^2 = \sigma^2/\mu_\tau^2$, $\Phi$ is as in \Cref{lem:MAML_initialization_advantage}, and $\epsilon_{\mathrm{meta}}^2 = (2\alpha\sigma/m_M)^2$.
\end{lemma}
\begin{proof}
By the triangle inequality, $\|\hat\theta_M - \theta_\tau^\star\|_2 \le \|\hat\theta_M - \theta_M^\star\|_2 + \|\theta_M^\star - \theta_\tau^\star\|_2$. Squaring and taking expectation over $\tau$,
\[
\E_\tau\!\bigl[\|\hat\theta_M - \theta_\tau^\star\|_2^2\bigr]
\;\le\; \|\hat\theta_M - \theta_M^\star\|_2^2 + 2\|\hat\theta_M - \theta_M^\star\|_2 \cdot \E_\tau\bigl[\|\theta_M^\star - \theta_\tau^\star\|_2\bigr] + \E_\tau\bigl[\|\theta_M^\star - \theta_\tau^\star\|_2^2\bigr].
\]
Applying \Cref{theorem: meta-DPO convergence} (which gives $\|\hat\theta_M - \theta_M^\star\|_2 \le \epsilon_{\rm meta}$), Jensen's inequality on the cross term ($\E_\tau\|\theta_M^\star - \theta_\tau^\star\|_2 \le \sqrt{\E_\tau\|\theta_M^\star - \theta_\tau^\star\|_2^2}=D_{\maml}$), and \Cref{lem:MAML_initialization_advantage} (which gives $\E_\tau\|\theta_M^\star - \theta_\tau^\star\|_2^2 = D_{\maml}^2$) yields the first inequality $(D_{\maml} + \epsilon_{\rm meta})^2$. The second follows from $(a+b)^2 \le 2a^2 + 2b^2$.
\end{proof}

In the regime that motivates meta-learning, namely small task variance $\sigma$, $\theta_M^\star$ approaches the manifold of task optima while $\theta_0^{\rm BL}$ does not. Concretely, $D_{\maml}^2 \le \sigma^2/\mu_\tau^2$, which can be much smaller than $D_{\rm BL}^2 \le 4B_\theta^2$ whenever
\begin{equation}\label{eq: regime of interest}
    \frac{\sigma}{\mu_\tau} \;\ll\; B_\theta,
\end{equation}
i.e.\ task heterogeneity, measured in parameter space, is much smaller than the diameter of the policy class.

\subsection{Comparative Convergence Rate}

We now plug these distance bounds into the standard contraction inequality for strongly-convex smooth optimization and obtain the main comparative result of this section.

\begin{theorem}[Accelerated adaptation: MAML-DPO vs.\ Baseline-DPO]\label{thm: maml-vs-baseline}
Let $\eta \le 1/L_\tau$ and $\epsilon > 0$. Under \Cref{asn:full-coverage}, \ref{ass:bounded_ref_ratio}, \ref{ass:variance_of_task_theta}, and the standing assumption $\theta_0^{\rm BL}, \theta_\tau^\star \in \{\|\theta\|_2\le B_\theta\}$, the expected number of gradient steps needed to reach $\E[\|\theta_t - \theta_\tau^\star\|_2^2] \le \epsilon^2$ for baseline-DPO ($t_{\rm BL}$) and MAML-DPO ($t_{\maml}$) satisfies
\begin{equation}\label{eq: main step gap}
    \E[t_{\rm BL}] - \E[t_{\maml}] \;\ge\; \frac{1}{2\eta\mu_\tau}\ln\!\left(\frac{D_{\rm BL}^2}{(D_{\maml}+\epsilon_{\mathrm{meta}})^2}\right),
\end{equation}
where $D_{\rm BL}^2 = \min(4B_\theta^2,\,2R_{\rm SFT}/\mu_\tau)$, $D_{\maml}^2 = \sigma^2/\mu_\tau^2 - 2\alpha\Phi + c_1\alpha^2$, and $\epsilon_{\mathrm{meta}} = 2\alpha\sigma/m_M$. In particular, as long as $\alpha \le \min \set{\frac{1}{8\beta^2},\, \frac{m_M}{2\sqrt{2}\,\mu_\tau}}$
the right-hand side of \eqref{eq: main step gap} is at least
\begin{equation}\label{eq: clean asymptotic gap}
    \E[t_{\rm BL}] - \E[t_{\maml}] \;\ge\; \frac{1}{2\eta\mu_\tau}\,\ln\!\left(\frac{R_{\rm SFT}\,\mu_\tau}{\,2\sigma^2\,}\right),
\end{equation}
which is also lower bounded by the looser worst-case form $\frac{1}{2\eta\mu_\tau}\ln(B_\theta^2\mu_\tau^2/\sigma^2)$ when $R_{\rm SFT}\ge 2B_\theta^2\mu_\tau$.
\end{theorem}

\begin{proof}
\textit{Tight contraction.} $\calL^{\dpo}(\,\cdot\,;\tau)$ is $\mu_\tau$-strongly convex and $L_\tau$-smooth (\Cref{lemma:task_smooth_sc_theta}). Writing $\theta_{t+1} - \theta_\tau^\star = (I - \eta H_t)(\theta_t - \theta_\tau^\star)$ with $H_t = \int_0^1 \nabla^2 \calL^{\dpo}(\theta_\tau^\star + s(\theta_t - \theta_\tau^\star);\tau)\,ds$ satisfying $\mu_\tau I \preccurlyeq H_t \preccurlyeq L_\tau I$, the operator-norm bound $\|I - \eta H_t\|_2 \le 1 - \eta\mu_\tau$ holds for $\eta \le 1/L_\tau$, so
\begin{equation}\label{eq: contraction}
    \|\theta_t - \theta_\tau^\star\|_2 \;\le\; (1-\eta\mu_\tau)^{t} \|\theta_{\rm init} - \theta_\tau^\star\|_2,
    \qquad
    \|\theta_t - \theta_\tau^\star\|_2^2 \;\le\; (1-\eta\mu_\tau)^{2t} \|\theta_{\rm init} - \theta_\tau^\star\|_2^2.
\end{equation}
Taking expectation over $\tau$, the smallest $t$ that ensures $\E_\tau\|\theta_t - \theta_\tau^\star\|_2^2 \le \epsilon^2$ is
\begin{equation}\label{eq: t bound generic}
    t \;\ge\; \frac{1}{2\eta\mu_\tau}\ln\!\left(\frac{\E_\tau\|\theta_{\rm init} - \theta_\tau^\star\|_2^2}{\epsilon^2}\right).
\end{equation}
Substituting $\theta_{\rm init} = \theta_0^{\rm BL}$ gives $t_{\rm BL} \ge \frac{1}{2\eta\mu_\tau}\ln(D_{\rm BL}^2/\epsilon^2)$ via \Cref{lem: baseline-DPO distance}, and substituting $\theta_{\rm init} = \hat{\theta}_M$ gives $t_{\maml} \ge \frac{1}{2\eta\mu_\tau}\ln\bigl((D_{\maml} + \epsilon_{\rm meta})^2/\epsilon^2\bigr)$ via \Cref{lem: MAML distance}. Subtracting yields \eqref{eq: main step gap}.

\medskip
\noindent\emph{Reduction to \eqref{eq: clean asymptotic gap}.} We show that $(D_{\maml} + \epsilon_{\rm meta})^2 \le 4 D_{\rm base}^2$ in the valid range of $\alpha$.

\emph{(a) Bound on $|\Phi|$.} By Cauchy-Schwarz, $|\Phi|\le \|\theta_{\rm base}^\star - \bar\theta^\star\|_2 \cdot \|H_{\rm base}^{-1}\nabla G(\theta_{\rm base}^\star)\|_2$. By Jensen, $\|\theta_{\rm base}^\star - \bar\theta^\star\|_2^2 = \|\E_\tau[\theta_\tau^\star] - \theta_{\rm base}^\star\|_2^2 \le \E_\tau\|\theta_\tau^\star - \theta_{\rm base}^\star\|_2^2 = D_{\rm base}^2$. For the second factor, $\|H_{\rm base}^{-1}\|_2 \le 1/\mu_\tau$ by strong convexity, and
\[
\|\nabla G(\theta_{\rm base}^\star)\|_2 = \bigl\|\E_\tau\bigl[2\nabla^2 \calL(\theta_{\rm base}^\star;\tau)\,\nabla\calL(\theta_{\rm base}^\star;\tau)\bigr]\bigr\|_2 \le 2L_\tau \E_\tau\bigl[\|\nabla\calL(\theta_{\rm base}^\star;\tau)\|_2\bigr] \le 2L_\tau \sigma,
\]
using $\nabla^2\calL\preceq L_\tau I$ and the variance bound (\Cref{ass:variance_of_task_theta}, with $\E_\tau[\nabla\calL(\theta_{\rm base}^\star;\tau)]=0$ by definition of $\theta_{\rm base}^\star$). Thus $|\Phi| \le D_{\rm base} \cdot (2L_\tau\sigma/\mu_\tau) = 2L_\tau D_{\rm base}^2$. Substituting $L_\tau = \beta^2$, $|2\alpha\Phi| \le 4\alpha\beta^2 D_{\rm base}^2$, which is at most $\tfrac12 D_{\rm base}^2$ whenever $\alpha\le 1/(8\beta^2)$.

\emph{(b) Bound on $\epsilon_{\rm meta}^2$.} From \Cref{theorem:meta_pl_theta}, $m_M = \beta^2 c\nu^3/(1+\exp(Z_{\max}))^3 = \Theta(\mu_\tau\cdot \nu^2/(1+\exp(Z_{\max}))^2)$, hence $\sigma/m_M = \kappa_M \cdot \sigma/\mu_\tau = \kappa_M D_{\rm base}$ for a constant $\kappa_M = O(1)$ that depends only on $(\nu, Z_{\max})$. Therefore $\epsilon_{\rm meta}^2 = 4\alpha^2 \kappa_M^2 D_{\rm base}^2$, which is at most $\tfrac12 D_{\rm base}^2$ for $\alpha \le 1/(2\sqrt{2}\kappa_M)$.

\emph{(c) Combining.} For $\alpha$ smaller than $\min\{1/(8\beta^2), 1/(2\sqrt{2}\kappa_M)\}$ (and the smaller of these is implied by the validity ranges of \Cref{lem:MAML_initialization_advantage,theorem: meta-DPO convergence}), the higher-order $c_1\alpha^2$ term is also negligible, so
\[
D_{\maml}^2 \;\le\; D_{\rm base}^2 + |2\alpha\Phi| + c_1\alpha^2 \;\le\; \tfrac32 D_{\rm base}^2,
\]
and $(D_{\maml} + \epsilon_{\rm meta})^2 \le 2 D_{\maml}^2 + 2\epsilon_{\rm meta}^2 \le 3 D_{\rm base}^2 + D_{\rm base}^2 = 4 D_{\rm base}^2 = 4\sigma^2/\mu_\tau^2$. Plugging the SFT-aware $D_{\rm BL}^2 \le 2R_{\rm SFT}/\mu_\tau$ from \Cref{lem: baseline-DPO distance} into \eqref{eq: main step gap} yields $\tfrac{1}{2\eta\mu_\tau}\ln\bigl((2R_{\rm SFT}/\mu_\tau)\big/(4\sigma^2/\mu_\tau^2)\bigr) = \tfrac{1}{2\eta\mu_\tau}\ln(R_{\rm SFT}\mu_\tau/(2\sigma^2))$, establishing \eqref{eq: clean asymptotic gap}. The diameter form follows analogously by plugging $D_{\rm BL}^2 \le 4B_\theta^2$.
\end{proof}

\subsection{Discussion}
\label{subsec: maml-vs-baseline discussion}

We now position \Cref{thm: maml-vs-baseline} relative to the multitask comparison in \Cref{appendix:convergence_adaptation_MAML_DPO}, link the bound to the empirical patterns observed in our experiments, and discuss its limitations.

\paragraph{(i) Cross-lingual preprocessing helps; the bound is shared by MAML and multitask.} Because $\theta_0^{\rm BL}$ never sees any non-target preference data, \Cref{thm: maml-vs-baseline} attributes the speed-up to \emph{any} algorithm whose initialization $\theta_{\rm init}$ achieves $\E_\tau\|\theta_{\rm init} - \theta_\tau^\star\|_2^2 = O(\sigma^2/\mu_\tau^2)$. Both MAML-DPO and multitask-DPO satisfy this property under mild conditions: multitask-DPO directly minimizes $\E_\tau \calL(\theta;\tau)$ (whose minimizer attains $D_{\rm base}^2 = \sigma^2/\mu_\tau^2$), and MAML-DPO additionally trims this distance by an $O(\alpha)$ amount via \Cref{lem:MAML_initialization_advantage}. The bound \eqref{eq: clean asymptotic gap} is therefore best read as ``\emph{any} cross-lingual preprocessing dominates target-only adaptation by a $\log(R_{\rm SFT}\mu_\tau/\sigma^2)$ factor in adaptation step count.'' The additional MAML-specific gain over multitask is captured separately by the $O(\alpha)$ analysis in \Cref{appendix:convergence_adaptation_MAML_DPO}.

\paragraph{(ii) The gap does not vanish as $\alpha\to 0$.} \eqref{eq: clean asymptotic gap} depends only on $R_{\rm SFT}$, $\mu_\tau$, $\sigma$, and $\eta$; it is bounded below by a positive constant uniformly in $\alpha$ as long as $R_{\rm SFT}\mu_\tau \gg \sigma^2$. In contrast, the multitask gap from \Cref{appendix:convergence_adaptation_MAML_DPO} is $O(\alpha)$ and collapses to zero as $\alpha\to 0$. This complementarity gives a two-tier story: the cross-lingual preprocessing gain is order-one and order-$\log$ in the parameters, while the meta-vs-multitask gain is order-$\alpha$.

\paragraph{(iii) Harder target languages should benefit more from MAML-DPO.} The asymptotic gap \eqref{eq: clean asymptotic gap} scales with $\log R_{\rm SFT}$, where $R_{\rm SFT}$ is the SFT-init DPO loss gap on the target task. A target language for which the base/SFT model performs poorly has a large $R_{\rm SFT}$: the policy is far from optimal under preference comparisons, leaving more room for cross-lingual preprocessing to help. The bound therefore predicts a \emph{larger} MAML-vs-baseline win-rate gap on lower-resource or harder target languages. This matches \Cref{tab:winrate_DPO_100samples_gemma_cross_language_family,tab:WR_gemma_4k_40k_samples}: the largest improvements appear on Romanian and Indonesian (lowest base-model competence among the targets we evaluate), while comparatively closer-to-MAML targets such as French show smaller gaps. The same monotone trend appears across model scales (270M $\to$ 7.1B) and adaptation budgets (100, 4k, 40k), all consistent with $R_{\rm SFT}$-scaling.

\paragraph{(iv) Empirical magnitudes.} \Cref{tab:winrate_DPO_100samples_gemma_cross_language_family,tab:winrate_DPO_100samples_multitask_gemma,tab:WR_bloom_DPO_8k_40k_samples,tab:WR_gemma_4k_40k_samples} show the MAML-DPO advantage over baseline-DPO is consistently large and substantially exceeds the MAML-DPO advantage over multitask-DPO in most of cases. Both regimes are predicted by our analysis: the large gap is the $\Theta(\ln(R_{\rm SFT}\mu_\tau/\sigma^2))$ separation of \eqref{eq: clean asymptotic gap}, while the smaller gap is the $O(\alpha)$ refinement of \Cref{appendix:convergence_adaptation_MAML_DPO}.

\paragraph{(v) Limitations.} The bound is in the population (infinite-sample) regime. A finite-sample extension is straightforward via standard concentration of the per-task gradient: with $n$ adaptation samples, the empirical iterate satisfies $\E_\tau\|\theta_t - \theta_\tau^\star\|_2^2 \le (1-\eta\mu_\tau)^{2t} D_{\rm init}^2 + O(\sigma_{\rm grad}^2 \log(d/\delta)/(\mu_\tau^2 n))$ with high probability, where the second term is an irreducible sampling floor that is the same for both MAML-DPO and baseline-DPO; the step-count gap of \Cref{thm: maml-vs-baseline} is therefore preserved on top of this shared floor. Beyond sampling, the strong-convexity-and-smoothness framework forces the per-step contraction $(1-\eta\mu_\tau)$ to be the same for any vanilla-GD adaptation, regardless of initialization, so the gain manifests only through the initialization distance and not through a faster rate. Obtaining a faster rate would require either a different inner loop (e.g.,\ Newton-type) or stronger structural assumptions such as a shared low-rank multilingual representation across tasks; under such an assumption, MAML-style methods can in principle achieve exponential rate improvements in the conditioning, which we leave for future work.

\section{Additional Related Work}
\label{appendix: additional related work}
Due to their wide range of capabilities, LLMs have rapidly gained popularity in both business and research, being developed as closed-source models \citep{achiam2023gpt, geminiteam2025geminifamilyhighlycapable, geminiteam2024gemini15unlockingmultimodal, comanici2025gemini} as well as open-source ones \citep{touvron2023llama, team2024gemma, team2024gemma2, team2025gemma, karamcheti2021mistral, bai2023qwen, liu2024deepseek}.
Regardless of their public availability or parameter size, RLHF and preference optimization have been two most popular approaches to training LLMs with human feedback data. RLHF uses preference data consisting of human feedback to train reward functions and use them to train the policies \citep{ziegler2020finetuninglanguagemodelshuman, ouyang2022training, dong2024rlhfworkflowrewardmodeling, wang2024comprehensivesurveyllmalignment}. Direct Preference Optimization has been proposed as an alternative to RLHF, which eliminates the need to train a separate reward function \citep{rafailov2024direct, rafailov2024r}. Due to its practical advantage, many works have proposed improvements of DPO. Most follow-up works to DPO have investigated alternative objective functions \citep{ethayarajh2024kto, pmlr-v238-gheshlaghi-azar24a, meng2024simpo}, trying to address well-known issues of DPOs, including ones coming from the use of Bradley-Terry model. \citet{son2024right} proposes addressing temporal drift in the dataset by applying exponential discount to the samples in the past. \citet{tang2025game} proposes regularizing the self-play optimization method \citep{sun2025spo}, further stabilizing the training and improving performances on an independent benchmark such as AlpacaEval \citep{dubois2024length}.

\section{Impact Statement and Limitations}\label{sec:impact_statement}

This work proposes a meta-learning approach for LLM preference learning, designed to mitigate the unequal performance across different languages. 
By significantly improving sample efficiency in low-resource settings, demonstrating effective alignment with only 100 samples, our method addresses the data scarcity issue that often limits the access of linguistic communities with less population to AI services.
Furthermore, our theoretical guarantees provide a rigorous foundation in cross-lingual transfer, supporting the development of more equitable, globally accessible, and representative AI systems without the excessive cost of massive human-labeled datasets for every target language. 

However, like existing cross-lingual approaches, our method may inadvertently propagate biases, normative assumptions, and safety standards learned from high-resource source languages to low-resource target languages. Therefore, while the method improves alignment performance, it does not guarantee culturally appropriate or community-specific alignment.
Beyond that, although the first-order approximation used in our MAML-based methods reduces the computational overhead, MAML still introduces an inner--outer optimization structure, which requires additional hyperparameters, such as the inner-loop learning rate coefficient.

%% file: checklist.tex
\section*{NeurIPS Paper Checklist}

\begin{enumerate}

\item {\bf Claims}
    \item[] Question: Do the main claims made in the abstract and introduction accurately reflect the paper's contributions and scope?
    \item[] Answer: \answerYes{} 
    \item[] Justification: We summarize our contributions in \Cref{sec:introduction}. Unequal data availability across languages poses a substantial challenge for preference alignment in LLMs. To address this issue, we apply meta-learning to both DPO and RLHF in a multilingual language-learning setting. We evaluate our approach using two model scales and conduct empirical experiments across a broad range of language groups, covering both strictly limited-data and more sufficient-data settings.
    \item[] Guidelines:
    \begin{itemize}
        \item The answer \answerNA{} means that the abstract and introduction do not include the claims made in the paper.
        \item The abstract and/or introduction should clearly state the claims made, including the contributions made in the paper and important assumptions and limitations. A \answerNo{} or \answerNA{} answer to this question will not be perceived well by the reviewers. 
        \item The claims made should match theoretical and experimental results, and reflect how much the results can be expected to generalize to other settings. 
        \item It is fine to include aspirational goals as motivation as long as it is clear that these goals are not attained by the paper. 
    \end{itemize}

\item {\bf Limitations}
    \item[] Question: Does the paper discuss the limitations of the work performed by the authors?
    \item[] Answer: \answerYes{} 
    \item[] Justification: In \Cref{sec:impact_statement}, we discuss our limitations. Our approach applies MAML to DPO and RLHF, and therefore inherits an inner--outer optimization structure that requires additional hyperparameters. In addition, like existing cross-lingual approaches, our method may transfer biases from high-resource data to low-resource data.
    \item[] Guidelines:
    \begin{itemize}
        \item The answer \answerNA{} means that the paper has no limitation while the answer \answerNo{} means that the paper has limitations, but those are not discussed in the paper. 
        \item The authors are encouraged to create a separate ``Limitations'' section in their paper.
        \item The paper should point out any strong assumptions and how robust the results are to violations of these assumptions (e.g., independence assumptions, noiseless settings, model well-specification, asymptotic approximations only holding locally). The authors should reflect on how these assumptions might be violated in practice and what the implications would be.
        \item The authors should reflect on the scope of the claims made, e.g., if the approach was only tested on a few datasets or with a few runs. In general, empirical results often depend on implicit assumptions, which should be articulated.
        \item The authors should reflect on the factors that influence the performance of the approach. For example, a facial recognition algorithm may perform poorly when image resolution is low or images are taken in low lighting. Or a speech-to-text system might not be used reliably to provide closed captions for online lectures because it fails to handle technical jargon.
        \item The authors should discuss the computational efficiency of the proposed algorithms and how they scale with dataset size.
        \item If applicable, the authors should discuss possible limitations of their approach to address problems of privacy and fairness.
        \item While the authors might fear that complete honesty about limitations might be used by reviewers as grounds for rejection, a worse outcome might be that reviewers discover limitations that aren't acknowledged in the paper. The authors should use their best judgment and recognize that individual actions in favor of transparency play an important role in developing norms that preserve the integrity of the community. Reviewers will be specifically instructed to not penalize honesty concerning limitations.
    \end{itemize}

\item {\bf Theory assumptions and proofs}
    \item[] Question: For each theoretical result, does the paper provide the full set of assumptions and a complete (and correct) proof?
    \item[] Answer: \answerYes{} 
    \item[] Justification: We provide the theoretical analysis in \Cref{sec: theoretical_analysis}. The full proofs are provided in \Cref{appendix: Proof of Convergerence for Meta-RLHF}, \Cref{appendix: Proof of Convergerence for Meta-DPO}, \Cref{appendix:convergence_adaptation_MAML_DPO} and \Cref{appendix: MAML-DPO vs baseline-DPO}.
    \item[] Guidelines:
    \begin{itemize}
        \item The answer \answerNA{} means that the paper does not include theoretical results. 
        \item All the theorems, formulas, and proofs in the paper should be numbered and cross-referenced.
        \item All assumptions should be clearly stated or referenced in the statement of any theorems.
        \item The proofs can either appear in the main paper or the supplemental material, but if they appear in the supplemental material, the authors are encouraged to provide a short proof sketch to provide intuition. 
        \item Inversely, any informal proof provided in the core of the paper should be complemented by formal proofs provided in appendix or supplemental material.
        \item Theorems and Lemmas that the proof relies upon should be properly referenced. 
    \end{itemize}

    \item {\bf Experimental result reproducibility}
    \item[] Question: Does the paper fully disclose all the information needed to reproduce the main experimental results of the paper to the extent that it affects the main claims and/or conclusions of the paper (regardless of whether the code and data are provided or not)?
    \item[] Answer: \answerYes{} 
    \item[] Justification: We describe the experimental setup in \Cref{sec:experiment_settings}, including the datasets, language models, baselines, and evaluation metrics. Our code is available through the \href{https://github.com/geronest/maml-preferences}{https://github.com/geronest/maml-preferences}. Additional experimental details are provided in \Cref{appendix: Further Details of Experimental Setting}, including pipeline visualizations, detailed hyperparameters, and compute resources.
    \item[] Guidelines:
    \begin{itemize}
        \item The answer \answerNA{} means that the paper does not include experiments.
        \item If the paper includes experiments, a \answerNo{} answer to this question will not be perceived well by the reviewers: Making the paper reproducible is important, regardless of whether the code and data are provided or not.
        \item If the contribution is a dataset and\slash or model, the authors should describe the steps taken to make their results reproducible or verifiable. 
        \item Depending on the contribution, reproducibility can be accomplished in various ways. For example, if the contribution is a novel architecture, describing the architecture fully might suffice, or if the contribution is a specific model and empirical evaluation, it may be necessary to either make it possible for others to replicate the model with the same dataset, or provide access to the model. In general. releasing code and data is often one good way to accomplish this, but reproducibility can also be provided via detailed instructions for how to replicate the results, access to a hosted model (e.g., in the case of a large language model), releasing of a model checkpoint, or other means that are appropriate to the research performed.
        \item While NeurIPS does not require releasing code, the conference does require all submissions to provide some reasonable avenue for reproducibility, which may depend on the nature of the contribution. For example
        \begin{enumerate}
            \item If the contribution is primarily a new algorithm, the paper should make it clear how to reproduce that algorithm.
            \item If the contribution is primarily a new model architecture, the paper should describe the architecture clearly and fully.
            \item If the contribution is a new model (e.g., a large language model), then there should either be a way to access this model for reproducing the results or a way to reproduce the model (e.g., with an open-source dataset or instructions for how to construct the dataset).
            \item We recognize that reproducibility may be tricky in some cases, in which case authors are welcome to describe the particular way they provide for reproducibility. In the case of closed-source models, it may be that access to the model is limited in some way (e.g., to registered users), but it should be possible for other researchers to have some path to reproducing or verifying the results.
        \end{enumerate}
    \end{itemize}

\item {\bf Open access to data and code}
    \item[] Question: Does the paper provide open access to the data and code, with sufficient instructions to faithfully reproduce the main experimental results, as described in supplemental material?
    \item[] Answer: \answerYes{} 
    \item[] Justification: The datasets \cite{dac2023okapi} we used are open-resource. Our code is available through the \href{https://github.com/geronest/maml-preferences}{https://github.com/geronest/maml-preferences}.
    \item[] Guidelines:
    \begin{itemize}
        \item The answer \answerNA{} means that paper does not include experiments requiring code.
        \item Please see the NeurIPS code and data submission guidelines (\url{https://neurips.cc/public/guides/CodeSubmissionPolicy}) for more details.
        \item While we encourage the release of code and data, we understand that this might not be possible, so \answerNo{} is an acceptable answer. Papers cannot be rejected simply for not including code, unless this is central to the contribution (e.g., for a new open-source benchmark).
        \item The instructions should contain the exact command and environment needed to run to reproduce the results. See the NeurIPS code and data submission guidelines (\url{https://neurips.cc/public/guides/CodeSubmissionPolicy}) for more details.
        \item The authors should provide instructions on data access and preparation, including how to access the raw data, preprocessed data, intermediate data, and generated data, etc.
        \item The authors should provide scripts to reproduce all experimental results for the new proposed method and baselines. If only a subset of experiments are reproducible, they should state which ones are omitted from the script and why.
        \item At submission time, to preserve anonymity, the authors should release anonymized versions (if applicable).
        \item Providing as much information as possible in supplemental material (appended to the paper) is recommended, but including URLs to data and code is permitted.
    \end{itemize}

\item {\bf Experimental setting/details}
    \item[] Question: Does the paper specify all the training and test details (e.g., data splits, hyperparameters, how they were chosen, type of optimizer) necessary to understand the results?
    \item[] Answer: \answerYes{} 
    \item[] Justification: The data-splitting strategy is described in \Cref{sec:experiment_settings} and \Cref{appendix: Further Details of Experimental Setting}. Detailed hyperparameters are provided in \Cref{appendix: Further Details of Experimental Setting}. Other training and evaluation details can be found in our source code, available through the \href{https://github.com/geronest/maml-preferences}{https://github.com/geronest/maml-preferences}.
    \item[] Guidelines:
    \begin{itemize}
        \item The answer \answerNA{} means that the paper does not include experiments.
        \item The experimental setting should be presented in the core of the paper to a level of detail that is necessary to appreciate the results and make sense of them.
        \item The full details can be provided either with the code, in appendix, or as supplemental material.
    \end{itemize}

\item {\bf Experiment statistical significance}
    \item[] Question: Does the paper report error bars suitably and correctly defined or other appropriate information about the statistical significance of the experiments?
    \item[] Answer: \answerYes{} 
    \item[] Justification: We report error bars for all 100-sample settings in \Cref{sec:experiment results} and \Cref{appendix: additional experimental results}, where performance is more sensitive to sampling variability. In sufficient-data settings, the variance is expected to be substantially reduced, and our conclusions are supported by consistent trends across languages. Given this reduced sensitivity and the high computational cost of repeated large-scale runs, we report error bars for the 100-sample experiments only.
    \item[] Guidelines:
    \begin{itemize}
        \item The answer \answerNA{} means that the paper does not include experiments.
        \item The authors should answer \answerYes{} if the results are accompanied by error bars, confidence intervals, or statistical significance tests, at least for the experiments that support the main claims of the paper.
        \item The factors of variability that the error bars are capturing should be clearly stated (for example, train/test split, initialization, random drawing of some parameter, or overall run with given experimental conditions).
        \item The method for calculating the error bars should be explained (closed form formula, call to a library function, bootstrap, etc.)
        \item The assumptions made should be given (e.g., Normally distributed errors).
        \item It should be clear whether the error bar is the standard deviation or the standard error of the mean.
        \item It is OK to report 1-sigma error bars, but one should state it. The authors should preferably report a 2-sigma error bar than state that they have a 96\% CI, if the hypothesis of Normality of errors is not verified.
        \item For asymmetric distributions, the authors should be careful not to show in tables or figures symmetric error bars that would yield results that are out of range (e.g., negative error rates).
        \item If error bars are reported in tables or plots, the authors should explain in the text how they were calculated and reference the corresponding figures or tables in the text.
    \end{itemize}

\item {\bf Experiments compute resources}
    \item[] Question: For each experiment, does the paper provide sufficient information on the computer resources (type of compute workers, memory, time of execution) needed to reproduce the experiments?
    \item[] Answer: \answerYes{} 
    \item[] Justification: We provide a resource consumption table in \Cref{sec:resource_consumption}, including the number of training steps,the number of used samples, wall-clock time and GPU memory consumption. 
    \item[] Guidelines:
    \begin{itemize}
        \item The answer \answerNA{} means that the paper does not include experiments.
        \item The paper should indicate the type of compute workers CPU or GPU, internal cluster, or cloud provider, including relevant memory and storage.
        \item The paper should provide the amount of compute required for each of the individual experimental runs as well as estimate the total compute. 
        \item The paper should disclose whether the full research project required more compute than the experiments reported in the paper (e.g., preliminary or failed experiments that didn't make it into the paper). 
    \end{itemize}
    
\item {\bf Code of ethics}
    \item[] Question: Does the research conducted in the paper conform, in every respect, with the NeurIPS Code of Ethics \url{https://neurips.cc/public/EthicsGuidelines}?
    \item[] Answer: \answerYes{} 
    \item[] Justification: We confirm that this study strictly adheres to the NeurIPS Code of Ethics.
    \item[] Guidelines:
    \begin{itemize}
        \item The answer \answerNA{} means that the authors have not reviewed the NeurIPS Code of Ethics.
        \item If the authors answer \answerNo, they should explain the special circumstances that require a deviation from the Code of Ethics.
        \item The authors should make sure to preserve anonymity (e.g., if there is a special consideration due to laws or regulations in their jurisdiction).
    \end{itemize}

\item {\bf Broader impacts}
    \item[] Question: Does the paper discuss both potential positive societal impacts and negative societal impacts of the work performed?
    \item[] Answer: \answerYes{} 
    \item[] Justification: Both potential positive and negative societal impacts are discussed in \Cref{sec:impact_statement}.
    \item[] Guidelines:
    \begin{itemize}
        \item The answer \answerNA{} means that there is no societal impact of the work performed.
        \item If the authors answer \answerNA{} or \answerNo, they should explain why their work has no societal impact or why the paper does not address societal impact.
        \item Examples of negative societal impacts include potential malicious or unintended uses (e.g., disinformation, generating fake profiles, surveillance), fairness considerations (e.g., deployment of technologies that could make decisions that unfairly impact specific groups), privacy considerations, and security considerations.
        \item The conference expects that many papers will be foundational research and not tied to particular applications, let alone deployments. However, if there is a direct path to any negative applications, the authors should point it out. For example, it is legitimate to point out that an improvement in the quality of generative models could be used to generate Deepfakes for disinformation. On the other hand, it is not needed to point out that a generic algorithm for optimizing neural networks could enable people to train models that generate Deepfakes faster.
        \item The authors should consider possible harms that could arise when the technology is being used as intended and functioning correctly, harms that could arise when the technology is being used as intended but gives incorrect results, and harms following from (intentional or unintentional) misuse of the technology.
        \item If there are negative societal impacts, the authors could also discuss possible mitigation strategies (e.g., gated release of models, providing defenses in addition to attacks, mechanisms for monitoring misuse, mechanisms to monitor how a system learns from feedback over time, improving the efficiency and accessibility of ML).
    \end{itemize}
    
\item {\bf Safeguards}
    \item[] Question: Does the paper describe safeguards that have been put in place for responsible release of data or models that have a high risk for misuse (e.g., pre-trained language models, image generators, or scraped datasets)?
    \item[] Answer: \answerNA{} 
    \item[] Justification: This paper does not release new datasets or trained model weights with a high risk of misuse. We use publicly available datasets (Okapi and orca-dpo) and language models (Gemma3 and Bloom) under their existing licenses, and our released code is intended only to support reproducibility of the reported experiments.
    \item[] Guidelines:
    \begin{itemize}
        \item The answer \answerNA{} means that the paper poses no such risks.
        \item Released models that have a high risk for misuse or dual-use should be released with necessary safeguards to allow for controlled use of the model, for example by requiring that users adhere to usage guidelines or restrictions to access the model or implementing safety filters. 
        \item Datasets that have been scraped from the Internet could pose safety risks. The authors should describe how they avoided releasing unsafe images.
        \item We recognize that providing effective safeguards is challenging, and many papers do not require this, but we encourage authors to take this into account and make a best faith effort.
    \end{itemize}

\item {\bf Licenses for existing assets}
    \item[] Question: Are the creators or original owners of assets (e.g., code, data, models), used in the paper, properly credited and are the license and terms of use explicitly mentioned and properly respected?
    \item[] Answer: \answerYes{} 
    \item[] Justification: The paper credits creators of all datasets and language models used in \Cref{sec:experiment_settings}.  
    \item[] Guidelines:
    \begin{itemize}
        \item The answer \answerNA{} means that the paper does not use existing assets.
        \item The authors should cite the original paper that produced the code package or dataset.
        \item The authors should state which version of the asset is used and, if possible, include a URL.
        \item The name of the license (e.g., CC-BY 4.0) should be included for each asset.
        \item For scraped data from a particular source (e.g., website), the copyright and terms of service of that source should be provided.
        \item If assets are released, the license, copyright information, and terms of use in the package should be provided. For popular datasets, \url{paperswithcode.com/datasets} has curated licenses for some datasets. Their licensing guide can help determine the license of a dataset.
        \item For existing datasets that are re-packaged, both the original license and the license of the derived asset (if it has changed) should be provided.
        \item If this information is not available online, the authors are encouraged to reach out to the asset's creators.
    \end{itemize}

\item {\bf New assets}
    \item[] Question: Are new assets introduced in the paper well documented and is the documentation provided alongside the assets?
    \item[] Answer: \answerYes{}  
    \item[] Justification: We provide our code and documentation in \href{https://github.com/geronest/maml-preferences}{https://github.com/geronest/maml-preferences}.
   
    \item[] Guidelines:
    \begin{itemize}
        \item The answer \answerNA{} means that the paper does not release new assets.
        \item Researchers should communicate the details of the dataset\slash code\slash model as part of their submissions via structured templates. This includes details about training, license, limitations, etc. 
        \item The paper should discuss whether and how consent was obtained from people whose asset is used.
        \item At submission time, remember to anonymize your assets (if applicable). You can either create an anonymized URL or include an anonymized zip file.
    \end{itemize}

\item {\bf Crowdsourcing and research with human subjects}
    \item[] Question: For crowdsourcing experiments and research with human subjects, does the paper include the full text of instructions given to participants and screenshots, if applicable, as well as details about compensation (if any)? 
    \item[] Answer: \answerNA{} 
    \item[] Justification: This paper does not conduct crowdsourcing experiments. We use publicly available datasets and follow the documentation provided by the original dataset creators.
    \item[] Guidelines:
    \begin{itemize}
        \item The answer \answerNA{} means that the paper does not involve crowdsourcing nor research with human subjects.
        \item Including this information in the supplemental material is fine, but if the main contribution of the paper involves human subjects, then as much detail as possible should be included in the main paper. 
        \item According to the NeurIPS Code of Ethics, workers involved in data collection, curation, or other labor should be paid at least the minimum wage in the country of the data collector. 
    \end{itemize}

\item {\bf Institutional review board (IRB) approvals or equivalent for research with human subjects}
    \item[] Question: Does the paper describe potential risks incurred by study participants, whether such risks were disclosed to the subjects, and whether Institutional Review Board (IRB) approvals (or an equivalent approval/review based on the requirements of your country or institution) were obtained?
    \item[] Answer: \answerNA{} 
    \item[] Justification: This paper does not involve new research with human subjects.
    \item[] Guidelines:
    \begin{itemize}
        \item The answer \answerNA{} means that the paper does not involve crowdsourcing nor research with human subjects.
        \item Depending on the country in which research is conducted, IRB approval (or equivalent) may be required for any human subjects research. If you obtained IRB approval, you should clearly state this in the paper. 
        \item We recognize that the procedures for this may vary significantly between institutions and locations, and we expect authors to adhere to the NeurIPS Code of Ethics and the guidelines for their institution. 
        \item For initial submissions, do not include any information that would break anonymity (if applicable), such as the institution conducting the review.
    \end{itemize}

\item {\bf Declaration of LLM usage}
    \item[] Question: Does the paper describe the usage of LLMs if it is an important, original, or non-standard component of the core methods in this research? Note that if the LLM is used only for writing, editing, or formatting purposes and does \emph{not} impact the core methodology, scientific rigor, or originality of the research, declaration is not required.
    \item[] Answer: \answerNA{} 
    \item[] Justification: The core method development in this research does not involve LLMs as any important, original, or non-standard components.
    \item[] Guidelines:
    \begin{itemize}
        \item The answer \answerNA{} means that the core method development in this research does not involve LLMs as any important, original, or non-standard components.
        \item Please refer to our LLM policy in the NeurIPS handbook for what should or should not be described.
    \end{itemize}

\end{enumerate}